\documentclass[10pt]{article}
\usepackage{fullpage}

\usepackage[numbers]{natbib}

\usepackage{times,amsmath,amsfonts,amssymb,url,color,graphicx}
\usepackage{accents}
\usepackage{wrapfig}
\usepackage{subcaption} 
\usepackage{boxedminipage}

\usepackage{algorithm}
\usepackage{algorithmic}
\algsetup{indent=2em}

\usepackage{etaremune}

\usepackage{relsize}

\usepackage{tikz}
\usetikzlibrary{arrows,decorations.pathmorphing,backgrounds,fit}
\usetikzlibrary{positioning} 
\usetikzlibrary{calc}
\usetikzlibrary{decorations.markings}
\usetikzlibrary{shapes}

\pgfdeclarelayer{l1}
\pgfdeclarelayer{l2}
\pgfsetlayers{l1,l2,main}
\makeatletter
\pgfkeys{%
  /tikz/node on layer/.code={
    \gdef\node@@on@layer{%
      \setbox\tikz@tempbox=\hbox\bgroup\pgfonlayer{#1}\unhbox\tikz@tempbox\endpgfonlayer\egroup}
    \aftergroup\node@on@layer
  },
  /tikz/end node on layer/.code={
    \endpgfonlayer\endgroup\endgroup
  }
}

\def\node@on@layer{\aftergroup\node@@on@layer}
\makeatother

\newtheorem{definition}{Definition}
\newtheorem{lemma}{Lemma}
\newtheorem{corollary}{Corollary}
\newtheorem{example}{Example}
\newtheorem{remark}{Remark}
\newcommand{\BlackBox}{\rule{1.5ex}{1.5ex}}  
\newenvironment{proof}{\par\noindent{\bf Proof\ }}{\hfill\BlackBox\\[2mm]}

\newcommand{\reals}{\mathbb{R}}

\newlength{\dhatheight}

\DeclareMathOperator*{\E}{\mathbb{E}}
\DeclareMathOperator*{\prob}{\mathbb{P}}

\DeclareMathOperator*{\argmin}{argmin} 
\DeclareMathOperator*{\argmax}{argmax}

\newcommand{\figref}[1]{Figure~\ref{#1}}

\newcommand{\secref}[1]{Section~\ref{#1}}

\newcommand{\lemref}[1]{Lemma~\ref{#1}}
\newcommand{\defref}[1]{Definition~\ref{#1}}
\newcommand{\exampleref}[1]{Example~\ref{#1}}

\newcommand{\appref}[1]{Appendix~\ref{#1}}

{
\begin{center}
\begin{boxedminipage}{0.8\linewidth}
\begin{center}
\textbf{\texttt{#1}}
\end{center}
\rm
\begin{tabbing}
....\=...\=...\=...\=...\=  \+ \kill
} %
{\end{tabbing} 
\end{boxedminipage} \end{center} 
}

\tikzset{
  lvl1/.style={draw,fill=blue!50,rounded corners=0.5cm,inner sep=5pt,node on layer=l1},
  lvl2/.style={draw,fill=blue!25,rounded corners=0.5cm,inner sep=5pt,node on layer=l2},
  lvl3/.style={draw=blue,fill=white,dashed,rounded corners=0.25cm,align=flush center,text width=10em,inner sep=4pt,minimum height=1cm},
  title/.style={node font=\LARGE,color=white},
  myarrow/.style={latex-latex,ultra thick,blue!80},
}

\title{On a Formal Model of Safe and Scalable Self-driving Cars}
\author{Shai Shalev-Shwartz, Shaked Shammah, Amnon Shashua}
\date{Mobileye, 2017}

\begin{document}

\maketitle

\begin{abstract}
  In recent years, car makers and tech companies have been racing
  towards self driving cars. It seems that the main parameter in this
  race is who will have the first car on the road. The goal of this
  paper is to add to the equation two additional crucial
  parameters. The first is standardization of safety assurance ---
  what are the minimal requirements that every self-driving car must
  satisfy, and how can we verify these requirements. The second
  parameter is scalability --- engineering solutions that lead to
  unleashed costs will not scale to millions of cars, which will push
  interest in this field into a niche academic corner, and drive the
  entire field into a ``winter of autonomous driving''. In the first
  part of the paper we propose a white-box, interpretable,
  mathematical model for safety assurance, which we call
  Responsibility-Sensitive Safety (RSS). In the second part we
  describe a design of a system that adheres to our safety assurance
  requirements and is scalable to millions of cars.
\end{abstract}

\section{Introduction}

The ``Winter of AI'' is commonly known as the decades long period of
inactivity following the collapse of Artificial Intelligence research
that over-reached its goals and hyped its promise until the inevitable
fall during the early 80s. We believe that the development of
Autonomous Vehicles (AV) is dangerously moving along a similar path
that might end in great disappointment after which further progress
will come to a halt for many years to come.

The challenges posed by most current approaches are centered around
lack of safety guarantees, and lack of scalability. Consider the issue
of guaranteeing a multi-agent safe driving (``Safety''). Given that
society will unlikely tolerate road accident fatalities caused by
machines, guarantee of Safety is paramount to the acceptance of
autonomous vehicles. Ultimately, our desire is to guarantee zero accidents, but this is
impossible since multiple agents are typically involved in an accident
and one can easily envision situations where an accident occurs solely
due to the blame of other agents (see Fig.~\ref{fig:absolute} for
illustration). In light of this, the typical response of practitioners
of autonomous vehicle is to resort to a statistical data-driven approach where Safety
validation becomes tighter as more mileage is collected.

To appreciate the problematic nature of a data-driven approach to
Safety, consider first that the probability of a fatality caused by an
accident per one hour of (human) driving is known to be $10^{-6}$. It
is reasonable to assume that for society to accept machines to replace
humans in the task of driving, the fatality rate should be reduced by
three orders of magnitude, namely a probability of $10^{-9}$ per
hour\footnote{This estimate is inspired from the fatality rate of air
  bags and from aviation standards. In particular, $10^{-9}$ is the
  probability that a wing will spontaneously detach from the aircraft
  in mid air.}.  In this regard, attempts to guarantee Safety using a
data-driven statistical approach, claiming increasing superiority as
more mileage is driven, are naive at best.  The amount of data
required to guarantee a probability of $10^{-9}$ fatality per hour of
driving is proportional to its inverse, $10^9$ hours of data (see
details in the sequel), which is roughly in the order of thirty
billion miles. Moreover, a multi-agent system interacts with its
environment and thus cannot be validated offline\footnote{unless a
  realistic simulator emulating real human driving with all its
  richness and complexities such as reckless driving is available, but
  the problem of validating the simulator is even harder than creating
  a Safe autonomous vehicle agent --- see
  \secref{sec:existing}.}, thus any change to the software of
planning and control will require a new data collection of the same
magnitude --- clearly unwieldy.  Finally, developing a system through
data invariably suffers from lack of transparency, interpretability,
and explainability of the actions being taken --- if an autonomous
vehicle kills someone, we need to know the reason. Consequently, a
model-based approach to Safety is required but the existing
"functional safety" and ASIL requirements in the automotive industry
are not designed to cope with multi-agent environments. Hence the need
for a formal model of Safety which is one of the goals of this paper.

The second area of risk lies with lack of scalability. The difference
between autonomous vehicles and other great science and technology achievements of the
past is that as a ``science project'' the effort is not sustainable
and will eventually lose steam. The premise underlying autonomous vehicles goes beyond
``building a better world'' and instead is based on the premise that
mobility without a driver can be sustained at a lower cost than with a
driver. This premise is invariably coupled with the notion of
scalability --- in the sense of supporting mass production of autonomous vehicles (in
the millions) and more importantly of supporting a negligible
incremental cost to enable driving in a new city. Therefore the cost
of computing and sensing does matter, if autonomous vehicles are to be mass
manufactured, the cost of validation and the ability to drive
``everywhere'' rather than in a select few cities is also a necessary
requirement to sustain a business.

The combined issues
of Safety and Scalability contain the risk of ``Winter of autonomous vehicles''. The
goal of this paper is to provide a formal model of how Safety and
Scalability are pieced together into an autonomous vehicles program that society can
accept and is scalable in the sense of supporting millions of cars
driving anywhere in the developed countries.

The contribution of this paper is twofold. On the Safety front we
introduce a model called ``Responsibility Sensitive Safety'' (RSS)
which formalizes an interpretation of ``Duty of Care'' from Tort
law. The Duty of Care states that an individual should exercise
``reasonable care'' while performing acts that could harm others. RSS is
a rigorous mathematical model formalizing an interpretation of the law
which is applicable to self-driving cars. RSS is designed to achieve
three goals: first, the interpretation of the law should be {\it
  sound\/} in the sense that it complies with how humans interpret the
law. While we are at it we would like also to prove ``utopia'' --- meaning
that if all agents follow RSS's interpretation then there will be zero
accidents. Second, the interpretation should lead to a {\it useful\/}
driving policy, meaning it will lead to an agile driving policy rather
than an overly-defensive driving which inevitably would confuse other
human drivers and will block traffic and in turn limit the scalability
of system deployment; third, the interpretation should be {\it
  efficiently verifiable\/} in the sense that we can rigorously prove
that the self-driving car implements correctly the interpretation of
the law. The last property is not obvious at all because there could
be many interpretations which are not analytically verifiable because
of ``butterfly effects'' where a seemingly innocent action could lead to
an accident of the agent's fault in the longer future.
 
As highlighted in Fig.~\ref{fig:absolute}, guaranteeing that an agent
will never be involved in an accident is impossible. Hence, our
ultimate goal is to guarantee that an agent will be careful enough so
as it will never be part of the \emph{cause} of an accident.  In other
words, the agent should never cause an accident and should be cautious
enough so as to be able to compensate for \emph{reasonable} mistakes
of other drivers.  Also noteworthy, is that the definition of RSS is
agnostic to the manner in which it is implemented --- which is a key
feature to facilitate our goal of creating a convincing global safety
model.

Our second contribution evolves around the introduction of a
``semantic'' language that consists of units, measurements, and action
space, and specification as to how they are incorporated into
Planning, Sensing and Actuation of the autonomous vehicles. To get a sense of what we
mean by Semantics, consider how a human taking driving lessons is
instructed to think about ``driving policy''.  These instructions are
not geometric --- they do not take the form ``drive 13.7 meters at the
current speed and then accelerate at a rate of 0.8 $m/s^2$''. Instead,
the instructions are of a semantic nature --- ``follow the car in
front of you'' or ``overtake that car on your left''. The language of
human driving policy is about longitudinal and lateral goals rather
than through geometric units of acceleration vectors. We develop a
formal Semantic language and show that the Semantic model is crucial
on multiple fronts connected to the computational complexity of
Planning that do not scale up exponentially with time and number of
agents, to the manner in which Safety and Comfort interact, to the way
the computation of sensing is defined and the specification of sensor
modalities and how they interact in a fusion methodology. We show how
the resulting fusion methodology (based on the semantic language)
guarantees the RSS model to the required $10^{-9}$ probability of
fatality, per one hour of driving, while performing only offline
validation over a dataset of the order of $10^5$ hours of driving
data.

Specifically, we show that in a reinforcement learning setting we can
define the Q function\footnote{A function evaluating the long term
  quality of performing an action $a\in A$ when the agent is at state
  $s\in S$. Given such a Q-function, the natural choice of an action
  is to pick the one with highest quality,
  $\pi(s) = \argmax_a Q(s, a)$.}  over actions defined over a semantic
space in which the number of trajectories to be inspected at any given
time is bounded by $10^4$ regardless of the time horizon used for
Planning. Moreover, the signal to noise ratio in this space is high,
allowing for effective machine learning approaches to succeed in
modeling the Q function. In the case of computation of sensing,
Semantics allow to distinguish between mistakes that affect Safety
versus those mistakes that affect the Comfort of driving. We define a
PAC model\footnote{Probably Approximate Correct (PAC), borrowing
  Valiant's PAC-learning terminology.} for sensing which is tied to
the Q function and show how measurement mistakes are incorporated into
Planning in a manner that complies with RSS yet allows to optimize the
comfort of driving. The language of semantics is shown to be crucial
for the success of this model as other standard measures of error,
such as error with respect to a global coordinate system, do not
comply with the PAC sensing model. In addition, the semantic language
is also a critical enabler for defining HD-maps that can be
constructed using low-bandwidth sensing data and thus be constructed
through crowd-sourcing and support scalability.
 
To summarize, we propose a formal model that covers all the important
ingredients of an autonomous vehicle: sense, plan and act. The model
guarantees that from a Planning perspective there will be no accidents
which are caused by the autonomous vehicle, and also through a PAC
sensing model guarantees that, with sensing errors, a fusion
methodology we present will require only offline data collection of a
very reasonable magnitude to comply with our Safety
model. Furthermore, the model ties together Safety and Scalability
through the language of semantics, thereby providing a complete
methodology for a safe and scalable autonomous vehicles. Finally, it
is worth noting that developing an accepted safety model that would be
adopted by the industry and regulatory bodies is a necessary condition
for the success of autonomous vehicles --- and it is better to do it
earlier rather than later. An early adoption of a safety model will
enable the industry to focus resources along a path that will lead to
acceptance of autonomous vehicles. Our RSS model contains parameters
whose values need to be determined through discussion with regulatory
bodies and it would serve everyone if this discussion happens early in
the process of developing autonomous vehicles solutions.

\subsection{Safety: Functional versus Nominal}

When dealing about safety it is important to bear in mind the
distinction between \emph{functional} versus \emph{nominal} safety.
Functional Safety (FuSa) refers to the integrity of the operation in
an electrical (i.e. HW/SW) subsystem that is operating in a safety
critical domain.  Functional Safety is concerned with a failure in HW
or bugs in the SW that could lead to a safety hazard.  For the
automotive industry this is well covered by ISO 26262 which defines
different Automotive Safety Integrity Levels (ASIL) that provide
Failure In Time (FIT) targets for HW and also define systematic
processes for how SW should be defined, developed and tested such that
it conforms with good systems engineering practices.  These include
the rigorous maintenance of requirements and traceability from those
requirements to different safety goals of the system.

However, the most Functionally Safe vehicle in the world can still
crash into everyone and everything due to bad logic in the code that
results in an unsafe driving decision.  Functional Safety cannot help
us here; instead this is the domain of Nominal safety.  Nominal safety
is the concern of whether the AV is making safe logical decisions
assuming that the HW and SW systems are operating error free (i.e. are
functionally safe).  Functional Safety then is a necessary, but not
sufficient measure of safety assurance when it comes to evaluating the
safety of an AV.  In fact, there exists no nominal safety standard for
the safe decision making capabilities of an AV.  In the remainder of
this paper, it is the nominal safety we focus on.

\subsection{Safety: Sense/Plan/Act Methodology}

Automated Vehicles are robotic systems and contain three primary
stages of functionality: Sense, Plan and Act.  Sensing is the ability
to accurately perceive the environment around the vehicle. Planning,
commonly referred to as {\it driving policy}, is where decisions are
made about what strategic (i.e. change lanes) and tactical
(i.e. overtake the blue car) decisions to take. Acting is the
issuance of the decision (translated into mathematical trajectories
and velocities) to the various actuators within the vehicle to perform
the driving decision. The focus of the paper is on the sensing
and planning parts (since the acting part is by and large well
understood by control theory).

Mistakes that might lead to accidents can stem from sensing errors or
planning errors.  Validation of Sensing systems can be efficiently
performed offline\footnote{Strictly speaking, the vehicle actions
  might change the distribution over the way we view the
  environment. However, this dependency can be circumvented by data
  augmentation techniques.}  through the use of large ground truth
data sets and sensing errors can be mitigated through redundant
sensing modalities that ensure there are at least two independent
subsystems to detect any one object, which also serves to simplify
validation for each subsystem. A detailed description of the redundant
system approach is given in \secref{sec:sensing-safety}.  Planning on
the other hand, presents unique validation challenges.  Driving a
vehicle is a multi-agent process and decisions should dependen on the
actions and responses of others.  This is why when humans take a
driving test, we do not do so on a closed track but rather in the real
world because it is only there that our multi-agent decision making
capabilities can be sufficiently evaluated. In the next section we
review existing approaches for evaluating the safety of planning, and
in \secref{sec:RSS} we describe our RSS model.

\section{Existing Approaches to Claims on Safe AV Decision Making}\label{sec:existing}

Five approaches for evaluating the safety of AV are currently being
promoted in the industry: miles driven, disengagements, simulation,
scenario based testing and proprietary approaches.

The ``Miles driven'' approach is based on a statistical argument attempting to show
that self-driving cars are statistically better than human
drivers. This approach is problematic because of the sheer amount of miles that would need to be driven to gain enough statistical evidence that a claimed probability of error (i.e. the chance of making an unsafe driving decision) has been met. 
In the following technical lemma, we formally show why a
statistical approach to validation of an autonomous vehicles system is infeasible, even
for validating a simple claim such as ``on average, the system makes
an accident once in $N$ hours''. 
\begin{lemma}\label{lem:statistical_validation_infeasible}
  Let $X$ be a probability space, and $A$ be an event for which
  $\Pr(A)=p_1<0.1$. Assume we sample $m=\frac{1}{p_1}$ i.i.d. samples from
  $X$, and let $Z=\sum_{i=1}^m \mathbf{1}_{[x\in A]}$. Then 
\[
\Pr(Z=0)\geq e^{-2}.
\]
\end{lemma}
\begin{proof}
We use the inequality $1-x\geq
e^{-2x}$ (proven for completeness in \appref{lem:1-x}), to get
\[\Pr(Z=0)=(1-p_1)^m\geq e^{-2p_1 m}=e^{-2}.\]
\end{proof}
\begin{corollary}\label{cor:cannot_validate}
  Assume an autonomous vehicle system $AV_1$ makes an accident with small yet
  insufficient probability $p_1$. Any deterministic validation
  procedure which is given $1/p_1$ samples, will, with constant
  probability, not distinguish between $AV_1$ and a different autonomous vehicle
  system $AV_0$ which never makes accidents.
\end{corollary}
In order to gain perspective over the typical values for such
probabilities, consider public accident statistics in the United
States. The probability of a fatal accident for a human driver in 1
hour of driving is $10^{-6}$.  From the Lemma above, if we want to
claim that an AV meets the same probability of a fatal accident, one
would need more than $10^6$ hours of driving.  Assuming that the average speed
in 1 hour of driving is 30 miles per hour, the AV would need to drive
30 million miles to have enough statistical evidence that the AV under
test meets the same probability of a fatal accident in 1 hour of
driving as a human driver. However, this would only deliver an AV as
good as a human, whereas the promise of AVs is to deliver a level of
safety above and beyond any human. Hence, if instead our target was to
be three orders of magnitude safer than a human, we have a new
probability target of $10^{-9}$, which would then require us to drive
30 Billion miles to have enough statistical evidence to argue that the
probability of the AV making an unsafe decision that will lead to a
fatality per one hour of driving is $10^{-9}$.  And then what if a
single line of code has changed in the AV's planning software?  Now
the testing must start all over again as it is impossible to know if
that code change has resulted in a new failure that was not present
during the first 30 Billion mile drive.

Further, there is the concern whether the miles driven are meaningful to begin with. One could easily create ``low value'' driving experiences, such as accumulating miles on empty roads without other road users to challenge the driving policy of the system. While a certain amount of meaningful miles is warranted when performing on-road testing (with the specific number and kind of miles to be defined together by government and industry), using a generic miles driven argument to make safety claims on the decision making capabilities of an AV does not provide sufficient assurance of whether an AV knows how to drive safely and so is thus worthy of a license to drive.

Another metric frequently cited as proof of the safety of an AV's decision making is known as a {\it disengagement}.  A disengagement is roughly defined as a situation where a human safety driver had to intervene in the operation of the AV because it made an unsafe decision that was going to lead to an accident. The main rational of the disengagements metric is that it counts for a more frequent statistical event: ``almost an accident'' rather than ``an accident'' or ``an accident leading to a fatality''. The problem, however, is twofold --- first, like in the ``mileage driven'' approach there is an issue with the distribution of ``easy'' versus ``challenging'' cases during testing. An AV that is tested in ``easy'' environments with limited traffic could rack up mileage with zero disengagements yet will not be ``safe'' to drive in real congested traffic. Second, the assumption  
about the relation between ``almost an accident'' and ``an accident'' is likely to not hold for the difficult rare cases.

Another approach to validate safety-claims of an AV is through {\it simulation}. The idea is to build a simulator with a virtual world within which the AV's software will be driven 10s of Billions of miles as a way to achieve the huge ``miles driven'' targets that would be needed to make a statistical claim on the safety of the AV's decision making capabilities.

The problem with this
argument is that validating that the simulator faithfully represents
reality is as hard as validating the driving policy itself. To see why this is
true, suppose that the simulator has been validated in the sense that
applying a driving policy $\pi$ in the simulator leads to a
probability of an accident of $\hat{p}$, and the probability of an
accident of $\pi$ in the real world is $p$, with
$|p - \hat{p}| < \epsilon$. (Say that we need that $\epsilon$ will be
smaller than $10^{-9}$.)  Now we replace the driving policy to be
$\pi'$. Suppose that with probability of $10^{-8}$, $\pi'$ performs a
weird action that confuses human drivers and leads to an accident. It
is possible (and even rather likely) that this weird action is not
modeled in the simulator, without contradicting its superb
capabilities in estimating the performance of the original policy
$\pi$.  This proves that even if a simulator has been shown to reflect
reality for a driving policy $\pi$, it is not guaranteed to reflect
reality for another driving policy.

Another active area of work that claims to provide assurance of an AV's ability to make safe driving decisions is {\it scenario based\/} verification.  The basic idea is that if only we could enumerate all the possible driving scenarios that could exist in the entire world then we could simply expose the AV (via simulation, closed track testing or on-road testing) to all of those scenarios and as a result we can now be confident that the AV will only ever make safe driving decisions.

The challenge with scenario-based approaches has to do with the notion of ``generalization'', in the sense of the underlying assumption that if the AV passes the scenarios successfully then it is likely to pass other similar scenarios as well. The danger, just as in machine learning, is ``overfitting'' the system to pass the test. Even if extra care is taken not to overfit, the arguments of generalization are weak at best.

%
%

\begin{figure}
\def\height{2}
\def\width{1}
\begin{center}
\begin{tikzpicture}[scale=0.5, every node/.style={transform shape}]
\fill [gray!80] (-2.5*\width,-2*\height) rectangle
(2.5*\width,2*\height);
\foreach \x in {-0.75*\width, 0.75*\width, -2.25*\width, 2.25*\width}
 \draw [white, densely dashed, very thick] (\x, -2*\height)--(\x,
2*\height);
\node[rotate=90] at (-2*\width+0.5,0) {\includegraphics[width=2cm]{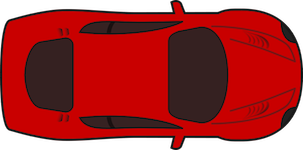}};
\node[rotate=90] at (2*\width-0.5,0) {\includegraphics[width=2cm]{red_car}};
\node[rotate=90] at (0,2.5) {\includegraphics[width=2cm]{red_car}};
\node[rotate=90] at (0,-2.5) {\includegraphics[width=2cm]{red_car}};
\node[rotate=90] at (0,0) {\includegraphics[width=2cm]{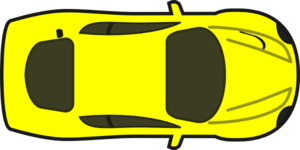}};
\end{tikzpicture}
\caption{The central car can do nothing to ensure absolute safety.}\label{fig:absolute}
\end{center}
\end{figure}

\section{The Responsibility-Sensitive Safety (RSS) model for Multi-agent Safety}\label{sec:RSS}

The discussion above focused on the shortcomings of the existing approaches for validating the nominal safety of an AV agent. Before we proceed, we should rule out what is clearly infeasible which is the naive statement that an AV, sharing the road with human-driven cars, will never be involved in an accident --- a statement we refer to as {\it Utopia}.

We say an action $a$
taken by a car $c$ is \emph{absolutely safe} if no accident can follow
it at some future time. It is easy to see that it is impossible to
achieve absolute safety, by observing simple driving scenarios, for
example, as depicted in \figref{fig:absolute}: from the central car's
perspective, no action can ensure that none of the surrounding cars will
crash into it, and no action can help it escape this potentially
dangerous situation. We emphasize that solving this problem by
forbidding the autonomous car from being in such situations is
completely impossible --- every highway with more than 2 lanes will
lead to it and forbidding this scenario amounts to staying in the parking lot.

Instead we refer to the driving forces underlying human judgement when
sharing the road with other road-users. Traffic laws are well defined
until one encounters the elusive directive, common in Tort law, called
\emph{the Duty of Care}.

The Duty of Care states that an individual should exercise
``reasonable care'' while performing acts that could harm others. What
is meant in ``being careful'' is open for interpretation and must
follow societal norms whose definitions are fluid, change over time,
and gradually get clarified through legal precedents over past
accidents that went through court proceedings for a resolution. A
human driver must exercise care due to the uncertainty regarding the
actions of other road users. If the driver must take into account the
extreme worst case about the actions of other road users then driving
becomes impossible. Hence, the human driver makes some ``reasonable''
assumptions about the worst case scenarios of other road-users. We
refer to the assumptions being made as an ``interpretation'' of the
Duty of Care law.

Responsibility-Sensitive-Safety (RSS) is a rigorous mathematical model
formalizing an interpretation of the Duty of Care law. RSS is designed
to achieve three goals: first, the interpretation of the law should be
{\it sound\/} in the sense that it complies with how humans interpret
the law. While we are at it we would like also to prove ``AI-Utopia''
--- meaning that if all agents follow RSS interpretation then there
will be zero accidents. Second, the interpretation should lead to a
{\it useful\/} driving policy, meaning it will lead to an agile
driving policy rather than an overly-defensive driving which
inevitably would confuse other human drivers and will block traffic
and in turn limit the scalability of system deployment; As an example
of a valid, but not useful, interpretation is to assume that in order
to be ``careful'' our actions should not affect other road
users. Meaning, if we want to change lane we should find a gap large
enough such that if other road users continue their own motion
uninterrupted we could still squeeze-in without a collision. Clearly,
for most societies this interpretation is over-cautious and will lead
the AV to block traffic and be non-useful.

Third, the interpretation should be {\it efficiently verifiable\/} in the sense that we can rigorously prove that the self-driving car implements correctly the interpretation of the law. The last property is not obvious at all because there could be many interpretations which are not analytically verifiable because of ``butterfly effects'' where a seemingly innocent action could lead to an accident of the agent's fault in the longer future. One way to ensure efficient verification is to design the interpretation to follow the inductive principle --- a feature we designed into the RSS.
 
By and large, RSS is constructed by formalizing the following 5
``common sense'' rules:
\begin{enumerate}
\item Do not hit someone from behind.
\item Do not cut-in recklessly.
\item Right-of-way is given, not taken.
\item Be careful of areas with limited visibility
\item If you can avoid an accident without causing another one, you must do it.
\end{enumerate}

The subsections below formalize those rules. We begin below with a simple scenario in order to get used to the concept we will be developing in the remainder of the paper. 

\subsection{Gentle Start --- Single Lane Road} \label{sec:safe-distance}

We start with the simplest possible scenario: a single lane road,
where cars cannot perform lateral manoeuvres. This scenario will allow
us to introduce some first, simplistic versions of key concepts such as
\emph{safe distance}, \emph{dangerous situation}, \emph{proper
  response}, and \emph{responsibility}. It also enables us to showcase
our technique for formally proving the safety of a driving policy.

When driving along a single lane road, the common sense rule is ``if someone
hits you from behind it is not your fault''. So, a first try would be
to define that if a rear car, $c_r$, hits a front car, $c_f$, from
behind, then $c_r$ is responsbile for the accident. But, this
definition does not always comply with common sense. For example,
suppose that both $c_r$ and $c_f$ are driving slowly up a hill, and then
$c_f$ slows down and starts rolling backward until it hits $c_r$. Even
though $c_r$ hit $c_f$ from behind, the common sense in this situation
is that $c_f$ should be responsible. 

We will get back to this issue in later subsections, so for the
remainder of this subsection, we assume that cars never drive
backward. We turn to discuss the more interesting question of: ``how
can $c_r$ ensure that it will never hit $c_f$ from
behind''. Intuitively, it is the responsibility of $c_r$ to keep a
``safe distance'' from $c_f$, and this ``safe distance'' should be
large enough so that no matter what, $c_r$ will not hit $c_f$. In our
simple case, the worst case situation is that $c_f$ will suddenly
brake hard, it will take $c_r$ some time to figure this out and to
brake as well, and then both cars will decelerate until reaching a full
stop.  A formalism of this concept is given below.
\begin{definition}[Safe longitudinal distance --- same
  direction] \label{def:dmin} A \emph{longitudinal distance} between a
  car $c_r$ that drives behind another car $c_f$, where both cars are
  driving at the same direction, is \emph{safe} w.r.t. a response time
  $\rho$ if for any braking of at most $a_{\max,\mathrm{brake}}$,
  performed by $c_f$, if $c_r$ will accelerate by at most
  $a_{\max,\mathrm{accel}}$ during the response time, and from there
  on will brake by at least $a_{\min,\mathrm{brake}}$ until a full
  stop then it won't collide with $c_f$.
\end{definition}
\lemref{lem:dmin} below calculates
the safe distance as a function of the velocities of $c_r,c_f$ and the
parameters in the definition.
\begin{lemma} \label{lem:dmin} Let $c_r$ be a vehicle which is behind
  $c_f$ on the longitudinal axis. Let
  $\rho,a_{\max,\mathrm{brake}}, a_{\max,\mathrm{accel}},
  a_{\min,\mathrm{brake}}$ be as in \defref{def:dmin}. Let $v_r,v_f$
  be the longitudinal velocities of the cars. Then, the minimal safe
  longitudinal distance between the front-most point of $c_r$ and the
  rear-most point of $c_f$ is:
\[
d_{\min} = \left[v_r \, \rho + \frac{1}{2}
  a_{\max,\mathrm{accel}}\,\rho^2 + \frac{(v_r + \rho\, a_{\max,\mathrm{accel}})^2}{2
    a_{\min,\mathrm{brake}}}  - \frac{v_f^2}{2 a_{\max,\mathrm{brake}}}\right]_+~,
\]
where we use the notation $[x]_+ := \max\{x,0\}$.
\end{lemma}
\begin{proof}
  Let $d_0$ denote the initial distance between $c_r$ and $c_f$.
  Denote $v_{\rho,\max} = v_r + \rho\, a_{\max,\mathrm{accel}}$.  The
  velocity of the front car decreases with $t$ at a rate
  $a_{\max,\mathrm{brake}}$ (until arriving to zero or that a
  collision happens), while the velocity of the rear car increases in
  the time interval $[0,\rho]$ (until reaching $v_{\rho,\max}$) and
  then decreases at a rate
  $a_{\min,\mathrm{brake}} < a_{\max,\mathrm{brake}}$ until arriving
  to zero or to a collision. It follows that if at some point in time
  the two cars have the same velocity, then from there on, the front
  car's velocity will be smaller, and the distance between them will
  be monotonically decreasing until both cars reach a full stop (where
  the ``distance'' can be negative if collision happens). From this it
  is easy to see that the worst-case distance will happen either at
  time zero or when the two cars reach a full stop. In the former case
  we should require $d_0 > 0$. In the latter case, the distances the front and rear
  car will pass until a full stop is
  $\frac{v_f^2}{2 a_{\max,\mathrm{brake}}} $ and
  $ v_r \, \rho + \frac{1}{2} a_{\max,\mathrm{accel}}\,\rho^2 +
  \frac{v_{\rho,\max}^2}{2 a_{\min,\mathrm{brake}}} $. At that point,
  the distance between them should be larger than zero, 
\[
d_0 + \frac{ v_f^2}{2 a_{\max,\mathrm{brake}}} - \left( v_r \, \rho +
  \frac{1}{2} a_{\max,\mathrm{accel}}\,\rho^2 + \frac{
    v_{\rho,\max}^2}{2 a_{\min,\mathrm{brake}}}\right)  > 0 ~.
\]
Rearranging terms, we conclude our proof. 
\end{proof}

Let us now get back to the question of ``how can one make sure to
never hit someone from behind''. The basic idea is as follows.  When
the distance between a rear car, $c_r$, and a front car, $c_f$, is not
safe, we say that the situation is \emph{dangerous}. Let $t_b$ be some
time in which the situation becomes dangerous (that is, $t_b$ is a
dangerous time but immediately before $t_b$ the situation is not
dangerous). We require $c_r$ to \emph{properly respond} to the
dangerous situation as follows: in the time interval $[t_b,t_b+\rho]$,
$c_r$ can apply any acceleration as long as it doesn't exceed
$a_{\max,\mathrm{accel}}$. After that, and as long as the situation is
still dangerous, $c_r$ must brake by at least
$a_{\min,\mathrm{brake}}$. In addition, if $c_r$ is at a full stop and
the situation is dangerous, it must not start driving.

We now claim that this proper response behavior guarantees that $c_r$
would never hit $c_f$. The crux of the proof is an inductive
argument. We will prove that there is a sequence of increasing times,
$0 = t_0 \le t_1 < t_2 < t_3 < \ldots$, such that for every time
$t_i$, $i \ge 1$, there is no collision in the time interval
$[t_{i-1},t_i]$, and at time $t_i$ the situation is not dangerous.
The basis of the induction is the first time, $t_1$, in which $c_r$ is
starting to drive. By the definition of proper response, $t_1$ is not
a dangerous situation, and it is clear that $c_r$ did not hit someone
from behind before it started to drive (recall that for now we assume
that no cars are driving backwards).  Suppose the inductive assumption
holds for $t_1 < \ldots < t_i$. So, at time $t_i$, the situation is
non-dangerous. Let $t_b$ be the earliest $t$ after $t_i$ for which the
situation becomes dangerous. If no such $t_b$ exists, then there will
be no accidents in the time interval $[t_i,\infty)$ (because, before
an accident occurs, the situation must be dangerous), hence we are
done. Otherwise, let $t_e$ be the earliest time after $t_b$ in which
either the situation becomes non-dangerous or $c_r$ arrives to a full
stop. If at $t_e$ the situation becomes non-dangerous, then we set
$t_{i+1} = t_e$. Otherwise, we set $t_{i+1}$ to be the earliest time
after $t_e$ in which $c_r$ starts to drive again. In both cases, by
the definitions of proper response and safe distance, there will be no
accident in the time interval $[t_b,t_e]$. And, clearly, there can be
no accidents in the time intervals $[t_i,t_b]$ and
$[t_e,t_{i+1}]$. Hence, our inductive argument is concluded.

\begin{remark}[No contradictions and star-shape
  calculations] \label{rem:no-contradictions} Our definition of proper
  response enables \emph{star-shape} calculations: we can consider the
  proper response of our car with respect to each other car
  individually, each proper response implies a constraint on the
  maximal acceleration we are allowed to perform, and taking the
  minimum of these constraints is guaranteed to satisfy all the
  constraints. Furthermore, the inductive proof technique relies on
  this pairwise structure. It is important to emphasize that there is
  an intimate relationship between the specific choice of definitions
  (of dangerous situation and proper response) and the ability to
  enable star-shape calculations and to make the inductive
  argument. To illustrate this point, suppose we would slightly change
  the definition of proper response, by also requiring the front car
  to accelerate a little bit when a rear car is approaching towards it
  fast from behind. In this case, we might reach contradictions: on
  one hand we should accelerate because someone approaches fast from
  behind, while on the other hand we need to brake because the car in
  front of us is braking. To resolve such contradictions we will need
  to consider all the vehicles together, which is expensive from
  computational perspective, and requires a different proof technique.
  Maintaining definitions which support efficient star-shape calculations and
  facilitate formal correctness in the full complexity of driving
  scenes (including lateral manoeuvres, junctions, pedestrians,
  and occlusions) is a great challenge that we tackle in this paper.
\end{remark}

\begin{remark}[The parameters control the soundness/usefulness
  tradeoff] \label{rem:sound-useful-parameters} The definitions of
  safe longitudinal distance and proper response depend on
  parameters:
  $\rho, a_{\max,\mathrm{accel}}, a_{\min,\mathrm{brake}},
  a_{\max,\mathrm{brake}}$. These parameters induce \emph{assumptions}
  on the behavior of road agents---for example, the rear car assumes
  that the front car will not brake stronger than
  $a_{\max,\mathrm{brake}}$, even if physically the front car is
  capable of braking stronger than that.  If the front car brakes
  stronger than $a_{\max,\mathrm{brake}}$, then the inductive proof
  breaks and there might be an accident. Since the rear car cannot
  know the exact braking mechanism of the front car, it has no way of
  knowing the exact value of $a_{\max,\mathrm{brake}}$. Setting it to
  a very large value makes the model more sound (the number of cases
  in which the assumptions will not hold, and therefore the model will
  not capture reality, will be much smaller). On the other hand, a
  very large value of $a_{\max,\mathrm{brake}}$ induces an extremely
  defensive driving. In the extreme case, when
  $a_{\max,\mathrm{brake}} = \infty$, the safe distance formula states
  that we should refer to every car in front of us as if it stands
  still, which does not allow a normal flow of traffic.

  We therefore argue that these parameters should be determined to
  some reasonable values by regulation, as they induce a set of
  \emph{allowed assumptions} that a driver (robotic or human) can make
  on the behavior of other road users. 

  Of course, the parameters can be set differently for a robotic car and
  a human driven car. For example, the response time of a robotic car
  is usually smaller than that of a human driver and a robotic car can
  brake more effectively than a typical human driver, hence
  $a_{\min,\mathrm{brake}}$ can set to be larger for a robotic car.
  They can also be set differently for different road conditions (wet
  road, ice, snow).
\end{remark}

\begin{remark}[Utopia is possible] 
  Our inductive proof shows that if a car responds properly to
  dangerous situations then it will not hit another car from behind,
  as long as the front car will not brake stronger than
  $a_{\max,\mathrm{brake}}$ (and will not drive backwards). This
  immediately implies that if all road users will adhere to the
  assumptions, and will respond properly to dangerous situations, then
  utopia is possible, in the sense that there will be no
  accidents. This strengthens the soundness of our definitions.  While
  this claim is trivial for the simplistic case we are considering
  now, we will later show that it holds even when considering a much
  more complicated world (which includes lateral manoeuvres, different
  geometry, pedestrians, and occlusions). 
\end{remark}

Having described the main idea behind our technique, we now turn to
the harder part of constructing adequate definitions for all type of
roads. We start with formally defining the notions of longitudinal
and lateral position/velocity/acceleration.

\subsection{Preliminaries --- A Lane-Based Coordinate System} \label{sec:lane_based}

In the previous subsection we have assumed that the road is
straight. In the next subsections, we would want to keep the
simplicity of definitions by assuming that the road is comprised by
straight lanes of constant width. In that case, there is a clear
meaning to the longitudinal and lateral axes, along with an ordering
of longitudinal position. This distinction plays a significant role in
common-sense driving. However, real road are never perfectly
straight. We therefore propose a generic transformation from (global) positions on the
plane, to a lane-based coordinate system, reducing the problem yet
again to the original, ``straight lane of constant width'', case.

Assume that the lane's center is a smooth directed curve $r$ on the
plane, where all of its pieces, denoted $r^{(1)},\ldots,r^{(k)}$, are
either linear, or an arc. Note that smoothness of the curve implies
that no pair of consecutive pieces can be linear. Formally, the curve
maps a ``longitudinal'' parameter,
$Y\in [Y_{\min},Y_{\max}]\subset \reals$, into the plane, namely, the
curve is a function of the form
$r : [Y_{\min},Y_{\max}] \to \reals^2$. We define a continuous
lane-width function $w: [Y_{\min},Y_{\max}]\to\reals_+$, mapping the
longitudinal position $Y$ into a positive lane width value. For each
$Y$, from smoothness of $r$, we can define the normal unit-vector to
the curve at position $Y$, denoted $r^\perp(Y)$. We naturally define
the subset of points on the plane which reside in the lane as follows:
\[
R = \{r(Y)+\alpha w(Y)r^\perp(Y) ~\mid~Y\in [Y_{\min},Y_{\max}],~\alpha\in[\pm 1/2]\}~.
\]
Informally, our goal is to construct a transformation $\phi$ of $R$
into $\reals^2$,\footnote{Where, as in RSS, we will associate the
  $y$-axis with the ``longitudinal'' axis, and the $x$-axis with the
  ``lateral''.} such that for two cars which are on the lane, their
``logical ordering'' will be preserved:
\begin{itemize}
\item If $c_r$ is ``behind'' $c_f$ on the curve, then $\phi(c_r)_y < \phi(c_f)_y$.
\item If $c_l$ is ``to the left of'' $c_r$ on the curve, then $\phi(c_l)_x < \phi(c_r)_x$.
\end{itemize}
In order to define $\phi$, we rely on the assumption that for all $i$,
if $r^{(i)}$ is an arc of radius $\rho$, then the width of the lane
throughout $r^{(i)}$ is $\le \rho/2$. Note that this assumption holds
for any practical road. The assumption trivially implies that for all
$(x',y')\in R$, there exists a unique pair
$Y'\in [Y_{\min},Y_{\max}]$, $\alpha'\in[\pm 1/2]$,
s.t. $(x',y')=r(Y')+\alpha' w(Y')r^\perp(Y')$. We can now define
$\phi:R\to\reals^2$ to be $\phi(x',y') = (Y',\alpha')$, where
$(Y',\alpha')$ are the unique values that satisfy
$(x',y')=r(Y')+\alpha' w(Y')r^\perp(Y')$.

\begin{figure}
\def\width{1}
\def\height{2}
\def\theta{-5}
\begin{center}
\begin{tikzpicture}[scale=0.5]
\coordinate (C) at (0,0);
\draw [gray,fill=gray] (-3*\width,-1.5*\height) rectangle
(3*\width,1.5*\height);
\draw [->,black, very thick]
(-1*\width,-1.5*\height)--(-1*\width,1.5*\height);
\begin{scope}[rotate=-\theta]
\draw [dashed, white, very thick]
(-2*\width,-1.5*\height)--(-2*\width,1.5*\height);
\end{scope}
\begin{scope}[rotate=\theta]
\draw [dashed, white, very thick]
(0,-2.5*\height)--(0,2.5*\height);
\draw [blue,fill=blue] (0.2-1.5*\width,-0.5*\height) rectangle
(0.2-0.5*\width,0.5*\height);
\end{scope}
\draw [red,fill=red] (0.5*\width,-0.5*\height) rectangle
(1.5*\width,0.5*\height);
\end{tikzpicture}
\caption{Changing lane width. Although the red car drives in parallel to
  the lane's center (black arrow), it clearly makes lateral movement towards the
  lane. The blue car, although getting further away from the lane's
  center, stays in the same position w.r.t. the lane boundary.}\label{fig:width}
\end{center}
\end{figure}

This definition captures the notion of a ``lateral manoeuvre'' in
lane's coordinate system. Consider, for example, a widening lane, with
a car driving exactly on one of the lane's boundaries (see
\figref{fig:width} for an illustration). The widening of the lane
means that the car is moving away from the center of the lane, and
therefore has lateral velocity w.r.t. it. However, this doesn't mean
it performs a lateral manoeuvre. Our definition of
$\phi(x',y')_x = \alpha'$, namely, the lateral distance to the lane's
center in $w(Y')$-units, implies that the lane boundaries have a fixed
lateral position of $\pm 1/2$, hence, a car which sticks to one of the
lane's boundaries is not considered to perform any lateral
movement. Finally, it is easy to see that $\phi$ is a bijective
embedding. We will use the term \emph{lane-based coordinate system}
when discussing $\phi(R)=[Y_{\min},Y_{\max}]\times[\pm 1/2]$. 

We have thus obtained a reduction from a general lane geometry to a
straight, longitudinal/lateral, coordinate system. The meaning of
longitudinal/lateral position follows immediately and the meaning of
longitudinal/lateral velocity/acceleration follows by taking the
first/second derivatives of the position. This will be the meaning of
longitudinal/lateral position/velocity/acceleration throughout the
rest of the paper.

\subsection{Longitudinal Safe Distance and Proper Response} \label{sec:longit-safe}

The definition of a safe distance from \secref{sec:safe-distance} is
sound for the case that both the rear and front cars are driving at
the same direction.  Indeed, in this case, it is the responsibility of
the rear car to keep a safe distance from the front car, and to be
ready for unexpected, yet reasonable, braking. However, when the two
cars are driving at opposite directions, we need to refine the
definition. Consider for example a car $c_r$ that is currently at a
safe distance from a preceding car, $c_f$, that stands
still. Suddenly, $c_f$ is reversing very fast into a parking spot and $c_r$ hits
it from behind. Depending on the speed of $c_f$'s manoeuvre, the
common sense here may be that the responsibility is
not on the rear car, even though it hits $c_f$ from behind.  To
formalize this common sense, we simply note that the definitions of
``rear and front'' do not apply to scenarios in which vehicles are
moving toward each other (namely, the signs of their longitudinal
velocities are opposite). In such cases we expect both cars to
decrease the \emph{absolute value} of their velocity in order to avoid
a crash.

We could therefore define the safe distance between cars that drive in
opposite directions to be the distance required so as if both cars
will brake (after a response time) then there will be no
crash. However, it makes sense that the car that drives at the
opposite direction to the lane direction should brake harder than the
one who drives at the correct direction. This leads to the following
definition.
\begin{definition}[Safe longitudinal distance --- opposite
  directions] \label{def:dmin_opposite} Consider cars $c_1,c_2$
  driving on a lane with longitudinal velocities $v_1,v_2$, where
  $v_2 < 0$ and $v_1 \ge 0$ (the sign of the longitudinal velocity is
  according to the allowed direction of driving on the lane).  The
  longitudinal distance between the cars is safe w.r.t. a response
  time $\rho$, braking parameters
  $a_{\min,\mathrm{brake}}, a_{\min,\mathrm{brake},\mathrm{correct}}$,
  and an acceleration parameter $a_{\max,\mathrm{accel}}$, if in case
  $c_1,c_2$ will increase the absolute value of their velocities at
  rate $a_{\max,\mathrm{accel}}$ during the response time, and from
  there on will decrease the absolute value of their velocities at
  rate
  $a_{\min,\mathrm{brake},\mathrm{correct}}, a_{\min,\mathrm{brake}}$,
  respectively, until a full stop, then there will not be a collision.
\end{definition}
A calculation of the safe distance for the case of opposite directions is given in the lemma below (whose proof is straightforward, and hence omitted).
\begin{lemma}\label{lem:dmin_two_way}
  Consider the notation given in \defref{def:dmin_opposite}. Define
  $v_{1,\rho} = v_1 + \rho\, a_{\max,\mathrm{accel}}$ and
  $v_{2,\rho} = |v_2| + \rho\,a_{\max,\mathrm{accel}}$.  Then, the
  minimal safe longitudinal distance between $c_1$ and $c_2$ is:
\[
d_{\min} = \frac{v_1 + v_{1,\rho}}{2}\rho + \frac{v_{1,\rho}^2}{ 2 a_{\min,\mathrm{brake},\mathrm{correct}}}
+ \frac{|v_2| + v_{2,\rho}}{2}\rho + \frac{v_{2,\rho}^2}{ 2 a_{\min,\mathrm{brake}}} ~.
\]
\end{lemma}

Before a collision between two cars, they first need to be at a
non-safe distance.  Intuitively, the idea of the safe distance
definitions is that if both cars will respond ``properly'' to
violations of the safe distance then there cannot be a collision. If
one of them didn't respond ``properly'' then it is responsible for the
accident. To formalize this, it is first important to know the moment
in wich the cars start to be at a non-safe distance.
\begin{definition}[Dangerous Longitudinal Situation and Danger Threshold] \label{def:longitudinal_dangerous}
  We say that time $t$ is \emph{longitudinally dangerous} for cars
  $c_1,c_2$ if the distance between them at time $t$ is non safe
  (according to \defref{def:dmin} or \defref{def:dmin_opposite}).
  Given a longitudinally dangerous time $t$, its \emph{Longitudinal
    Danger Threshold}, denoted $t_b^{\mathrm{long}}$, is the earliest longitudinally
  dangerous time such that all the times in the interval $[t_b^{\mathrm{long}},t]$ are
  longitudinally dangerous. In particular, an accident can only happen
  at time $t$ if it is longitudinally dangerous, and in that case we
  say that the longitudinally Danger Threshold of the accident is the
  longitudinally threshold time of $t$.
\end{definition}
Next, we define what is a ``longitudinal proper response'' to
longitudinal dangerous situations.
\begin{definition}[Longitudinal Proper
  response] \label{def:proper_longitudinal} Let $t$ be a longitudinally dangerous
  time for cars $c_1,c_2$ and let $t_b^{\mathrm{long}}$ be the corresponding longitudinally blame
  time.  The longitudinally proper response of the two cars is to comply with the
  following constraints on the longitudinal speed:
\begin{enumerate}
\item If at the longitudinally Danger Threshold time, the two cars were driving at the same direction, and say that $c_1$ is the rear car, then:
\begin{itemize}
\item $c_1$'s acceleration must be at most $a_{\max,\mathrm{accel}}$
  during the interval $[t_b^{\mathrm{long}},t_b^{\mathrm{long}}+\rho)$ and at most
  $-a_{\min,\mathrm{brake}}$ from time $t_b^{\mathrm{long}}+\rho$ until reaching a
  safe longitudinal situation. After that, any non-positive
  acceleration is allowed.
\item $c_2$ acceleration must be at least $-a_{\max,\mathrm{brake}}$
  until reaching a safe longitudinal situation. After that, any
  non-negative acceleration is allowed.
\end{itemize}
\item If at the longitudinally Danger Threshold time the two cars were driving at opposite
  directions, and say that $c_2$ was driving at the wrong direction
  (negative velocity), then:
\begin{itemize}
\item $c_1$ acceleration must be at most
  $a_{\max,\mathrm{accel}}$ during
  the interval $[t_b^{\mathrm{long}},t_b^{\mathrm{long}}+\rho)$ and at most
  $-a_{\min,\mathrm{brake},\mathrm{correct}}$ from time $t_b^{\mathrm{long}}+\rho$ until reaching a full
  stop. After that, it can apply any non-positive acceleration
\item $c_2$ acceleration must be at least $-a_{\max,\mathrm{accel}}$ during
  the interval $[t_b^{\mathrm{long}},t_b^{\mathrm{long}}+\rho)$ and at least
  $a_{\min,\mathrm{brake}}$ from time $t_b^{\mathrm{long}}+\rho$ until reaching a full
  stop. After that, any non-negative acceleration is allowed. 
\end{itemize} 
\end{enumerate}
\end{definition}

\subsection{Lateral Safe Distance and Proper Response} \label{sec:lateral-safe}

Unlike longitudinal velocity, which can be kept to a value of $0$ for
a long time (the car is simply not moving), keeping lateral velocity
at exactly $0$ is impossible as cars usually perform small lateral
fluctuations.  It is therefore required to introduce a robust notion
of lateral velocity.
\begin{definition}[$\mu$-lateral-velocity]
Consider a point located at a lateral location $l$ at time $t$. Its
$\mu$-lateral velocity at time $t$ is defined as follows. Let $t_{out} > t$ be the
earliest future time in which the point's lateral position, denoted
$l_{out}$, is either $l-\mu/2$ or $l+\mu/2$ (if no such time exists we
set $t_{out}=\infty$). If at some time $t' \in (t,t_{out})$ the
point's lateral position is $l$, then the $\mu$-lateral-velocity is
$0$. Otherwise, the $\mu$-lateral-velocity is
$(l_{out}-l)/(t_{out}-t)$. 
\end{definition}

Roughly speaking, in order to have a collision between two vehicles,
it is required that they will be close both longitudinally and
laterally. For the longitudinal axis, we have already formalized the
notion of ``being close'' using the safe distance. We will now do the
same for lateral distance. 
\begin{definition}[Safe Lateral Distance] \label{def:lateral_safe_distance}
The lateral distance between cars $c_1,c_2$ driving with lateral
velocities $v_1,v_2$ is safe w.r.t. parameters $\rho,
a^{\mathrm{lat}}_{\min,\mathrm{brake}},
a^{\mathrm{lat}}_{\max,\mathrm{accel}},\mu$, if during the time
interval $[0,\rho]$ the two cars will apply lateral acceleration of
$a^{\mathrm{lat}}_{\max,\mathrm{accel}}$ toward each other, and after that
the two cars will apply lateral braking of
$a^{\mathrm{lat}}_{\min,\mathrm{brake}}$, until they reach zero
lateral velocity, then the final lateral distance between them will be
at least $\mu$.  
\end{definition}
A calculation of the lateral safe distance is given in the lemma below (whose proof is straightforward, and hence omitted).
\begin{lemma}\label{lem:lateral_safe_distance}
Consider the notation given in
\defref{def:lateral_safe_distance}. W.l.o.g. assume that $c_1$ is to
the left of $c_2$. Define $v_{1,\rho} = v_1 + \rho\, a^{\mathrm{lat}}_{\max,\mathrm{accel}}$ and $v_{2,\rho} = v_2 - \rho\, a^{\mathrm{lat}}_{\max,\mathrm{accel}}$.  Then, the 
  minimal safe lateral distance between the right side of $c_1$ and
  the left part of $c_2$ is:
\[
d_{\min} = \mu + \left[\frac{v_1 + v_{1,\rho}}{2}\rho +
\frac{v_{1,\rho}^2}{ 2 a^{\mathrm{lat}}_{\min,\mathrm{brake}}} - \left( \frac{v_2 + v_{2,\rho}}{2}\rho - \frac{v_{2,\rho}^2}{ 2
   a^{\mathrm{lat}}_{\min,\mathrm{brake}}} \right) \right]_+~.
\]
\end{lemma}

The following definitions are the analogue of \defref{def:longitudinal_dangerous} and
\defref{def:proper_lateral} for the lateral case. 
\begin{definition}[Dangerous Lateral Situation and Danger Threshold time] \label{def:lateral_dangerous}
  We say that time $t$ is \emph{laterally dangerous} for cars
  $c_1,c_2$ if the lateral distance between them at time $t$ is non safe
  (according to \defref{def:lateral_safe_distance}).
  Given a laterally dangerous time $t$, its \emph{Lateral
    Danger Threshold time}, denoted $t_b^{\mathrm{lat}}$, is the earliest laterally
  dangerous time such that all the times in the interval $[t_b^{\mathrm{lat}},t]$ are
  laterally dangerous. In particular, an accident can only happen
  at time $t$ if it is laterally dangerous, and in that case we
  say that the laterally threshold time of the accident is the
  laterally Danger Threshold time of $t$.
\end{definition}
\begin{definition}[Lateral Proper response] \label{def:proper_lateral}
  Let $t$ be a laterally dangerous time for cars $c_1,c_2$, let $t_b^{\mathrm{lat}}$
  be the corresponding laterally Danger Threshold time, and w.l.o.g. assume that
  at that time $c_1$ was to the left of $c_2$.  The laterally proper
  response of the two cars is to comply with the following constraints
  on the lateral speed:
\begin{itemize}
\item If $t \in [t_b^{\mathrm{lat}},t_b^{\mathrm{lat}}+\rho)$ then both cars can do any lateral action as
  long as their lateral
  acceleration, $a$, satisfies $|a| \le
  a^{\mathrm{lat}}_{\max,\mathrm{accel}}$. 
\item Else, if $t \ge t_b^{\mathrm{lat}}+\rho$:
\begin{itemize}
\item Before reaching $\mu$-lateral-velocity of $0$, $c_1$ must apply
  lateral acceleration of at most
  $-a^{\mathrm{lat}}_{\min,\mathrm{brake}}$ and $c_2$ must apply
  lateral acceleration of at least
  $a^{\mathrm{lat}}_{\min,\mathrm{brake}}$
\item After reaching $\mu$-lateral-velocity of $0$, $c_1$ can have any
  non-positive $\mu$-lateral-velocity and $c_2$ can have any
  non-negative $\mu$-lateral-velocity
\end{itemize}
\end{itemize}
\end{definition}

\begin{remark} \label{rem:more_reasonable_lateral_braking} For
  simplicity, we assumed that cars can immediately switch from
  applying a ``lateral braking'' of
  $a^{\mathrm{lat}}_{\min,\mathrm{brake}}$ to being at a
  $\mu$-lateral-velocity of $0$. This might not always be possible due
  to physical properties of the car. But, the important factor in the
  definition of proper response is that from time $t_b+\rho$ to the
  time the car reaches a $\mu$-lateral-velocity of $0$, the total
  lateral distance it will pass will not be larger than the one it
  would have passed had it applied a braking of
  $a^{\mathrm{lat}}_{\min,\mathrm{brake}}$ until a full
  stop. Achieving this goal by a real vehicle is possible by first
  braking at a stronger rate and then decreasing lateral speed more
  gradually at the end. It is easy to see that this change will have
  no effect on the essence of RSS (as well as on the procedure for
  efficiently guaranteeing RSS safety that will be described in the
  next section).
\end{remark}
\begin{remark} \label{rem:all_points_are_considered} The definitions
  hold for vehicles of arbitrary shapes, by taking the worst-case with
  respect to all points of each car.  In particular, this covers
  semi-trailers or a car with an open door.
\end{remark}

\subsection{Combining Longitudinal and Lateral Proper Responses}

We next combine the longitudinal and lateral proper responses into a
single proper response. We start with the case of a multi-lane
road (typical highway situations or rural roads), where all lanes
share the same geometry. In this case we can refer to all lanes as a
single wide lane and the meaning of longitudinal and lateral position
is well defined. Cases of multiple geometries (merges, junctions,
roundabouts, etc.), and unstructured roads, are discussed in the next
subsection, where we introduce the concept of priority.

In order to have a collision between two cars, they must be both at a
non-safe longitudinal distance and at a non-safe lateral
distance. 
\begin{definition}[Dangerous Situation and Danger Threshold time] \label{def:dangerous_and_blame_time}
  We say that time $t$ is \emph{dangerous} for cars $c_1,c_2$ if it is
  both longitudinally and laterally dangerous (according to
  \defref{def:longitudinal_dangerous} and
  \defref{def:lateral_dangerous}). Given a dangerous time $t$, its
  \emph{Danger Threshold time}, denoted $t_b$, is
  $\max\{t_b^{\mathrm{long}}, t_b^{\mathrm{lat}}\}$, where
  $t_b^{\mathrm{long}}, t_b^{\mathrm{lat}}$ are the
  longitudinal/lateral Danger Threshold time, respectively. In particular, an
  accident can only happen at time $t$ if it is dangerous, and in that
  case we say that the Danger Threshold time of the accident is 
  $t$.
\end{definition}
When two cars are driving side-by-side, they are already at a non-safe
longitudinal distance. If the situation becomes dangerous, it means
that they are getting closer laterally, and the lateral situation
becomes dangerous as well. That is, $t_b = t_b^{\mathrm{lat}}$. In
this case, it makes sense that the proper response is to apply the
proper response with respect to laterally dangerous situations.
Similarly, if cars are driving one behind the other, they are already
at a non-safe lateral distance. If the situation becomes dangerous, it
means that they are getting closer longitudinal, and the longitudinal
situation becomes dangerous as well. That is,
$t_b = t_b^{\mathrm{long}}$. In this case, it makes sense that the
proper response is to apply the proper response with respect to
longitudinal dangerous situations.
This is captured in the following definition.
\begin{definition}[Basic Proper response to dangerous
  situations] \label{def:basic_proper_response} Let $t$ be a
  dangerous time for cars $c_1,c_2$ and let
  $t_b, t_b^{\mathrm{long}}, t_b^{\mathrm{lat}}$ be the corresponding
  Danger Threshold time, longitudinal Danger Threshold time, and lateral Danger Threshold time, respectively.  The
  basic proper response of the two cars is to comply with the
  following constraints on the lateral/longitudinal speed:
\begin{enumerate}
\item If $t_b = t_b^{\mathrm{long}}$ then the longitudinal speed is
  constrained according to \defref{def:proper_longitudinal}. 
\item If $t_b = t_b^{\mathrm{lat}}$ then the lateral speed is
  constrained according to \defref{def:proper_lateral}. 
\end{enumerate}
\end{definition}

\figref{fig:examples_same_geometry} illustrates the definition.

\begin{figure}[p]
\begin{center}
\begin{tabular}{ccc}
Before & Danger Threshold time & Proper Response \\ \hline & & \\ 
\begin{tikzpicture}[scale=0.4, every node/.style={transform shape}]

\fill[gray!50] (-5,-5) rectangle (5,5);
\def\x{-1}
\def\y{-3}
\coordinate (yellow) at (\x,\y);
\fill[yellow!10]  ($(-1,0.75) + (yellow)$) rectangle  ($(1,4) + (yellow)$);
\node[rotate=90] at (yellow) {\includegraphics[width=2cm]{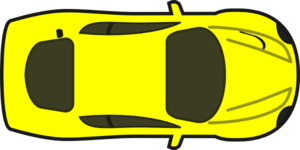}};
\foreach \offset in {-1,1}
   \draw[densely dashed,very thick,yellow] (\x+\offset,-5) -- (\x+\offset,5);

\def\x{2}
\def\y{-1.5}
\coordinate (red) at (\x,\y);
\fill[red!10]  ($(-1,0.75) + (red)$) rectangle  ($(1,4) + (red)$);
\node[rotate=90] at (red) {\includegraphics[width=2cm]{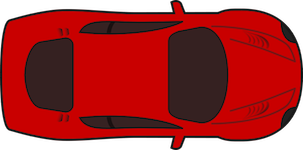}};
\foreach \offset in {-1,1}
   \draw[densely dashed,very thick,red] (\x+\offset,-5) -- (\x+\offset,5);

\end{tikzpicture} & 
\begin{tikzpicture}[scale=0.4, every node/.style={transform shape}]

\fill[gray!50] (-5,-5) rectangle (5,5);
\def\x{-1}
\def\y{-3}
\coordinate (yellow) at (\x,\y);
\fill[yellow!10]  ($(-1,0.75) + (yellow)$) rectangle  ($(1,4) + (yellow)$);
\node[rotate=90] at (yellow) {\includegraphics[width=2cm]{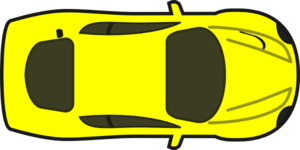}};
\foreach \offset in {-1,1}
   \draw[densely dashed,very thick,yellow] (\x+\offset,-5) -- (\x+\offset,5);

\def\x{1.5}
\def\y{-1.5}
\def\offsetleft{-1.5}
\def\offsetright{0.75}
\coordinate (red) at (\x,\y);
\fill[red!10]  ($(\offsetleft, 0.75) + (red)$) rectangle  ($(\offsetright,4) + (red)$);
\node[rotate=100] at (red) {\includegraphics[width=2cm]{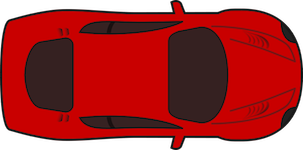}};
\foreach \offset in {\offsetleft,\offsetright}
   \draw[densely dashed,very thick,red] (\x+\offset,-5) -- (\x+\offset,5);

\end{tikzpicture} & 
\begin{tikzpicture}[scale=0.4, every node/.style={transform shape}]

\fill[gray!50] (-5,-5) rectangle (5,5);
\def\x{-1}
\def\y{-3}
\coordinate (yellow) at (\x,\y);
\fill[yellow!10]  ($(-1,0.75) + (yellow)$) rectangle  ($(1,4) + (yellow)$);
\node[rotate=90] at (yellow) {\includegraphics[width=2cm]{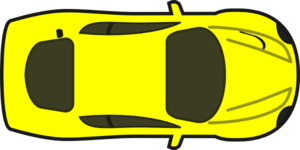}};
\foreach \offset in {-1,1}
   \draw[densely dashed,very thick,yellow] (\x+\offset,-5) -- (\x+\offset,5);

\def\x{1}
\def\y{-1.5}
\coordinate (red) at (\x,\y);
\fill[red!10]  ($(-1,0.75) + (red)$) rectangle  ($(1,4) + (red)$);
\node[rotate=90] at (red) {\includegraphics[width=2cm]{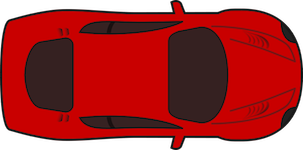}};
\foreach \offset in {-1,1}
   \draw[densely dashed,very thick,red] (\x+\offset,-5) -- (\x+\offset,5);

\end{tikzpicture}
\\ & & \\ 
\begin{tikzpicture}[scale=0.4, every node/.style={transform shape}]

\fill[gray!50] (-5,-5) rectangle (5,5);
\def\x{-1}
\def\y{0}
\coordinate (yellow) at (\x,\y);
\fill[yellow!10]  ($(-1,0.75) + (yellow)$) rectangle  ($(1,4) + (yellow)$);
\node[rotate=90] at (yellow) {\includegraphics[width=2cm]{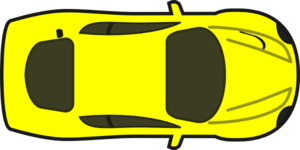}};
\foreach \offset in {-1,1}
   \draw[densely dashed,very thick,yellow] (\x+\offset,-5) -- (\x+\offset,5);

\def\x{2}
\def\y{-1}
\coordinate (red) at (\x,\y);
\fill[red!10]  ($(-1,0.75) + (red)$) rectangle  ($(1,4) + (red)$);
\node[rotate=90] at (red) {\includegraphics[width=2cm]{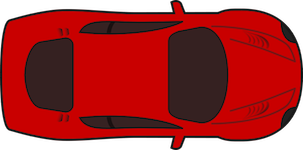}};
\foreach \offset in {-1,1}
   \draw[densely dashed,very thick,red] (\x+\offset,-5) -- (\x+\offset,5);

\end{tikzpicture} & 
\begin{tikzpicture}[scale=0.4, every node/.style={transform shape}]

\fill[gray!50] (-5,-5) rectangle (5,5);
\def\x{-1}
\def\y{0}
\coordinate (yellow) at (\x,\y);
\fill[yellow!10]  ($(-1,0.75) + (yellow)$) rectangle  ($(1,4) + (yellow)$);
\node[rotate=90] at (yellow) {\includegraphics[width=2cm]{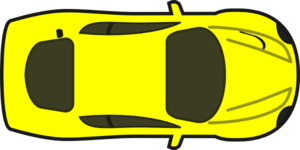}};
\foreach \offset in {-1,1}
   \draw[densely dashed,very thick,yellow] (\x+\offset,-5) -- (\x+\offset,5);

\def\x{1.5}
\def\y{-1}
\def\offsetleft{-1.5}
\def\offsetright{0.75}
\coordinate (red) at (\x,\y);
\fill[red!10]  ($(\offsetleft, 0.75) + (red)$) rectangle  ($(\offsetright,4) + (red)$);
\node[rotate=100] at (red) {\includegraphics[width=2cm]{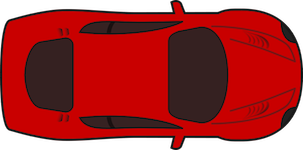}};
\foreach \offset in {\offsetleft,\offsetright}
   \draw[densely dashed,very thick,red] (\x+\offset,-5) -- (\x+\offset,5);

\end{tikzpicture} & 
\begin{tikzpicture}[scale=0.4, every node/.style={transform shape}]

\fill[gray!50] (-5,-5) rectangle (5,5);
\def\x{-1}
\def\y{0}
\coordinate (yellow) at (\x,\y);
\fill[yellow!10]  ($(-1,0.75) + (yellow)$) rectangle  ($(1,4) + (yellow)$);
\node[rotate=90] at (yellow) {\includegraphics[width=2cm]{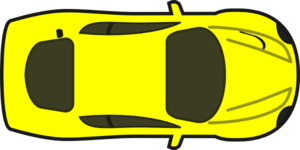}};
\foreach \offset in {-1,1}
   \draw[densely dashed,very thick,yellow] (\x+\offset,-5) -- (\x+\offset,5);

\def\x{1}
\def\y{-1}
\coordinate (red) at (\x,\y);
\fill[red!10]  ($(-1,0.75) + (red)$) rectangle  ($(1,4) + (red)$);
\node[rotate=90] at (red) {\includegraphics[width=2cm]{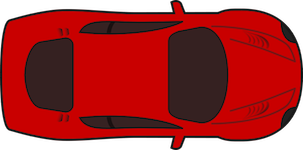}};
\foreach \offset in {-1,1}
   \draw[densely dashed,very thick,red] (\x+\offset,-5) -- (\x+\offset,5);

\end{tikzpicture}
\\ & & \\ 
\begin{tikzpicture}[scale=0.4, every node/.style={transform shape}]

\fill[gray!50] (-5,-5) rectangle (5,5);
\def\x{-1}
\def\y{-3}
\coordinate (yellow) at (\x,\y);
\fill[yellow!10]  ($(-1,0.75) + (yellow)$) rectangle  ($(1,4) + (yellow)$);
\node[rotate=90] at (yellow) {\includegraphics[width=2cm]{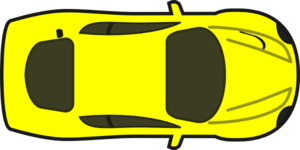}};
\foreach \offset in {-1,1}
   \draw[densely dashed,very thick,yellow] (\x+\offset,-5) -- (\x+\offset,5);

\def\x{0}
\def\y{1}
\coordinate (red) at (\x,\y);
\fill[red!10]  ($(-1,0.75) + (red)$) rectangle  ($(1,4) + (red)$);
\node[rotate=90] at (red) {\includegraphics[width=2cm]{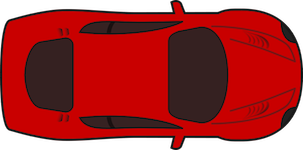}};
\foreach \offset in {-1,1}
   \draw[densely dashed,very thick,red] (\x+\offset,-5) -- (\x+\offset,5);

\end{tikzpicture} & 
\begin{tikzpicture}[scale=0.4, every node/.style={transform shape}]

\fill[gray!50] (-5,-5) rectangle (5,5);
\def\x{-1}
\def\y{-3}
\coordinate (yellow) at (\x,\y);
\fill[yellow!10]  ($(-1,0.75) + (yellow)$) rectangle  ($(1,4) + (yellow)$);
\node[rotate=90] at (yellow) {\includegraphics[width=2cm]{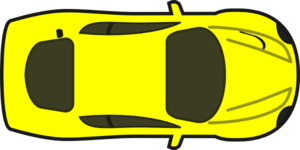}};
\foreach \offset in {-1,1}
   \draw[densely dashed,very thick,yellow] (\x+\offset,-5) -- (\x+\offset,5);

\def\x{0}
\def\y{0.25}
\coordinate (red) at (\x,\y);
\fill[red!10]  ($(-1,0.75) + (red)$) rectangle  ($(1,4) + (red)$);
\node[rotate=90] at (red) {\includegraphics[width=2cm]{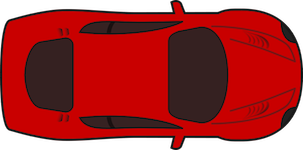}};
\foreach \offset in {-1,1}
   \draw[densely dashed,very thick,red] (\x+\offset,-5) -- (\x+\offset,5);

\end{tikzpicture} & 
\begin{tikzpicture}[scale=0.4, every node/.style={transform shape}]

\fill[gray!50] (-5,-5) rectangle (5,5);
\def\x{-1}
\def\y{-3.75}
\coordinate (yellow) at (\x,\y);
\fill[yellow!10]  ($(-1,0.75) + (yellow)$) rectangle  ($(1,4) + (yellow)$);
\node[rotate=90] at (yellow) {\includegraphics[width=2cm]{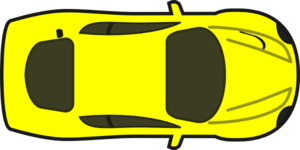}};
\foreach \offset in {-1,1}
   \draw[densely dashed,very thick,yellow] (\x+\offset,-5) -- (\x+\offset,5);

\def\x{0}
\def\y{0.25}
\coordinate (red) at (\x,\y);
\fill[red!10]  ($(-1,0.75) + (red)$) rectangle  ($(1,4) + (red)$);
\node[rotate=90] at (red) {\includegraphics[width=2cm]{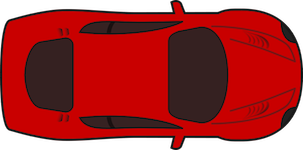}};
\foreach \offset in {-1,1}
   \draw[densely dashed,very thick,red] (\x+\offset,-5) -- (\x+\offset,5);

\end{tikzpicture}
\\ & & \\ 
\begin{tikzpicture}[scale=0.4, every node/.style={transform shape}]

\fill[gray!50] (-5,-5) rectangle (5,5);
\def\x{-1}
\def\y{-4}
\coordinate (yellow) at (\x,\y);
\fill[yellow!10]  ($(-1,0.75) + (yellow)$) rectangle  ($(1,3.5) + (yellow)$);
\node[rotate=90] at (yellow) {\includegraphics[width=2cm]{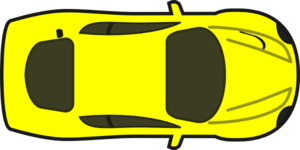}};
\foreach \offset in {-1,1}
   \draw[densely dashed,very thick,yellow] (\x+\offset,-5) -- (\x+\offset,5);

\def\x{0}
\def\y{4}
\coordinate (red) at (\x,\y);
\fill[red!10]  ($(-1,-3.5) + (red)$) rectangle  ($(1,-0.75) + (red)$);
\node[rotate=-90] at (red) {\includegraphics[width=2cm]{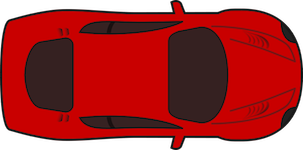}};
\foreach \offset in {-1,1}
   \draw[densely dashed,very thick,red] (\x+\offset,-5) -- (\x+\offset,5);

\end{tikzpicture} & 
\begin{tikzpicture}[scale=0.4, every node/.style={transform shape}]

\fill[gray!50] (-5,-5) rectangle (5,5);
\def\x{-1}
\def\y{-3.5}
\coordinate (yellow) at (\x,\y);
\fill[yellow!10]  ($(-1,0.75) + (yellow)$) rectangle  ($(1,3.5) + (yellow)$);
\node[rotate=90] at (yellow) {\includegraphics[width=2cm]{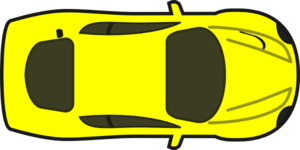}};
\foreach \offset in {-1,1}
   \draw[densely dashed,very thick,yellow] (\x+\offset,-5) -- (\x+\offset,5);

\def\x{0}
\def\y{3.5}
\coordinate (red) at (\x,\y);
\fill[red!10]  ($(-1,-3.5) + (red)$) rectangle  ($(1,-0.75) + (red)$);
\node[rotate=-90] at (red) {\includegraphics[width=2cm]{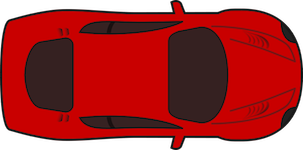}};
\foreach \offset in {-1,1}
   \draw[densely dashed,very thick,red] (\x+\offset,-5) -- (\x+\offset,5);

\end{tikzpicture} & 
\begin{tikzpicture}[scale=0.4, every node/.style={transform shape}]

\fill[gray!50] (-5,-5) rectangle (5,5);
\def\x{-1}
\def\y{-3}
\coordinate (yellow) at (\x,\y);
\fill[yellow!10]  ($(-1,0.25) + (yellow)$) rectangle  ($(1,3) + (yellow)$);
\node[rotate=90] at (yellow) {\includegraphics[width=2cm]{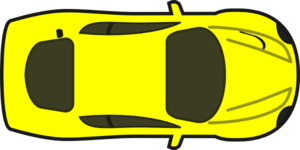}};
\foreach \offset in {-1,1}
   \draw[densely dashed,very thick,yellow] (\x+\offset,-5) -- (\x+\offset,5);

\def\x{0}
\def\y{3}
\coordinate (red) at (\x,\y);
\fill[red!10]  ($(-1,-3) + (red)$) rectangle  ($(1,-0.25) + (red)$);
\node[rotate=-90] at (red) {\includegraphics[width=2cm]{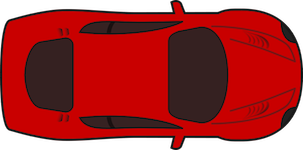}};
\foreach \offset in {-1,1}
   \draw[densely dashed,very thick,red] (\x+\offset,-5) -- (\x+\offset,5);

\end{tikzpicture}
\end{tabular}
\end{center}
\caption{\small This figure illustrates the basic proper response as
  in \defref{def:basic_proper_response}.  The vertical lines around
  each car show the possible lateral position of the car if it will
  accelerate laterally during the response time and then will brake
  laterally. Similarly, the rectangles show the possible longitudinal
  positions of the car (it it'll either brake by
  $a_{\max,\mathrm{brake}}$ or will accelerate during the response
  time and then will brake by $a_{\min,\mathrm{brake}}$). In the top
  two rows, before the Danger Threshold time there was a safe lateral distance,
  hence the proper response is to brake laterally. The yellow car is
  already at $\mu$-lateral-velocity of zero, hence only the red car
  brakes laterally. Third row: before the Danger Threshold time there was a safe
  longitudinal distance, hence the proper response is for the car
  behind to brake longitudinally. Fourth row: before the Danger Threshold time
  there was a safe longitudinal distance, in an oncoming scenario,
  hence both car should brake
  longitudinally.} \label{fig:examples_same_geometry}
\end{figure}

To strengthen the soundness of the above definition, the lemma below
shows that if all cars will respond properly then there will be no
accidents.
\begin{lemma} \label{lem:basic_utopia}
Consider a multi-lane road where all lanes share the same
geometry. Suppose that at all times, all cars on the road comply with
the basic proper response as given in
\defref{def:basic_proper_response}. Then, there will be no collisions.
\end{lemma}
\begin{proof}
  The definitions have been carefully crafted so that a similar
  inductive argument to the one given in \secref{sec:safe-distance}
  would hold.  We will prove that for any pair of cars, $c_1,c_2$,
  there is a sequence of increasing times,
  $0 = t_0 \le t_1 < t_2 < t_3 < \ldots$, such that for every time
  $t_i$, $i \ge 1$, there is no collision between $c_1$ and $c_2$ in
  the time interval $[t_{i-1},t_i]$, and at time $t_i$ the situation
  between $c_1$ and $c_2$ is not dangerous.  The basis of the
  induction is the earliest time, $t_1$, in which one of the cars is
  starting to drive. By the definition of proper response, $t_1$ is
  not a dangerous situation, and it is clear that there cannot be an
  accident from $t_0$ to $t_1$.  Suppose the inductive assumption
  holds for $t_1 < \ldots < t_i$. So, at time $t_i$, the situation is
  non-dangerous. Let $t_b$ be the earliest time after $t_i$ for which
  the situation becomes dangerous. If no such $t_b$ exists, then there
  will be no accidents in the time interval $[t_i,\infty)$ (because,
  before an accident occurs, the situation must be dangerous), hence
  we are done.  Otherwise, the definition of proper response implies
  that if both agents apply the proper response then there is a time
  $t_e \ge t_b$ for which either the situation becomes non-dangerous,
  or the relative longitudinal velocity of the two cars is
  non-positive, or the relative lateral velocity of the two cars is
  non-positive. In the former case, we set $t_{i+1} = t_e$. In the
  latter cases, we set $t_{i+1}$ to be the earliest time after $t_e$
  in which the cars are again not at a dangerous situation, and if no
  such time exists it must mean that the cars would never collide in
  the time interval $[t_e,\infty)$. In all cases, by the definitions
  of proper response and safe distance, there will be no accident in
  the time interval $[t_b,t_e]$. And, clearly, there can be no
  accidents in the time intervals $[t_i,t_b]$ and
  $[t_e,t_{i+1}]$. Hence, our inductive argument is concluded.
  Finally, it is crucial to note that the definitions of proper
  response imply no contradictions between the proper response of
  $c_1$ relatively to $c_i$ and relatively to $c_j$, for any other two
  cars $c_i,c_j$.
\end{proof}

\subsection{Compensating for improper behavior of others}

In the previous subsection we have shown that if all cars respond
properly, according to the definition of basic proper response, then
there will be no collisions. But, it may be the case that some agent
does not respond properly, yet the other agent can prevent an
accident. In such a case, it is reasonable to require that agents will
do their ``best effort'' in order to avoid dangerous situations. On the other
hand, we do not want that an attempt to avoid one accident would lead
to another accident and we do not want that a requirement to avoid all
accidents will severely harm the usefulness of the model (e.g., if the
implication would be to always be at an extremely low speed). 

We tackle the tradeoff by introducing another layer of protection as
follows. Suppose we are already at a dangerous situation with respect
to some other agent and we figure out that if the other agent will
keep its current behavior while we will keep applying proper response,
there will be a collision. For example, consider the top row of
\figref{fig:examples_same_geometry}, where we are the yellow car. Our
basic proper response is to not move laterally toward the red car and
the red car's proper response is to stop its lateral movement toward
us. Suppose the red car does not stop its lateral manoeuvre. To avoid
a collision we can either brake longitudinally or move laterally to
the left (or both). Any action for which there will be no collision,
assuming the red car will keep its current behavior, is a fine evasive
manoeuvre. Furthermore, even if a collision is not going to happen,
we still don't want to be at a dangerous situation for a long
time. Therefore, we should plan an action that will take us back to a
non-dangerous situation.

To make this formal we need some additional definitions.
\begin{definition}[Naive Prediction] \label{def:naive_prediction} The
  longitudinal or lateral state of a road agent is defined by its
  position, velocity, and acceleration, denoted by $p_0,v_0,a_0$.  The
  future state, assuming a naive prediction, is as follows. Let
  $\tau = -v_0/a_0$ in case the signs of $v_0$ and $a_0$ are opposite,
  or $\tau = \infty$ otherwise. The acceleration at time $t$ is $a_0$
  if $t \in [0,\tau]$ and $0$ otherwise. The velocity at time $t$ is
  $v_0$ plus the integral of the velocity, and the position is $p_0$
  plus the integral of the velocity. 
\end{definition}

\begin{definition}[Evasive Manoeuvre] \label{def:evasive}
Suppose that at time $t_0$ the situation is dangerous with respect to
two road agents $c_1,c_2$. An evasive manoeuvre for $c_1$, with
respect to a plan time parameter $T$, is 
a set of two functions $m^{\mathrm{long}} : [t_0,t_1]
\to \reals$ and $m^{\mathrm{lat}} : [t_0,t_1] \to \reals$, where $t_1 =
t_0+T$, such that:
\begin{itemize}
\item For $t \in [t_0,t_1]$, the longitudinal and lateral positions of
  $c_1$ is given by $m^{\mathrm{long}}(t)$ and
  $m^{\mathrm{lat}}(t)$. As a result, the longitudinal/lateral
  velocities/accelerations of $c_1$ are given by the first/second
  derivatives of the above.
\item Assume that during the time interval $[t_0,t_1]$, $c_2$ will
  drive according to its naive prediction (as given in
  \defref{def:naive_prediction}) and $c_1$ will drive according to
  $m^{\mathrm{long}}(t)$ and $m^{\mathrm{lat}}(t)$, then the two cars
  will not collide in the time interval $[t_0,t_1]$ and $t_1$ will no
  longer be dangerous time.
\end{itemize} 
We say that the evasive manoeuvre is legal if in addition to the
above, it holds that:
\begin{itemize}
\item The longitudinal and lateral accelerations of $c_1$ at time
$t_0$ satisfy the basic proper response constraints of $c_1$ (as given
in \defref{def:basic_proper_response}) with respect to all other
agents. 
\item The absolute value of the lateral acceleration according to $m^{\mathrm{lat}}$ is
  always at most $a^{\mathrm{lat}}_{\max,\mathrm{accel}}$ and the
  longitudinal acceleration according to $m^{\mathrm{long}}$ is always in the interval
  $[-a_{\max,\mathrm{brake}}, a_{\max,\mathrm{accel}}]$. 
\end{itemize} 
\end{definition}

We now extend \defref{def:basic_proper_response} to include the common
sense rule of ``if you can avoid an accident without causing another accident, you
must do it''.
\begin{definition}[Proper Response with Extra Evasive
  Effort] \label{def:proper_response_evasive} Let $t$ be a dangerous
  time for cars $c_1,c_2$ and let $t_b$ be the corresponding blame
  time. Assume that during the time interval $[t_b,t]$, car $c_2$ did
  not comply with the proper response constraints as given in
  \defref{def:basic_proper_response}. Assume also that there exists a
  legal evasive manoeuvre for $c_1$ as defined in
  \defref{def:evasive}. Then, the proper response of
  $c_1$ is to apply a legal evasive manoeuvre (in addition to the
  proper response constraints as given in
  \defref{def:basic_proper_response}). 
\end{definition}

By requiring that the evasive manoeuvres would never contradict those of
\defref{def:basic_proper_response}, we make sure that they would not
cause another accident. Clearly, the proof of \lemref{lem:basic_utopia} still holds
even if cars apply proper response as in \defref{def:proper_response_evasive}.

\subsection{Multiple Geometry and Right-of-Way Rules} \label{sec:lane-priority}

\begin{figure}
\begin{center}
\begin{subfigure}[t]{0.22\textwidth}
\includegraphics[width=0.8\textwidth, height=2cm]{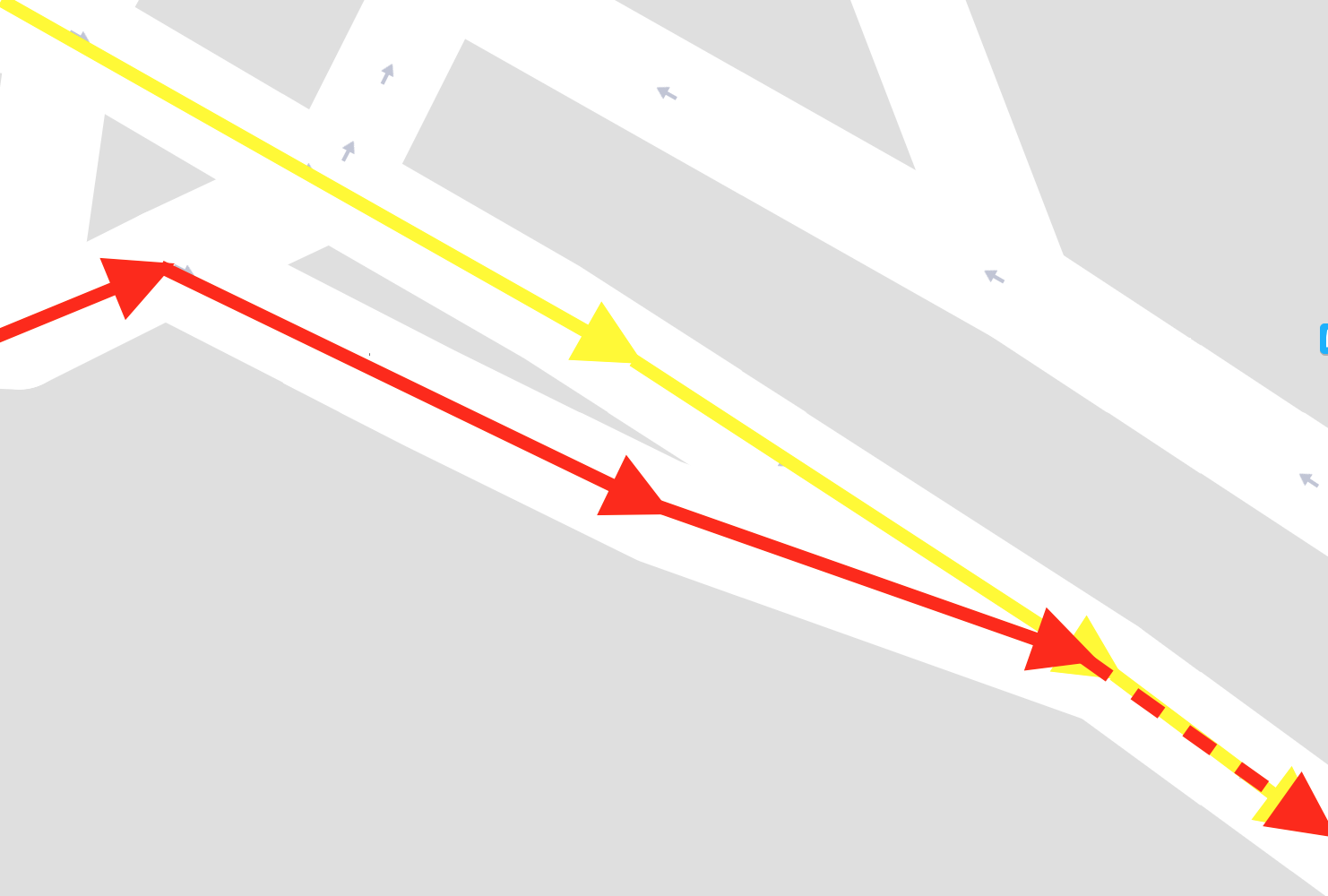}
\caption{}\label{fig:scenarios:merge}
\end{subfigure}
~
\begin{subfigure}[t]{0.22\textwidth}
\includegraphics[width=0.8\textwidth, height=2cm]{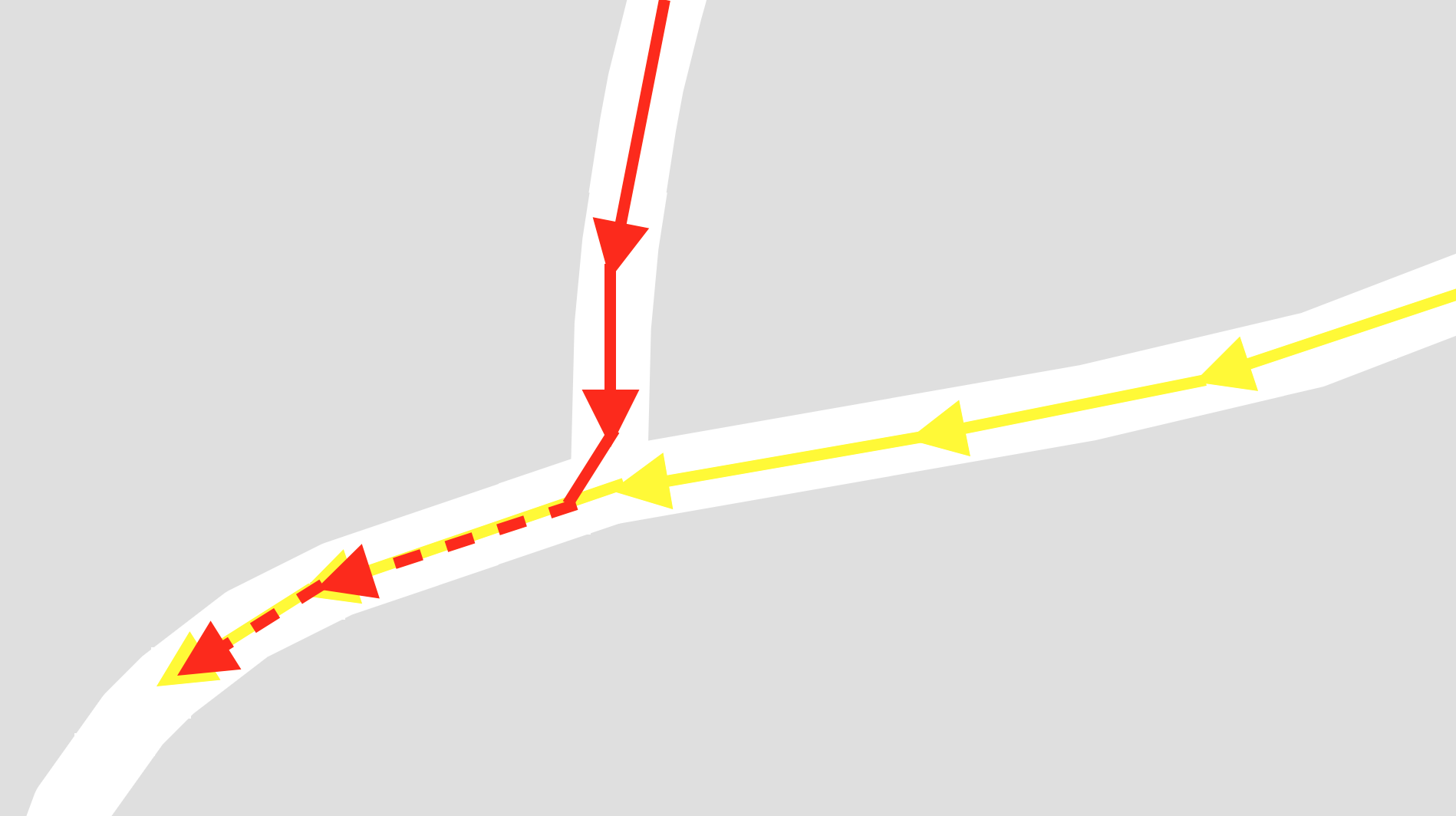}
\caption{}\label{fig:scenarios:t}
\end{subfigure}
~
\begin{subfigure}[t]{0.22\textwidth}
\includegraphics[width=0.8\textwidth, height=2cm]{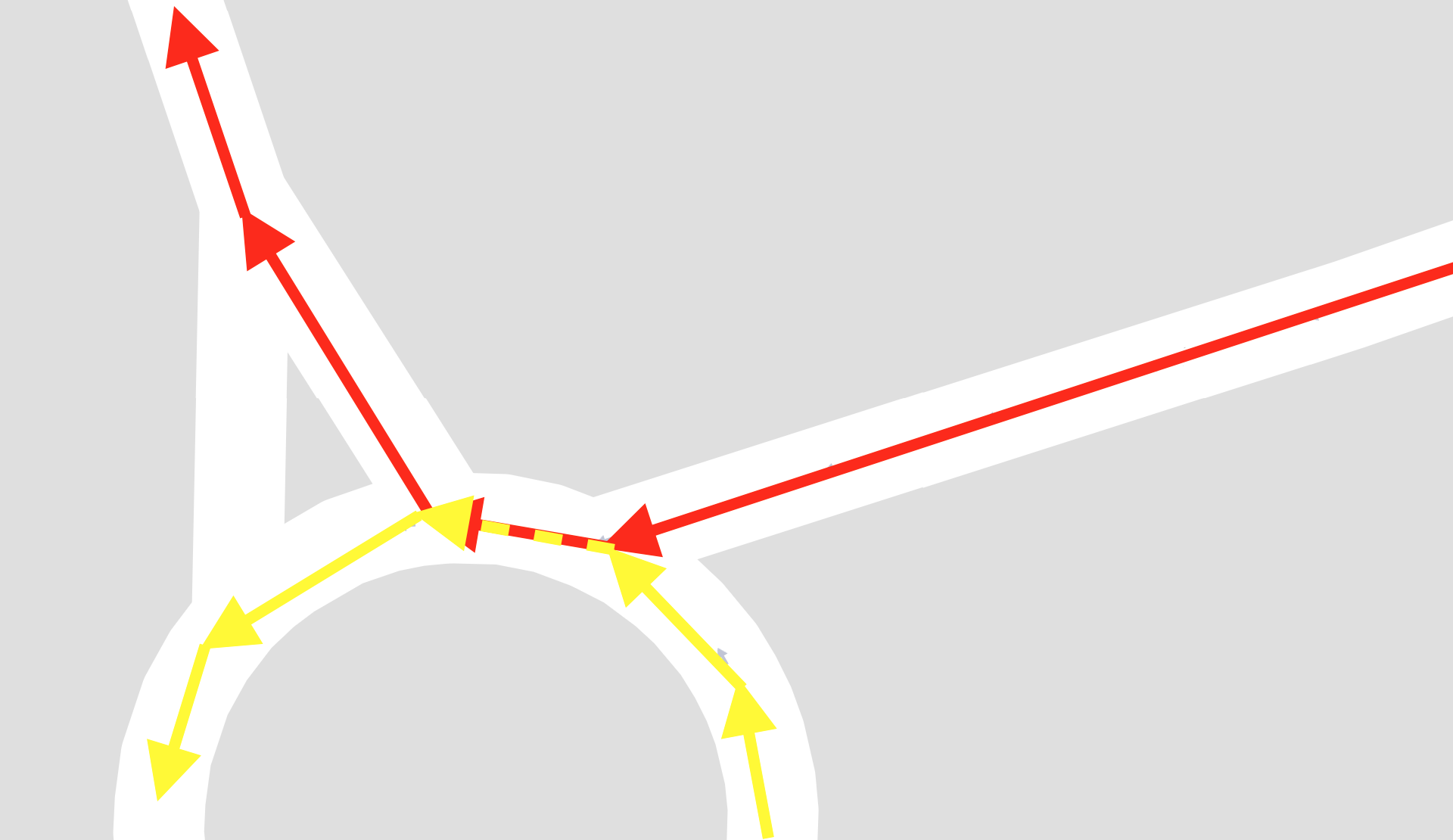}
\caption{}
\end{subfigure}
~
\begin{subfigure}[t]{0.22\textwidth}
\includegraphics[width=0.8\textwidth, height=2cm]{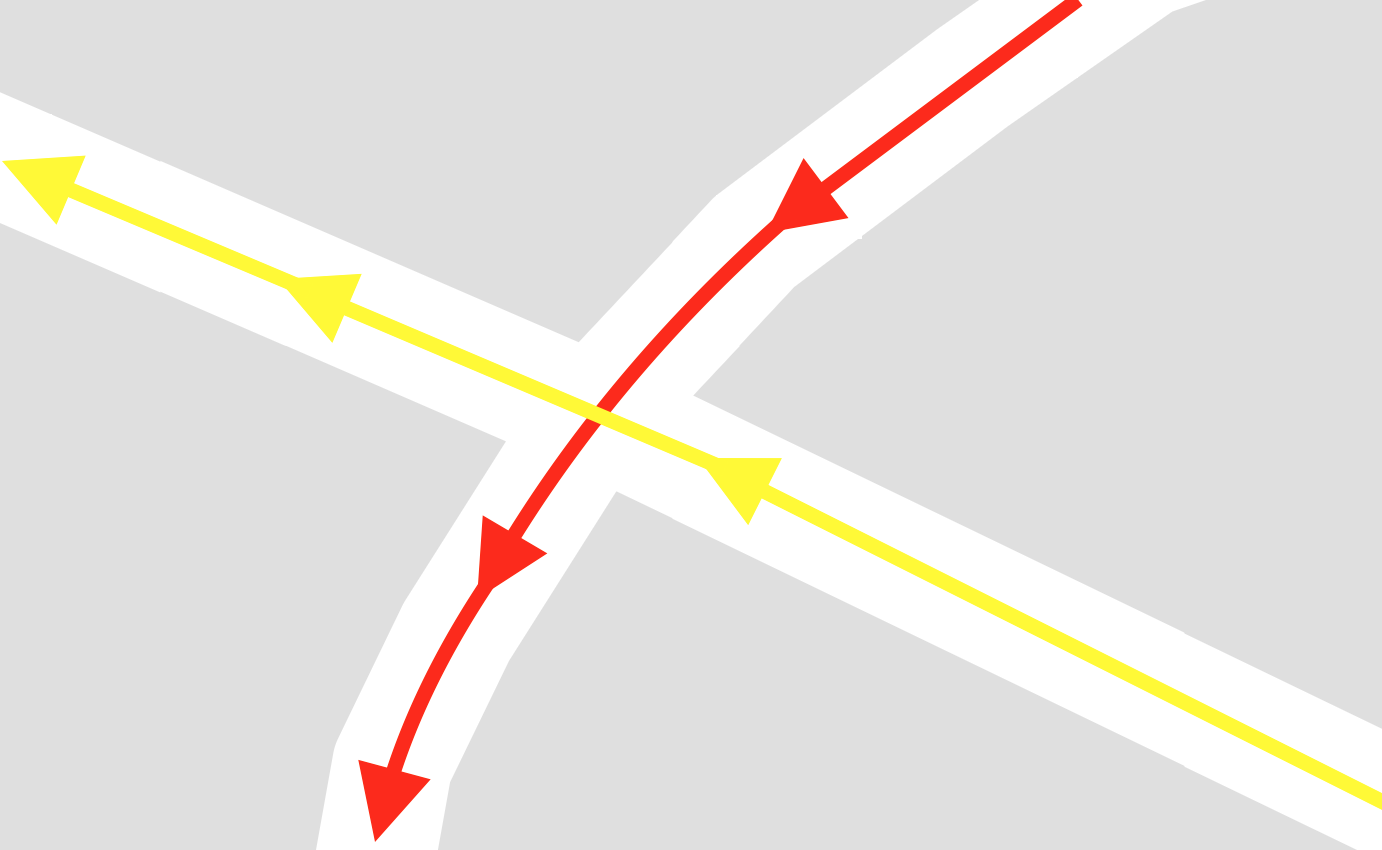}
\caption{}
\end{subfigure}
\end{center}
\caption{Different examples for multiple routes scenarios. In yellow, the prioritized route. In red, the secondary route.}\label{fig:scenarios}
\end{figure}

We next turn to deal with scenarios in which there are multiple
different road geometries in one scene that overlap in a certain
area. Examples include roundabouts, junctions, and merge into
highways. See \figref{fig:scenarios} for illustration. In many such
cases, one route has priority over others, and vehicles riding on it
have the \emph{right of way}. 

In the previous subsections we could assume that the route is straight, by
relying on \secref{sec:lane_based} that shows how to construct a
bijection between a general lane geometry and a straight road, with
a coherent meaning for longitudinal and lateral axes. When facing
scenarios of multiple route geometries, the definitions should be
adjusted. First, two routes with different geometries can yield
conflicts in the constraints of the proper response. An example is
given in \figref{fig:proper_conflict}. 

\begin{figure}
\begin{center}
\begin{tabular}{cc} 
Different geometry & Same geometry \\
\begin{minipage}{0.5\linewidth}
\begin{center}
\begin{tikzpicture}[scale=0.4, every node/.style={transform shape}]

\fill[gray!50] (-5,-5) rectangle (5,5);

\draw[white, line width=0.75cm] (0,-4.5) -- (0,4.5);
\draw[white, line width=0.5cm] (0,-4.5) arc (0:60:9);

\begin{scope}[very thick,decoration={markings,mark=at position 0.25 with {\arrow{>}},mark=at position 0.5 with {\arrow{>}},mark=at position 0.75 with {\arrow{>}},mark=at position 0.99 with {\arrow{>}}}] 
\draw[red,densely dashed,very thick, postaction={decorate}] (0,-4.5) arc (0:60:9);
\draw[yellow,densely dashed,very thick, postaction={decorate}] (0,-4.5) -- (0,4.5);
\end{scope}
\node[rotate=100] at (1.3,-3.5) {\includegraphics[width=2cm]{red_car}};
\node[rotate=80] at (-1.3,-3.5) {\includegraphics[width=2cm]{yellow_car}};
\node[rotate=90] at (0,-3.5) {\includegraphics[width=2cm]{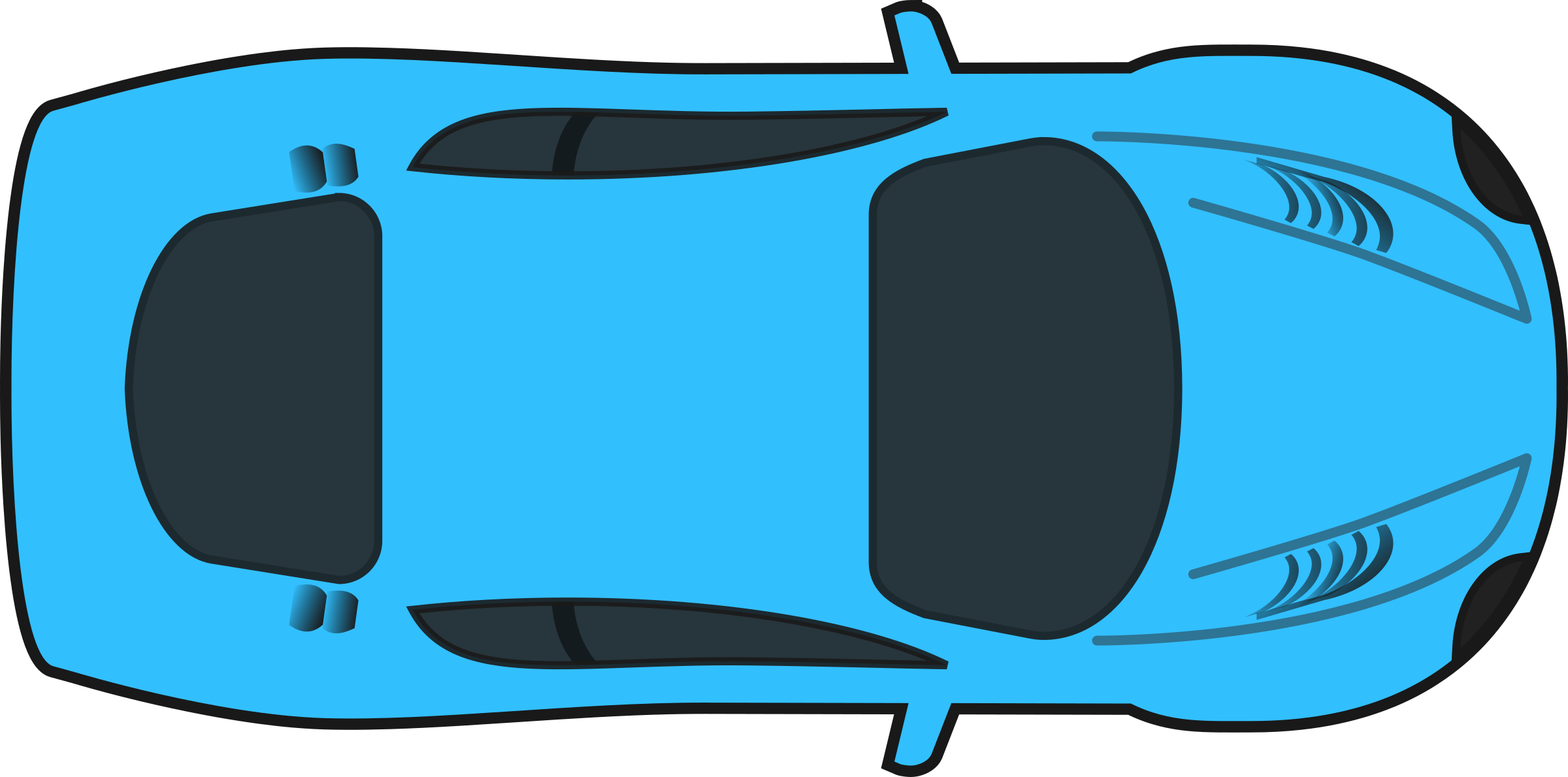}};

\end{tikzpicture} 
\end{center}
\end{minipage}
  & 
\begin{minipage}{0.5\linewidth}
\begin{center}
\begin{tikzpicture}[scale=0.4, every node/.style={transform shape}]

\fill[gray!50] (-5,-5) rectangle (5,5);

\draw[white, line width=1.5cm] (0,-4.5) -- (0,4.5);

\begin{scope}[very thick,decoration={markings,mark=at position 0.25 with {\arrow{>}},mark=at position 0.5 with {\arrow{>}},mark=at position 0.75 with {\arrow{>}},mark=at position 0.99 with {\arrow{>}}}] 
\draw[red,densely dashed,very thick, postaction={decorate}] (1,-4.5) -- (1,4.5);
\draw[yellow,densely dashed,very thick, postaction={decorate}] (-1,-4.5) -- (-1,4.5);
\end{scope}
\node[rotate=100] at (1.3,-3.5) {\includegraphics[width=2cm]{red_car}};
\node[rotate=80] at (-1.3,-3.5) {\includegraphics[width=2cm]{yellow_car}};
\node[rotate=90] at (0,-3.5) {\includegraphics[width=2cm]{blue_car}};

\end{tikzpicture} 
\end{center}
\end{minipage}
\end{tabular}
\end{center}
\caption{A proper response conflict due to multiple geometry. Right:
  if all routes share the same geometry, there are no conflicts
  between the proper response constraints with regard to different
  cars. Left: the proper response of the blue car with respect to the
  yellow car and yellow route is to continue straight according to the
  yellow route, while the proper response of the blue car with respect
  to the red car and the red route is to continue driving on the
  center of the red route. These two constraints contradict each
  other.} \label{fig:proper_conflict}
\end{figure}

Moreover, consider the T-junction depicted on
\figref{fig:scenarios:t}, and suppose that there is a stop sign for
the red route. Suppose that $c_1$ is approaching the intersection on
the yellow route and at the same time $c_2$ is approaching the
intersection on the red route. According to the yellow route's
coordinate system, $c_2$ has a very large lateral velocity, hence
$c_1$ might deduce that $c_2$ is already at a non-safe lateral
distance, which implies that $c_1$, driving on the prioritized route,
must reduce speed in order to maintain a safe longitudinal distance to
$c_2$. This means that $c_2$ should be very conservative
w.r.t. traffic that coming from the red route. This is of course an
unnatural behavior, as cars on the yellow route have the right-of-way
in this case. Furthermore, even $c_2$, who doesn't have the priority,
should be able to merge into the junction as long as $c_1$ can stop in
time (this will be crucial in dense traffic). This example shows that
when $c_1$ drives on $r_1$, it doesn't make sense to consider its
position and velocity w.r.t the coordinate system of $r_2$. As a
result, we need to generalize basic notions from previous subsections
such as ``what does it mean that $c_1$ is in front of $c_2$'', and
what does it mean to be at a non safe distance.

\begin{remark}
  The definitions below assume that two cars, $c_1,c_2$ are driving on
  different routes, $r_1,r_2$. We emphasize that in some situations
  (for example, the T-junction given in \figref{fig:scenarios:t}),
  once there is exactly a single route $r_1$ such that both cars are
  assigned to it, and the time is not dangerous, then from that moment
  on, the definitions are solely w.r.t. $r_1$.
\end{remark}

We start with generalizing the definition of safe lateral distance and
lateral proper response. It is not hard to verify that applying the
definition below to two routes of the same geometry indeed yields the
same definition as in \defref{def:lateral_safe_distance}. Throughout
this section, we sometimes refer to a route as a subset of $\reals^2$.
\begin{definition}[Lateral Safe Distance for Two Routes of Different
  Geometry] \label{def:safe_lat_geometry} Consider vehicles $c_1,c_2$
  driving on routes $r_1,r_2$ that intersect.  For every
  $i \in\{1,2\}$, let $[x_{i,\min},x_{i,\max}]$ be the minimal and
  maximal lateral positions in $r_i$ that $c_i$ can be in, if during
  the time interval $[0,\rho)$ it will apply a lateral acceleration
  (w.r.t. $r_i$) s.t.
  $|a^{\mathrm{lat}}| \le a^{\mathrm{lat}}_{\max,\mathrm{accel}}$, and
  after that it will apply a lateral braking of at least
  $a^{\mathrm{lat}}_{\min,\mathrm{brake}}$ (again w.r.t. $r_i$), until
  reaching a zero lateral velocity (w.r.t. $r_i$).  The lateral
  distance between $c_1$ and $c_2$ is safe if the
  restrictions\footnote{The restriction of $r_i$ to the lateral intervals
  $[x_{i,\min},x_{i,\max}]$ is the subset of $\reals^2$ obtained by
  all points $(x,y) \in r_i$ for which the lateral position in lane coordinates
  of $(x,y)$ (as defined in \secref{sec:lane_based}) is in the interval $[x_{i,\min},x_{i,\max}]$.} of
  $r_1,r_2$ to the lateral intervals
  $[x_{1,\min},x_{1,\max}], [x_{2,\min},x_{2,\max}]$ are at a
  distance\footnote{The distance between sets $A,B$ is $\min\,\{\|a-b\|
      : a \in A, b \in B\}$.}
  of at least $\mu$. 
\end{definition}

The definition of laterally dangerous time and Danger Threshold time is exactly
as in  \defref{def:lateral_dangerous}, except that safe lateral
distance is according to \defref{def:safe_lat_geometry}. We next
define lateral proper response. 

\begin{definition}[Lateral Proper response for Two Routes of Different
  Geometry] \label{def:proper_lateral_geometry}
  Let $t$ be a laterally dangerous time for cars $c_1,c_2$, driving on
  routes $r_1,r_2$, let $t_b^{\mathrm{lat}}$
  be the corresponding laterally Danger Threshold time, and let $x_1,x_2$ be the
  lateral positions of $c_1,c_2$ at time $t_b^{\mathrm{lat}}$
  w.r.t. routes $r_1,r_2$, respectively. The laterally proper
  response of the two cars is to comply with the following constraints
  on the lateral speed:
\begin{itemize}
\item If $t \in [t_b^{\mathrm{lat}},t_b^{\mathrm{lat}}+\rho)$ then both cars can do any lateral action as
  long as their lateral
  acceleration, $a$, satisfies $|a| \le
  a^{\mathrm{lat}}_{\max,\mathrm{accel}}$. 
\item Else, if $t \ge t_b^{\mathrm{lat}}+\rho$, for every $i \in \{1,2\}$:
\begin{itemize}
\item Before reaching $\mu$-lateral-velocity of $0$, $c_i$ must apply
  lateral acceleration that will take it toward $x_i$ whose magnitude
  value is at least
  $|a^{\mathrm{lat}}_{\min,\mathrm{brake}}|$
\item After reaching $\mu$-lateral-velocity of $0$, $c_i$ can have any
  lateral acceleration that will take it away from the other car
\end{itemize}
\end{itemize}
\end{definition}

Before we define longitudinal safe distance, we need to quantify
ordering between cars when no common longitudinal axis exists.
\begin{definition}[Longitudinal Ordering for Two Routes of Different
  Geometry] \label{def:in_front}
Consider $c_1,c_2$ driving on routes $r_1,r_2$ that intersect. We say
that $c_1$ is longitudinally in front of $c_2$ if either of the
following holds:
\begin{enumerate}
\item For every $i$, if both vehicles are on $r_i$ then $c_1$ is in front
  of $c_2$ according to $r_i$
\item $c_1$ is outside $r_2$ and $c_2$ is outside $r_1$, and the
  longitudinal distance from $c_1$ to the set $r_1 \cap r_2$,
  w.r.t. $r_1$, is smaller than the longitudinal distance from $c_2$ to
  the set $r_1 \cap r_2$, w.r.t. $r_2$.
\end{enumerate}
\end{definition}
\begin{remark}
One may worry that the longitudinal ordering definition is not robust,
for example, in item (2) of the definition, suppose that $c_1,c_2$ are at 
distances of $20,20.1$ meters, respectively, from the
intersection. This is not an issue as this definition is effectively
being used only when there is a safe longitudinal distance between the
two cars, and in that case the ordering between the cars will be
obvious. Furthermore, this is exactly analogous to the non-robustness of
ordering when two cars are driving side by side on a multi-lane
highway road. 
\end{remark}
An illustration of the ordering definition is given in
\figref{fig:ordering}.
\begin{figure}[t]
\begin{center}
\begin{tabular}{cc}
Yellow is in front & Yellow is in front \\
\begin{tikzpicture}[scale=0.4, every node/.style={transform shape}]

\fill[gray!50] (-5,-5) rectangle (7,3);

\draw[white, line width=0.5cm] (7,0) -- (0,0);
\draw[white, line width=0.5cm] (7,-4) .. controls (3,-4) and (2,0) .. (0,0) -- (-4.5,0);

\begin{scope}[very thick,decoration={markings,mark=at position 0.25 with {\arrow{>}},mark=at position 0.5 with {\arrow{>}},mark=at position 0.75 with {\arrow{>}},mark=at position 0.99 with {\arrow{>}}}] 
\draw[red,densely dashed,very thick, postaction={decorate}] (7,-4) .. controls (3,-4) and (2,0) .. (0,0) -- (-4.5,0);
\draw[yellow,densely dashed,very thick, postaction={decorate}] (7,0) -- (0,0);
\end{scope}

\node[rotate=155] at (5,-3.5) {\includegraphics[width=2cm]{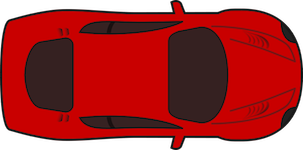}};
\node[rotate=180] at (4,0) {\includegraphics[width=2cm]{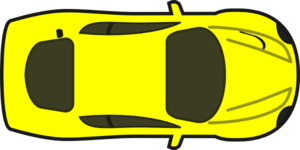}};

\end{tikzpicture} & 
\begin{tikzpicture}[scale=0.4, every node/.style={transform shape}]

\fill[gray!50] (-5,-5) rectangle (7,3);

\draw[white, line width=0.5cm] (7,0) -- (0,0);
\draw[white, line width=0.5cm] (7,-4) .. controls (3,-4) and (2,0) .. (0,0) -- (-4.5,0);

\begin{scope}[very thick,decoration={markings,mark=at position 0.25 with {\arrow{>}},mark=at position 0.5 with {\arrow{>}},mark=at position 0.75 with {\arrow{>}},mark=at position 0.99 with {\arrow{>}}}] 
\draw[red,densely dashed,very thick, postaction={decorate}] (7,-4) .. controls (3,-4) and (2,0) .. (0,0) -- (-4.5,0);
\draw[yellow,densely dashed,very thick, postaction={decorate}] (7,0) -- (0,0);
\end{scope}

\node[rotate=155] at (5,-3.5) {\includegraphics[width=2cm]{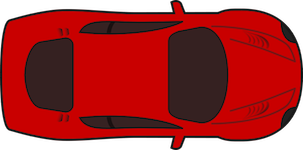}};
\node[rotate=180] at (0,0) {\includegraphics[width=2cm]{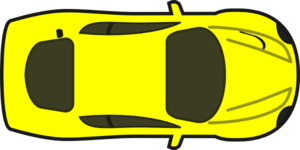}};

\end{tikzpicture} \\ & \\
Yellow is in front & Nobody is in front \\
\begin{tikzpicture}[scale=0.4, every node/.style={transform shape}]

\fill[gray!50] (-5,-5) rectangle (7,3.5);

\draw[white, line width=0.5cm] (7,-1) -- (-4.5,-1);
\draw[white, line width=0.5cm] (7,-4) .. controls (3,-4) and (2,0) .. (-4,3);

\begin{scope}[very thick,decoration={markings,mark=at position 0.25 with {\arrow{>}},mark=at position 0.5 with {\arrow{>}},mark=at position 0.75 with {\arrow{>}},mark=at position 0.99 with {\arrow{>}}}] 
\draw[red,densely dashed,very thick, postaction={decorate}] (7,-4) .. controls (3,-4) and (2,0) .. (-4,3);
\draw[yellow,densely dashed,very thick, postaction={decorate}] (-4.5,-1) -- (7,-1);
\end{scope}
\node[rotate=155] at (5,-3.5) {\includegraphics[width=2cm]{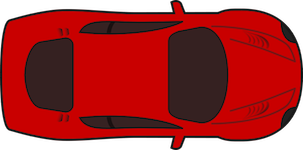}};
\node[rotate=0] at (0,-1) {\includegraphics[width=2cm]{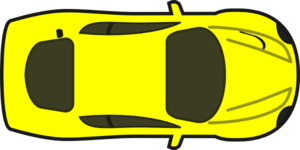}};

\end{tikzpicture}  &
\begin{tikzpicture}[scale=0.4, every node/.style={transform shape}]

\fill[gray!50] (-5,-5) rectangle (7,3.5);

\draw[white, line width=0.5cm] (7,-1) -- (-4.5,-1);
\draw[white, line width=0.5cm] (7,-4) .. controls (3,-4) and (2,0) .. (-4,3);

\begin{scope}[very thick,decoration={markings,mark=at position 0.25 with {\arrow{>}},mark=at position 0.5 with {\arrow{>}},mark=at position 0.75 with {\arrow{>}},mark=at position 0.99 with {\arrow{>}}}] 
\draw[red,densely dashed,very thick, postaction={decorate}] (7,-4) .. controls (3,-4) and (2,0) .. (-4,3);
\draw[yellow,densely dashed,very thick, postaction={decorate}] (-4.5,-1) -- (7,-1);
\end{scope}
\node[rotate=140] at (2.75,-2) {\includegraphics[width=2cm]{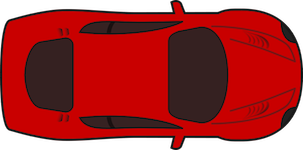}};
\node[rotate=0] at (0,-1) {\includegraphics[width=2cm]{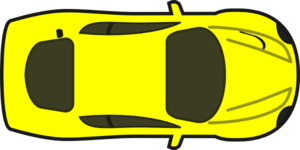}};

\end{tikzpicture}  
\end{tabular}
\end{center}
\caption{Illustration of \defref{def:in_front}.} \label{fig:ordering}
\end{figure}

\begin{definition}[Longitudinal Safe Distance for Two Routes of
  Different Geometry] \label{def:safe_long_geometry}
Consider $c_1,c_2$ driving on routes $r_1,r_2$ that intersect. The
longitudinal distance between $c_1$ and $c_2$ is safe if one of the
following holds:
\begin{enumerate}
\item If for all $i \in \{1,2\}$ s.t. $r_i$ has no priority, if $c_i$
  will accelerate by $a_{\max,\mathrm{accel}}$ for $\rho$ seconds, and
  will then brake by $a_{\min,\mathrm{brake}}$ until reaching zero
  longitudinal velocity (all w.r.t. $r_i$), then during this time
  $c_i$ will remain outside of the other route.
\item Otherwise, if $c_1$ is in front of $c_2$ (according to
  \defref{def:in_front}), then they are at a safe longitudinal distance if in case
  $c_1$ will brake by $a_{\max,\mathrm{brake}}$ until reaching a zero
  velocity (w.r.t. $r_1$), and $c_2$ will accelerate by at most
  $a_{\max,\mathrm{accel}}$ for $\rho$ seconds and then will brake by
  at least $a_{\min,\mathrm{brake}}$ (w.r.t. $r_2$) until reaching a
  zero velocity, then $c_1$ will remain in front of $c_2$ (according
  to \defref{def:in_front}).
\item Otherwise, consider a point $p \in r_1 \cap r_2$ s.t. for
  $i \in \{1,2\}$, the lateral position of $p$ w.r.t. $r_i$ is in
  $[x_{i,\min},x_{i,\max}]$ (as defined in
  \defref{def:safe_lat_geometry}). Let $[t_{i,\min},t_{i,\max}]$ be
  all times s.t. $c_i$ can arrive to the longitudinal position of $p$
  w.r.t. $r_i$ if it will apply longitudinal accelerations in the
  range $[-a_{\max,\mathrm{brake}}, a_{\max,\mathrm{accel}}]$ during
  the first $\rho$ seconds, and then will apply longitudinal braking
  in the range $[a_{\min,\mathrm{brake}}, a_{\max,\mathrm{brake}}]$
  until reaching a zero velocity. Then, the vehicles are at a safe
  longitudinal distance if for every such $p$ we have that
  $[t_{1,\min},t_{1,\max}]$ does not intersect
  $[t_{2,\min},t_{2,\max}]$.
\end{enumerate}
\end{definition}
Illustrations of the definition is given in
\figref{fig:safe_long_geometry}.

The definition of longitudinally dangerous time and Danger Threshold time is exactly
as in  \defref{def:longitudinal_dangerous}, except that safe longitudinal
distance is according to \defref{def:safe_long_geometry}.

\begin{definition}[Longitudinal Proper Response for Routes of Different
  Geometry] \label{def:proper_multi_geometry} Suppose that time $t$ is
  longitudinally dangerous for vehicles
  $c_1,c_2$ driving on routes $r_1,r_2$ The longitudinal proper
  response depends on the situation
  immediately before the Danger Threshold time:
\begin{itemize}
\item If the longitudinal distance was safe according to item
  (1) in \defref{def:safe_long_geometry}, then if a vehicle is on the
  prioritized route it can drive normally, and otherwise it must brake
  by at least $a_{\min,\mathrm{brake}}$ if $t-t_b \ge \rho$. 
\item Else,  if the longitudinal distance was safe according to item
  (2) in \defref{def:safe_long_geometry}, then $c_1$ can drive
  normally and $c_2$ must brake by at least $a_{\min,\mathrm{brake}}$
  if $t-t_b \ge \rho$. 
\item Else, if the longitudinal distance was safe according to item
  (3) in \defref{def:safe_long_geometry}, then both cars can drive
  normally if $t-t_b < \rho$, and otherwise, both cars should brake
  laterally and longitudinally by at least
  $a^{\mathrm{lat}}_{\min,\mathrm{brake}}, a_{\min,\mathrm{brake}}$
  (each one w.r.t. its own route).
\end{itemize}
\end{definition}

Finally, the analogues of \defref{def:basic_proper_response}, 
\defref{def:dangerous_and_blame_time}, and
\defref{def:proper_response_evasive} are straightforward.  

\begin{figure}[t]
\begin{center}
\begin{subfigure}[b]{0.28\textwidth}
\begin{tikzpicture}[scale=0.4, every node/.style={transform shape}]

\fill[gray!50] (-5,-5) rectangle (7,3);

\draw[white, line width=0.5cm] (7,0) -- (0,0);
\draw[white, line width=0.5cm] (7,-4) .. controls (3,-4) and (2,0) .. (0,0) -- (-4.5,0);

\begin{scope}[very thick,decoration={markings,mark=at position 0.25 with {\arrow{>}},mark=at position 0.5 with {\arrow{>}},mark=at position 0.75 with {\arrow{>}},mark=at position 0.99 with {\arrow{>}}}] 
\draw[red,densely dashed,very thick, postaction={decorate}] (7,-4) .. controls (3,-4) and (2,0) .. (0,0) -- (-4.5,0);
\draw[yellow,densely dashed,very thick, postaction={decorate}] (7,0) -- (0,0);
\end{scope}

\node[rotate=155] at (5,-3.5) {\includegraphics[width=2cm]{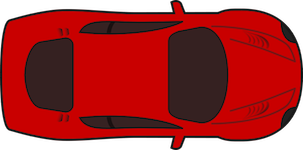}};
\node[rotate=180] at (5.5,0) {\includegraphics[width=2cm]{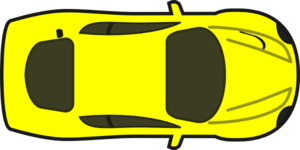}};

\node[regular polygon, regular polygon sides=8,
      draw=red, double, double distance=0.5mm, ultra thick,
      fill=red,text=white,inner sep=0mm,text width=8mm] at (0.25,-1.5)
      {STOP};

\end{tikzpicture}
\caption{Safe because yellow has priority and red can stop before
  entering the intersection.} \label{fig:A}
\end{subfigure} 
\hspace{1cm}
\begin{subfigure}[b]{0.28\textwidth}
\begin{tikzpicture}[scale=0.4, every node/.style={transform shape}]

\fill[gray!50] (-5,-5) rectangle (7,3);

\draw[white, line width=0.5cm] (7,0) -- (0,0);
\draw[white, line width=0.5cm] (7,-4) .. controls (3,-4) and (2,0) .. (0,0) -- (-4.5,0);

\begin{scope}[very thick,decoration={markings,mark=at position 0.25 with {\arrow{>}},mark=at position 0.5 with {\arrow{>}},mark=at position 0.75 with {\arrow{>}},mark=at position 0.99 with {\arrow{>}}}] 
\draw[red,densely dashed,very thick, postaction={decorate}] (7,-4) .. controls (3,-4) and (2,0) .. (0,0) -- (-4.5,0);
\draw[yellow,densely dashed,very thick, postaction={decorate}] (7,0)
-- (0,0);
\end{scope}

\node[rotate=140] at (3.25,-2.5) {\includegraphics[width=2cm]{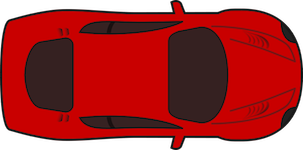}};
\node[rotate=180] at (0,0) {\includegraphics[width=2cm]{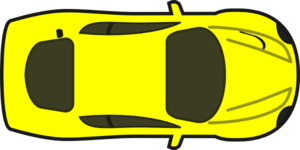}};

\end{tikzpicture}
\caption{Safe because yellow is in front of red, and if yellow will
  brake, red can brake as well and avoid a collision.} \label{fig:B}
\end{subfigure} \hspace{1cm}
\begin{subfigure}[b]{0.28\textwidth}
\begin{tikzpicture}[scale=0.4, every node/.style={transform shape}]

\fill[gray!50] (-5,-5) rectangle (7,3);

\draw[white, line width=0.5cm] (7,-1) -- (-4.5,-1);
\draw[white, line width=0.5cm] (7,-4) .. controls (3,-4) and (0,0) .. (0,2.8);

\begin{scope}[very thick,decoration={markings,mark=at position 0.25 with {\arrow{>}},mark=at position 0.5 with {\arrow{>}},mark=at position 0.75 with {\arrow{>}},mark=at position 0.99 with {\arrow{>}}}] 
\draw[red,densely dashed,very thick, postaction={decorate}] (7,-4) .. controls (3,-4) and (0,0) .. (0,2.8);
\draw[yellow,densely dashed,very thick, postaction={decorate}]
(-4.5,-1) -- (7,-1);
\end{scope}

\node[rotate=140] at (2.85,-2) {\includegraphics[width=2cm]{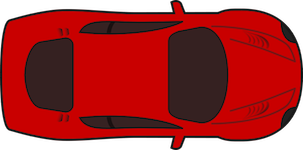}};
\node[rotate=0] at (-0.25,-1) {\includegraphics[width=2cm]{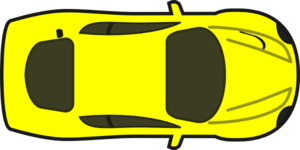}};

\end{tikzpicture}  
\caption{If yellow is at a full stop and red is at a full lateral stop, safe by item (3) of \defref{def:safe_long_geometry}.} \label{fig:C}
\end{subfigure}
\end{center}
\caption{Illustration of safe longitudinal distance (\defref{def:safe_long_geometry})} \label{fig:safe_long_geometry}
\end{figure}

\begin{remark}[No contradictions and star-shape calculations
  Revisited] \label{rem:no-contradictions_geometry} The updated
  definition of proper response is with respect to a pair of cars
  riding on a pair of routes. As before, we need to consider the
  proper response of our car with respect to each other car
  individually. That is, here again we adopt a \emph{star-shape}
  calculation. Since the proper response with respect to every other
  car is translated to a lateral and longitudinal braking constraints
  with respect to \emph{our route}, there can be no conflicts between
  proper responses with respect to different agents. For example, in
  the situation depicted on the left of \figref{fig:proper_conflict},
  the proper response with respect to the red car is a longitudinal
  brake, and hence it does not contradict the proper response of
  lateral brake with respect to the yellow car.  As a result, it is
  easy to verify that our inductive proof sill holds.  Finally, note
  that there are cases where the route used by another agent is
  unknown: for example, see \figref{fig:routes}. In such case, every
  agent should comply with the proper response obtained by checking
  all possibilities.
\end{remark}

\begin{figure}
\begin{center}
\def\RL{4}
\def\LW{2}
\def\height{2}
\def\width{1}
\def\theta{5}
\begin{tikzpicture}[scale=0.6]
\coordinate (C) at (1.5*\LW, 0.5*\LW);
\fill[gray!50] (-\LW,-\RL) rectangle (2*\LW,\RL); 
\draw [white, very thick] (-0.5*\LW,-\RL)--(-0.5*\LW,\RL);
\draw [white, very thick] (0.5*\LW,-\LW)--(0.5*\LW,-\RL);
\draw [white, very thick] (0.5*\LW,\LW)--(0.5*\LW,\RL);
\draw [white, very thick] (0.5*\LW,\LW)--(2*\LW,\LW);
\draw [dashed,white, very thick] (0.5*\LW,0)--(2*\LW,0);
\draw [white, very thick] (0.5*\LW,-\LW)--(2*\LW,-\LW);
\node[rotate=90] at (0,-3) {\includegraphics[width=1.0cm]{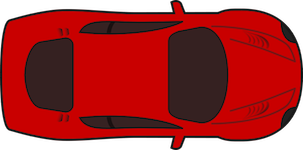}}; 
\begin{scope}[shift={(C)}]
\begin{scope}[rotate around={90:(0,0)}]
\node[rotate=180] at (0,0) {\includegraphics[width=1.0cm]{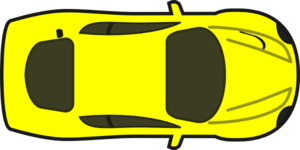}}; 
\end{scope}
\end{scope}
\draw [->, dashed, thick, red] (0,-\LW)--(0,\RL);
\draw [->, dashed, thick, red] (0,-\LW) to [out=90, in=180](0.5*\LW,-0.5*\LW)--(2*\LW, -0.5*\LW);
\end{tikzpicture}
\end{center}
\caption{The yellow car cannot know for sure what is the route of the red one. }\label{fig:routes}
\end{figure}

\begin{figure}
\def\LW{2}
\def\height{2}
\def\width{1}
\def\theta{5}
\begin{center}
\begin{tikzpicture}[scale=0.4, every node/.style={transform shape}]
\draw [gray,fill=gray!70] (-3*\LW,-3*\LW) rectangle (3*\LW,3*\LW);
\draw [white, very thick] (-3*\LW,-\LW) -- (-\LW,-\LW) -- (-\LW,-3*\LW);
\draw [white, very thick] (-3*\LW,\LW) -- (-\LW,\LW) -- (-\LW,3*\LW);
\draw [white, very thick] (3*\LW,-\LW) -- (\LW,-\LW) -- (\LW,-3*\LW);
\draw [white, very thick] (3*\LW,\LW) -- (\LW,\LW) -- (\LW,3*\LW);
\draw [white, dashed, very thick] (0,-3*\LW) -- (0,-\LW);
\draw [white, dashed, very thick] (0,3*\LW) -- (0,\LW);
\draw [white, dashed, very thick] (-3*\LW,0) -- (-\LW,0);
\draw [white, dashed, very thick] (3*\LW,0) -- (\LW,0);
\draw [red, dashed, very thick, ->] (3*\LW,0.5*\LW) -- (-2*\LW,0.5*\LW);
\draw [yellow, dashed, very thick, ->] (0.5*\LW,-3*\LW) -- (0.5*\LW,2*\LW);
\node[rotate=180] at (0.5,1) {\includegraphics[width=2cm]{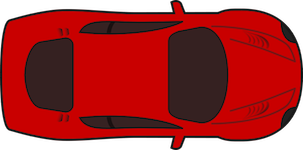}};
\node[rotate=90] at (0.5*\LW, -2*\LW) {\includegraphics[width=2cm]{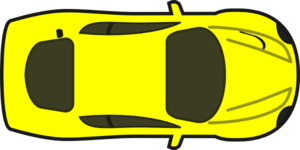}};
\end{tikzpicture}
\end{center}
\caption{``Right of way is given, not taken'': The red car's route has a red light and it is stuck in the intersection. Even though the yellow car's route has a green light, since it has enough distance, it should brake so as to avoid an accident.} \label{fig:red-light-violation}
\end{figure}

\subsubsection{Traffic Lights} \label{sec:traffic-lights}

We next discuss intersections with traffic lights. One might think
that the simple rule for traffic lights scenarios is ``if one car's
route has the green light and the other car's route has a red light,
then the blame is on the one whose route has the red light''. However,
this is not the correct rule. Consider for example the scenario
depicted in \figref{fig:red-light-violation}. Even if the yellow car's
route has a green light, we do not expect it to ignore the red car
that is already in the intersection. The correct rule is that the
route that has a green light have a priority over routes that have a
red light. Therefore, we obtain a clear reduction from traffic lights
to the route priority concept we have described previously. The above
discussion is a formalism of the common sense rule of \textbf{right of
  way is given, not taken}.

\subsubsection{Unstructured Road} \label{sec:unstructured}
\begin{figure}
\begin{center}
\begin{subfigure}[t]{0.4\textwidth}
\includegraphics[width=0.8\textwidth, height=2cm]{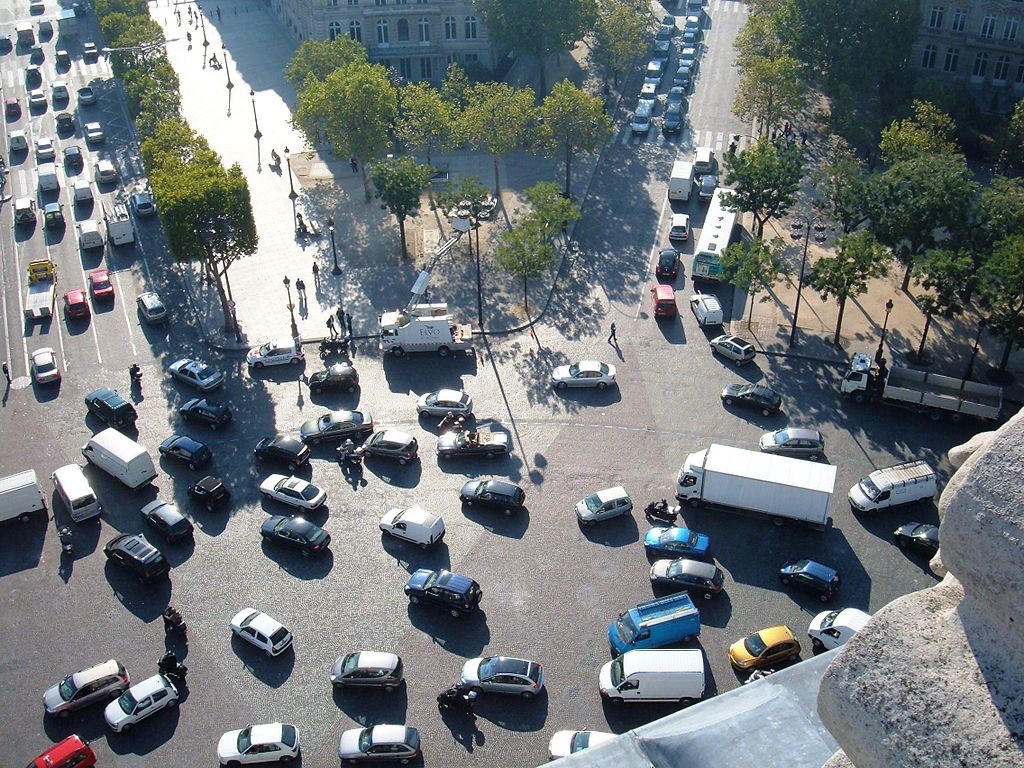}
\caption{}\label{fig:arc_de_triomphe}
\end{subfigure}
~
\begin{subfigure}[t]{0.4\textwidth}
\includegraphics[width=0.8\textwidth, height=2cm]{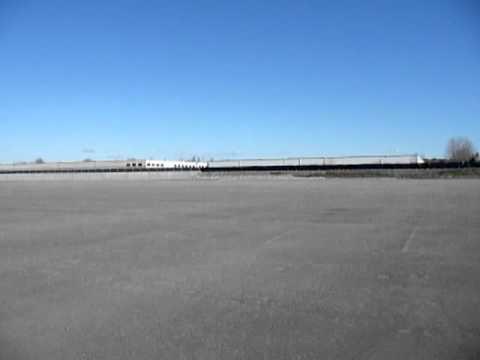}
\caption{}\label{fig:parking_lot}
\end{subfigure}
\end{center}
\caption{Unstructured roads. (a) a wide roundabout around arc-de-triomphe. (b) a parking lot.}\label{fig:unstructured}
\end{figure}

We next turn to consideration of unstructured roads, for example, see
\figref{fig:unstructured}.  Consider first the scenario given in
\figref{fig:arc_de_triomphe}. Here, while the partition of the road
area to lanes is not well defined, the partition of the road to
multiple routes (with a clear geometry for every route) is well
defined. Since our definitions of proper response only depend on the
route geometry, they apply as is to such scenarios.

%







Next, consider the scenario where there is no route geometry at all
(e.g. the parking lot given in \figref{fig:parking_lot}). 
Unlike the structured case, in which we separated the lateral and
longitudinal directions, here we need two dimensional trajectories. 
\begin{definition}[Trajectories] \label{def:general_trajectories}
  Consider a vehicle $c$ riding on some road. A future trajectory of
  $c$ is a function $\tau :\reals_+ \to \reals^2$, where $\tau(t)$ is
  the position of $c$ in $t$ seconds from the current time. The
  tangent vector to the trajectory at $t$, denoted $\tau'(t)$, is the
  Jacobian of $\tau$ at $t$. We denote
  $t_s(\tau) = \sup \{t : \forall t_1 \in [0,t),~ \|\tau'(t_1)\|>0\}$,
  namely, $t_s$ is the first time in which the vehicle will arrive to
  a full stop, where if no such $t$ exists we set $t_s(\tau)=\infty$.
\end{definition}

Dangerous situations will depend on the possibility of a collision
between two trajectories. This is formalized below. 
\begin{definition}[Trajectory Collision] \label{def:trajectory_collision}
  Let $\tau_1,\tau_2$ be two future trajectories of $c_1,c_2$, with
  corresponding stopping times $t_1 = t_s(\tau_1),t_2 =
  t_s(\tau_2)$. Given parameters $\epsilon,\theta$, we say that
  $\tau_1$ and $\tau_2$ do not collide, and denote it by
  $\tau_1 \cap \tau_2 = \emptyset$, if either of the following holds:
\begin{enumerate}
\item For every $t \in [0,\max(t_1,t_2)]$ we have that
  $\|\tau_1(t)-\tau_2(t)\| > \epsilon$.
\item For every $t \in [0,t_1]$ we have that
  $\|\tau_1(t)-\tau_2(t)\| > \epsilon$ and the absolute value of the angle between the vectors 
  $(\tau_2(t_1) - \tau_1(t_1))$ and $\tau'_2(t_1)$ is at most
  $\theta$.
\end{enumerate}
Given a set of trajectories for $c_1$, denoted $\mathcal{T}_1$, and a
set of trajectories for $c_2$, denoted $\mathcal{T}_2$, we say that
$\mathcal{T}_1 \cap \mathcal{T}_2 = \emptyset$ if for every
$(\tau_1,\tau_2) \in \mathcal{T}_1 \times \mathcal{T}_2$ we have that
$\tau_1 \cap \tau_2 = \emptyset$. 
\end{definition}
The first item states that both vehicles will be away from
each other until they are both at a full stop. The second item states
that the vehicles will be away from each other until the
first one is at a full stop, and at that time, the velocity vector of
the second one points away from the first vehicle. 

Note that the collision operator we have defined is not commutative
--- think about two cars currently driving on a very large circle at
the same direction, where $c_1$ is closely behind $c_2$, and consider
$\tau_1$ to be the trajectory in which $c_1$ brakes strongly and $\tau_2$
is the trajectory in which $c_2$ continues at the same speed
forever. Then, $\tau_1 \cap \tau_2 = \emptyset$ while
$\tau_2 \cap \tau_1 \neq \emptyset$.

We continue with a generic approach, that relies on abstract notions 
of ``braking'' and ``continue forward'' behaviors. In the structured case the meanings of
these behaviors were defined based on allowed intervals
for lateral and longitudinal accelerations. We will later specify the
meanings of these behaviors in the unstructured case, but for now we
proceed with the definitions while relying on the abstract notions. 
\begin{definition}[Possible Trajectories due to Braking and Normal
  Driving] \label{def:gen_possible_trajectories} Consider a vehicle
  $c$ riding on some road.  Given a set of constraints, $C$, on the
  behavior of the car, we denote by $\mathcal{T}(C,c)$ the
  set~\footnote{A superset is also allowed. We will use supersets when
    it makes the calculation of the collision operator more easy.} of
  possible future trajectories of $c$ if it will comply with the
  constraints given in $C$. Of particular interest are
  $\mathcal{T}(C_b,c), \mathcal{T}(C_f,c)$ representing the future
  trajectories due to constraints on braking behavior and constraints
  on continue forward behavior.
\end{definition}

We can now refine the notions of safe distance, dangerous situation,
Danger Threshold time, proper response, and responsibility.
\begin{definition}[Safe Distance, Dangerous Situation, Danger Threshold time,
  and Proper Response, in Unstructured Roads] \label{def:safe_distance_generic}
  The distance between $c_0,c_1$ driving on an unstructured road is
  safe if either of the following holds:
  \begin{enumerate}
    \item For some $i \in \{0,1\}$ we have $\mathcal{T}(C_b,c_i) \cap \mathcal{T}(C_f,c_{1-i}) = \emptyset$
      and  $\mathcal{T}(C_b,c_{1-i}) \cap \mathcal{T}(C_f,c_i) \neq \emptyset$
    \item $\mathcal{T}(C_b,c_0) \cap \mathcal{T}(C_b,c_1) = \emptyset$
  \end{enumerate}
  We say that time $t$ is dangerous w.r.t. $c_0,c_1$ if the distance
  between them is non safe. The corresponding Danger Threshold time is the
  earliest dangerous time $t_b$ s.t. during the entire time
  interval $[t_b,t]$ the situation was dangerous. The proper response
  of car $c_j$ at a dangerous time $t$ with corresponding Danger Threshold time
  $t_b$ is as follows:
  \begin{itemize}
  \item If both cars were already at a full stop, then $c_j$ can drive
    away from $c_{1-j}$ (meaning that the absolute value of the angle
    between its velocity vector and the vector of the difference
    between $c_j$ and $c_{1-j}$ should be at most $\theta$, where
    $\theta$ is as in \defref{def:trajectory_collision})
  \item Else, if $t_b$ was safe due to item (1) above and $j = 1-i$,
    then $c_j$ should comply with the constraints of ``continue forward'' behavior,
    as in $C_f$, as long as $c_{1-j}$ is not at a full stop, and after
    that it should behave as in the case that both cars are at a full stop.
  \item Otherwise, the proper response is to brake, namely, to comply
    with the constraints $C_b$.
  \end{itemize}
\end{definition}

Finally, to make this generic approach concrete, we need to specify
the braking constraints, $C_b$, and the continue forward behavior
constraints, $C_f$.  Recall that there are two main things that a
clear structure gives us. Firstly, vehicles can predict what other
vehicles will do (other vehicles are supposed to drive on their route,
and change lateral/longitudinal speed at a bounded rate). Secondly,
when a vehicle is at a dangerous time, the proper response is defined
w.r.t. the geometry of the route (``brake'' laterally and
longitudinally). It is very important that the proper response is not
defined w.r.t. the other vehicle from which we are at a non-safe
distance, because had this been the case, we could have conflicts when
a vehicle were at a non-safe distance w.r.t. more than a single other
vehicle. Therefore, when designing the definitions for unstructured
scenarios, we must make sure that the aforementioned two properties
will still hold.

The approach we take relies on a basic kinematic model of
vehicles. For the speed, we take the same approach as we have taken
for longitudinal velocity (bounding the range of allowed
accelerations). For lateral movements, observe that when a vehicle
maintains a constant angle of the steering wheel and a constant speed,
it will move (approximately) on a circle. In other words, the heading
angle of the car will change at a constant rate, which is called the
yaw rate of the car. We denote the speed of the car by $v(t)$, the
heading angle by $h(t)$, and the yaw rate by $h'(t)$ (as it is the
derivative of the heading angle).  When $h'(t)$ and $v(t)$ are
constants, the car moves on a circle whose ``radius'' is $v(t)/h'(t)$
(where the sign of the ``radius'' determines clockwise or counter
clockwise and the ``radius'' is $\infty$ if the car moves on a line,
i.e. $h'(t)=0$). We therefore denote $r(t)=v(t)/h'(t)$. We will make
two constraints on normal driving. The first is that the inverse of
the radius changes at a bounded rate The second is that $h'(t)$ is
bounded as well.  The expected braking behavior would be to change
$h'(t)$ and $1/r(t)$ in a bounded manner during the response time, and
from there on continue to drive on a circle (or at least be at a
distance of at most $\epsilon/2$ from the circle).  This behavior
forms the analogue of accelerating by at most
$a^{\mathrm{lat}}_{\max,\mathrm{accel}}$ during the response time and
then decelerating until reaching a lateral velocity of zero.

To make efficient calculations of the safe distance, we construct the
superset $\mathcal{T}(C_b,c)$ as follows. W.l.o.g., lets call the
Danger Threshold time to be $t=0$, and assume that at the Danger Threshold time the heading
of $c$ is zero. By the constraint of $|h'(t)| \le h'_{\max}$ we know
that $|h(\rho)| \le \rho\,h'_{\max}$. In addition, the inverse of the
radius at time $\rho$ must satisfy
\begin{equation} \label{eqn:inv_radius}
\frac{1}{r(0)} - \rho\,r^{-1'}_{\max} \le \frac{1}{r(\rho)} \le
\frac{1}{r(0)} + \rho\, r^{-1'}_{\max} ~,
\end{equation}
where $r(0) = v(0)/h'(0)$. All in all, we define  the
superset $\mathcal{T}(C_b,c)$ to be all trajectories such that the
initial heading (at time $0$) is in the range $[-\rho\,h'_{\max},
\rho\,h'_{\max}]$, the trajectory is always on a circle whose inverse
radius is according to \eqref{eqn:inv_radius}, and the longitudinal
velocity on the circle is as in the structured case. For the continue
forward trajectories, we perform the same except that the allowed
longitudinal acceleration even after the response time is in 
$[-a_{\max,\mathrm{brake}},a_{\max,\mathrm{accel}}]$. An illustration
of the extreme radiuses is given in \figref{fig:manifolds}. 

\begin{figure}
\begin{center}
\begin{tabular}{c@{\hskip 5cm}c}
\begin{tikzpicture}[scale=0.4]
\fill (5,0) circle (0.02);
\draw[thick] (-30:5) arc (-30:35:5);
\draw[dashed,red,thick] (5,0) arc (-10:23:6);
\draw[dashed,blue,thick] (5,0) arc (10:50:4);
\end{tikzpicture} &
\begin{tikzpicture}[scale=0.4]
\fill (5,0) circle (0.02);
\draw[thick] (5,-2) -- (5,7);
\draw[dashed,red,thick] (5,0) arc (170:150:20);
\draw[dashed,blue,thick] (5,0) arc (10:30:20);
\end{tikzpicture} 
\end{tabular}
\end{center}
\caption{Illustration of the lateral behavior in unstructured
  scenes. The black line is the current trajectory. The blue and red
  lines are the extreme arcs.} \label{fig:manifolds}
\end{figure}

Finally, observe that these proper responses satisfy the
aforementioned two properties of the proper response for structured
scenarios: it is possible to bound the future positions of other
vehicles in case an emergency will occur, and the same proper response
can be applied even if we are at a dangerous situation w.r.t. more
than a single other vehicle.

\subsection{Pedestrians} \label{sec:pedestrians}

The proper response rules for avoiding collisions involving
pedestrians (or other road users) follow the same ideas described in
previous subsections, except that we need to adjust the parameters in
the definitions of safe distance and proper response, as well as to
specify pedestrians' routes (possibly unstructured routes) and their
priority w.r.t. vehicles' routes. In some cases, a pedestrian's route
is well defined (e.g. a zebra crossing or a sidewalk on a fast road).
In other cases, like a typical residential street, we follow the
approach we have taken for unstructured roads except that unlike
vehicles that typically ride on circles, for pedestrians we constrain the change of heading, $|h'(t)|$,
and assume that at emergency, after the response time, the pedestrian
will continue at a straight line. If the
pedestrian is standing, we assign it to all possible lines
originating from his current position. The priority is set according
to the type of the road and possibly based on traffic lights. For
example, in a typical residential street, a pedestrian has the
priority over the vehicles, and it follows that vehicles must yield and
be cautious with respect to pedestrians. In contrast, there are roads
with a sidewalk where the common sense behavior is that vehicles
should not be worried that a pedestrian on the sidewalk will suddenly
start running into the road. There, cars have the priority.  Another
example is a zebra crossing with a traffic light, where the priority
is set dynamically according to the light.  Of course, priority is
given not taken, hence even if pedestrians do not have priority, if
they entered the road at a safe distance, cars must brake and let them
pass.

Let us illustrate the idea by some examples. The first example is a
pedestrian that stands on a residential road. The pedestrian is
assigned to all routes obtained by rays originating from its current
position. Her safe longitudinal distance w.r.t. each of these virtual
routes is quite short. For example, setting a delay of $500$ ms, and
maximal acceleration\footnote{As mentioned in \cite{lin2016discrete},
  the estimated acceleration of Usain Bolt is $3.09~m/s^2$.} and
braking of $2~ m/s^2$, yields that her part of the safe longitudinal
distance is $50cm$. It follows that a vehicle must be in a kinematic
state such that if it will apply a proper response (acceleration for
$\rho$ seconds and then braking) it will remain outside of a ball of
radius $50cm$ around the pedestrian. 

A second example is a pedestrian standing on the sidewalk right in
front of a zebra crossing, the pedestrian has a red light, and a
vehicle approaches the zebra crossing at a green light. Here, the
vehicle route has the priority, hence the vehicle can assume that the
pedestrian will stay on the sidewalk. If the pedestrian enters the
road while the vehicle is at a safe longitudinal distance (w.r.t. the
vehicle's route), then the vehicle must brake (``right of way is
given, not taken''). However, if the pedestrian enters the road while
the vehicle is not at a safe longitudinal distance, and as a result
the vehicle hits the pedestrian, then the vehicle is not
responsible. It follows that in this situation, the vehicle can drive
at a normal speed, without worrying about the pedestrian.

The third example is a pedestrian that runs on a residential road at
10 km per hour (which is $\approx 2.7 m/s$). The possible future
trajectories of the pedestrian form an isosceles triangle shape. Using
the same parmeters as in the first example, the height of this
triangle is roughly $15 m$. It follows that cars should not enter the
pedestrian's route at a distance smaller than $15 m$. But, if the car
entered the pedestrian's route at a distance larger than $15 m$, and
the pedestrian didn't stop and crashed into the car, then the
responsibility is of the pedestrian.

\subsection{Cautiousness with respect to Occlusion}

A very common human response, when blamed for an accident, falls into
the ``but I couldn't see him'' category. It is, many times,
true. Human sensing capabilities are limited, sometimes because of an
unaware decision to focus on a different part of the road, sometimes
because of carelessness, and sometimes because of physical limitations
- it is impossible to see a little kid hidden behind a parked car.
While advanced automatic sensing systems are never careless, and have
a $360^\circ$ view of the road, they might still suffer from limited
sensing due to physical occlusions or range of sensor detection.  Few
examples are:
\begin{example}
\label{occlusion:example_building} A junction in which a building or a
fence occludes traffic approaching the junction from a different route
(see illustration in \figref{fig:building}).
\end{example}
\begin{example}
 \label{occlusion:example_lane_change} When changing lanes on a highway, there is a limited view range
  for detecting cars arriving from behind (see illustration in \figref{fig:bus}).
\end{example}
\begin{example} \label{occlusion:example_kid} A kid that might be occluded behind a parked car.
\end{example}
\begin{example} \label{occlusion:example_swerve} An obstacle ahead which is occluded by the car in
  front of us (see illustration in \figref{fig:swerve}).
\end{example}

The examples above show that we might be in a dangerous situation
without knowing it (due to the occlusion), and as a result we will not
respond properly. When a human driver claims ``but I couldn't see
him'', a counter argument is often ``well, you should've been more
careful''. Analogously, we should formalize what does it mean to be
careful.

The extreme form of ``being careful'' is to assume the worst possible
scenario. That is, we should assume that every occluded position in
the world is occupied by a vehicle whose velocity is the worst
possible. Unfortunately, without additional assumptions over the
occluded object, this results in an over defensive, non natural
driving. To see this, consider the simplest case of an occlusion by a
static object, as given in \exampleref{occlusion:example_building},
and depicted in \figref{fig:building}. In any position the yellow car 
is, there can be a very far red car which is occluded by the
building. When the yellow car enters the corridor of the red car,
there exists a high enough speed for the red car, for which this will
be considered a non-safe cut-in. This results in inability of the
yellow car to merge into the main road. Of course, it is unreasonable
to limit such manoeuvres. 

This motivates us to formalize several additional ``reasonable
assumptions'' that a driver may make with regard to the behavior of
other road users. Throughout this section, we implicitly defined
``unreasonable
situations'' and allowed vehicles to make the ``reasonable
assumption'' that such cases will not happen. For example, our basic
definition of self longitudinal distance assumes that the other car
will not brake stronger than $a_{\max,\mathrm{brake}}$. By making this
assumption, we implicitly say that a situation in which a vehicle
brakes stronger than $a_{\max,\mathrm{brake}}$ is an unreasonable
situation, and we allow ourselves to not fear from such a case. Of
course, we also required vehicles not to cause ``unreasonable
situations'', for example, we should never brake stronger than
$a_{\max,\mathrm{brake}}$. To tackle occlusions, we follow the exact
same rational, by defining situations that are unreasonable, and allowing the vehicle
  to assume they will not happen. So, the vehicle should plan for the
  worst case, except these unreasonable situations. As we will
  see, we also require the vehicle to take into account that other
  vehicles may not fully observe it, hence it should be careful not to
  cause unreasonable situations. 

  To formalize the above, as a preliminary statement, it is clear that
  once an object becomes observed, we should act according to the
  regular proper response rules. We denote this point in time by the
  \emph{Exposure Time}.

\begin{definition}[Exposure Time] \label{def:exposure_time}
The \emph{Exposure Time} of an object is the first time in which we
can see it (meaning that nothing blocks visibility along the line from the
object to the ego vehicle).
\end{definition}

We continue to define two types of unreasonable situations.  The first
deals with unreasonable situations due to vehicles that drive
unreasonably fast. Consider again the occlusion by a static object, as
given in \exampleref{occlusion:example_building}, and depicted in
\figref{fig:building}. As discussed above, if the yellow car merges
into the road carelessly, it may be that the exposure time comes after
the Danger Threshold time - namely, a dangerous cut-in was performed
unknowingly. However, had there been a limit $v_{\mathrm{limit}}$ on
the speed of the occluded car (the red car), the yellow car would have
been able to approach the merge point slow enough so as to make sure
that it will not enter a dangerous situation w.r.t. any car whose
speed is at most $v_{\mathrm{limit}}$. This leads to the following
definition.

\begin{definition}[Unreasonable Situation due to Unreasonable
  Speed] \label{def:occlusion_RSS} Consider two vehicles, $c_0,c_1$
  driving on routes $r_1,r_2$, where the two vehicles are occluded
  from each other until the exposure time $t$. Assume that $t$ is a
  dangerous time and let $t_b$ be its corresponding blame
  time. For $i \in \{1,2\}$, let
  $v^{\mathrm{lat}}_i, v^{\mathrm{long}}_i$ be the average
  lateral/longitudinal speed from $t_b$ until $t$ of vehicle $i$.  We
  say that this situation is unreasonable w.r.t. parameters
  $v^{\mathrm{lat,limit,low}}_i$, $v^{\mathrm{lat,limit,high}}$, 
  $v^{\mathrm{long,limit,low}}_i$, $v^{\mathrm{long,limit,high}}$, if
  $v^{\mathrm{lat}}_i$ is not in
  $[v^{\mathrm{lat,limit,low}}_i, v^{\mathrm{lat,limit,high}}]$ or
  $ v^{\mathrm{long}}_i$ is not in
  $[v^{\mathrm{long,limit,low}}_i, v^{\mathrm{long,limit,high}}]$.
  The parameters are associated with the position of each $c_i$ on the
  map, priority rules, and possibly on other scene structure and
  conditions.
\end{definition}

\begin{figure}[ht]
\begin{center}
\begin{tabular}{c@{\hskip 2cm}c}
Danger Threshold time & Exposure Time \\ \vspace{2cm}
\begin{tikzpicture}[scale=0.4, every node/.style={transform shape}]

\fill[gray!10] (-5,-5) rectangle (5,2.3);
\draw[densely dashed,thick] (-5,0) -- (0,0) -- (0,-5);
\draw[densely dashed,thick] (2,-5) -- (2,0) -- (5,0);
\draw[densely dashed,thick] (-5,2) -- (5,2);

\draw [fill=red] (-4,-4) rectangle node[color=white,font=\relsize{3}] {Building} (-0.3,-0.3);

\node at (-4,1) {\includegraphics[width=2cm]{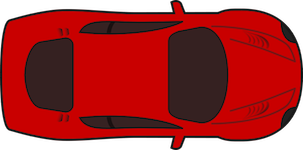}};
\node[rotate=90] at (1,-3) {\includegraphics[width=2cm]{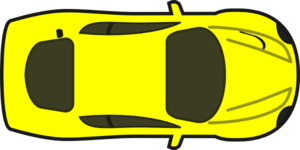}};

\draw[blue] (1,-2.5) -- (-1.5,2);

\end{tikzpicture} &
\begin{tikzpicture}[scale=0.4, every node/.style={transform shape}]

\fill[gray!10] (-5,-5) rectangle (5,2.3);
\draw[densely dashed,thick] (-5,0) -- (0,0) -- (0,-5);
\draw[densely dashed,thick] (2,-5) -- (2,0) -- (5,0);
\draw[densely dashed,thick] (-5,2) -- (5,2);

\draw [fill=red] (-4,-4) rectangle node[color=white,font=\relsize{3}] {Building} (-0.3,-0.3);

\node at (-2.9,1) {\includegraphics[width=2cm]{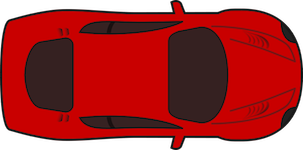}};
\node[rotate=90] at (1,-2) {\includegraphics[width=2cm]{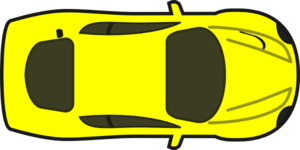}};

\draw[blue] (1,-1.5) -- (-2.8,2);

\end{tikzpicture}
\end{tabular}
\end{center}

\caption{Illustration of the blame and exposure times. During the
  Danger Threshold time, the red car is unobserved by the yellow car, as depicted
  by the blue line bounding its view range. At the
  exposure time, the corner of the red car is firstly observed.}\label{fig:building}
\end{figure}

For example, in \figref{fig:building}, both the red and yellow car may
assume that the other car will approach the intersection slowly
enough, where the upper bound on the velocity depends on the priority
between the two routes. This will allow the yellow car to merge into the main road safely, in the
same manner a human driver does. 

Few of the factors which should be included in determining the
constraints on the speed are given below.  First, the definition takes
into account priority rules - the upper bound on the maximal speed is
higher for the route with priority.  Second, the definition also
considers different scene structure rules. Of particular interest is
the case given in \exampleref{occlusion:example_lane_change}, with its
Danger Threshold time depicted in \figref{fig:bus}. The red car, denoted $c_2$,
is occluded from the yellow car, denoted $c_0$. The lateral distance
was unsafe for some time already, due to the lateral velocity of
$c_2$. The longitudinal distance becomes unsafe in the blame
time. $c_0$ does not have priority due to the route
structure. However, it is clear that the ``longitudinally safe cut
in'' which was performed by $c_2$ during its occlusion, would not be
considered ``safe'' by common sense - as $c_2$ did not make sure to be
seen. It is therefore reasonable to assume a bound on the lateral
velocity of an occluded object. Using this assumption, $c_0$ does not
have to slow down very much when passing near the parked bus, as it
can assume that no cut-in has been performed.

The case given in \exampleref{occlusion:example_kid}, where an
occluded pedestrian may jump into the road, is dealt with in the same
manner - the car can assume that the occluded pedestrian will not
perform a cut-in.

\begin{figure}[ht]
\begin{center}
\begin{tikzpicture}[scale=0.4, every node/.style={transform shape}]
\fill[gray!10] (-5,-2) rectangle (5,2);
\draw[densely dashed,thick] (-5,1.5) -- (5,1.5);
\draw[densely dashed,thick] (-5,0) -- (5,0);
\draw[densely dashed,thick] (-5,-1.5) -- (5,-1.5);
\draw [fill=red] (-2,-1.5+0.1) rectangle node[color=white,font=\relsize{3}]
{Parked Bus} (2,-0.1);
\node at (-4,0.75) {\includegraphics[width=2cm]{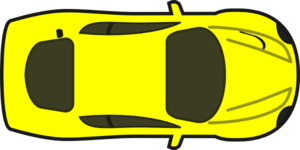}};
\node[rotate=25] at (3.5,-1.2) {\includegraphics[width=2cm]{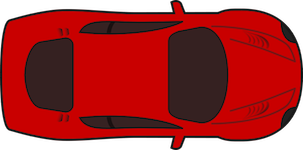}};
\draw[blue] (-3.5,0.75) -- (5,-0.55);
\end{tikzpicture}
\end{center}

\caption{Example for Danger Threshold time, where different velocity constraint
  assumptions should be made by the different cars. }\label{fig:bus}
\end{figure}

Though explicitly dealing with static occluders, the reasonable speed
assumptions are also applicable to moving occluders. For example, had
the bus in \figref{fig:bus} been moving slowly, $c_0$ will still be
able to assume a low lateral velocity for $c_1$.  

The second type of unreasonable situations is a one that stems from
improper behavior of other agents. Consider
\exampleref{occlusion:example_swerve}, as depicted in
\figref{fig:swerve}. The yellow car $c_0$ follows the red car $c_1$,
at the same speed, and at a safe distance. The blue car $c_2$ is at a
complete stop. At the last moment, $c_1$ swerves to the side and
avoids hitting $c_2$. If $c_0$ can also swerve to the side and avoid
hitting $c_2$ it must do so (as follows from
\defref{def:proper_response_evasive}). However, it is possible that
$c_0$ cannot avoid an accident (because there may be a car on its
side), and will end up crashing into $c_2$ from behind. Seemingly,
this accident is $c_0$'s responsibility.

\begin{figure}[ht]
\begin{center}
\begin{tabular}{c@{\hskip 2cm}c}
Danger Threshold time & Exposure Time \\ \vspace{2cm}
\begin{tikzpicture}[scale=0.4, every node/.style={transform shape}]

\fill[gray!10] (-7,-2) rectangle (7,2);
\draw[densely dashed,thick] (-7,1.5) -- (7,1.5);
\draw[densely dashed,thick] (-7,0) -- (7,0);
\draw[densely dashed,thick] (-7,-1.5) -- (7,-1.5);
\node at (-6,-0.75) {\includegraphics[width=2cm]{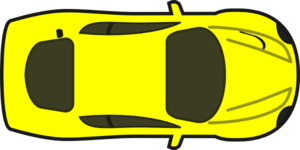}};
\node at (1,-0.75) {\includegraphics[width=2cm]{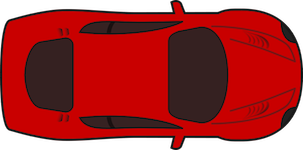}};
\node at (6,-0.75) {\includegraphics[width=2cm]{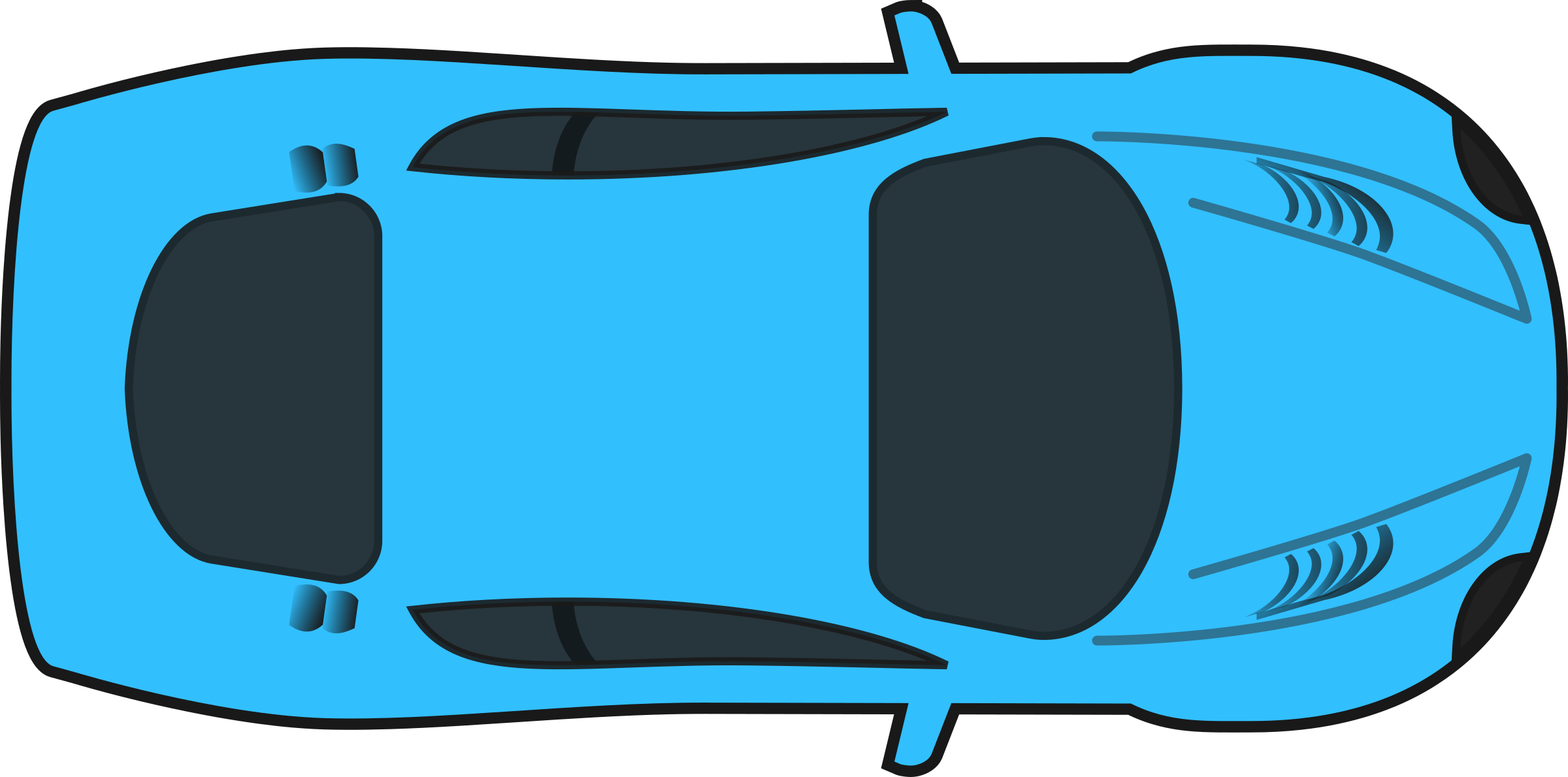}};
\draw[<->](-5,0.2)--(0, 0.2);
\node[above] at (-2.5, 0.2) {Safe};
\draw[<->](2, 0.2)--(5, 0.2);
\node[above] at (3.5, 0.2) {Unsafe};
\draw[<->](-5,0.7)--(5,0.7);
\node[above] at (0,0.7) {Unsafe};

\end{tikzpicture} &
\begin{tikzpicture}[scale=0.4, every node/.style={transform shape}]

\fill[gray!10] (-7,-2) rectangle (7,2);
\draw[densely dashed,thick] (-7,1.5) -- (7,1.5);
\draw[densely dashed,thick] (-7,0) -- (7,0);
\draw[densely dashed,thick] (-7,-1.5) -- (7,-1.5);
\node at (-4,-0.75) {\includegraphics[width=2cm]{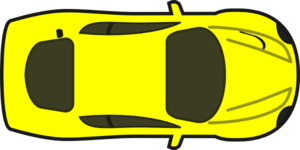}};
\node[rotate=25] at (3,-0.2) {\includegraphics[width=2cm]{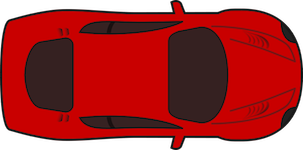}};
\node at (6,-0.75) {\includegraphics[width=2cm]{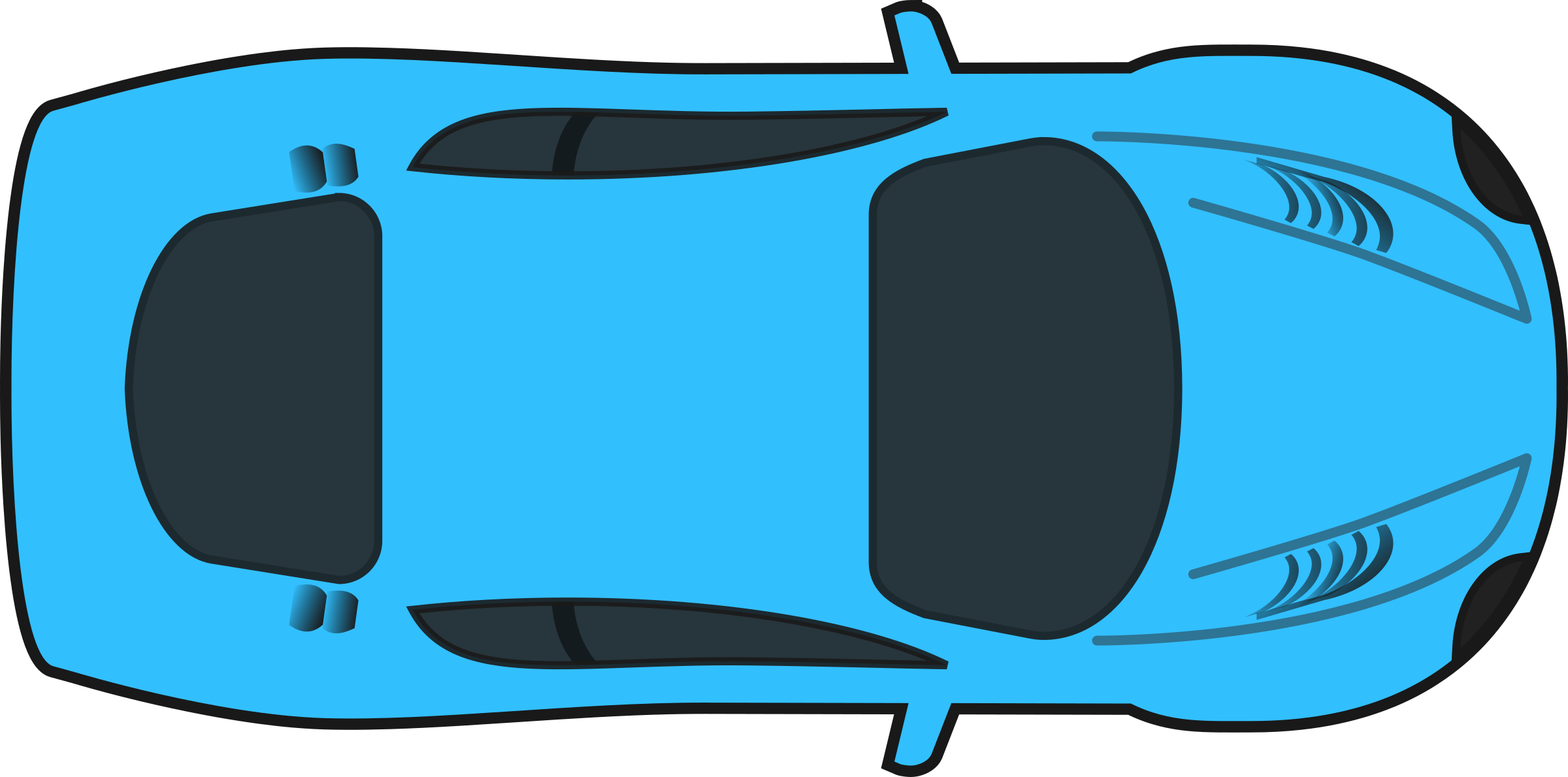}};
\draw[<->](-3,0.7)--(5,0.7);
\node[above] at (1,0.7) {Unavoidable Accident};

\end{tikzpicture}
\end{tabular}
\end{center}

\caption{Danger Threshold time and exposure time. At the Danger Threshold time, the distance between the yellow
  car and the red car is safe due to the high velocity of the red car.
  However, the distance to the blue car is not so, due to it being at
  full stop. No proper response is performed by the yellow car until
  the exposure time, which can result in an accident.}\label{fig:swerve}
\end{figure}

The first attempt one may take in order to deal with this situation is
to indeed require $c_0$ to always assume the worst case situation ---
there might be a stopped car, $c_2$, in front of $c_1$.  Unfortunately,
one cannot drive normally on a highway while assuming this. To
illustrate the defensiveness of the driving that will result of taking
this assumption, let us consider the braking distance of $c_0$, which
is roughly the same as the safe distance w.r.t. a stopped car. This is
the distance it should keep from $c_1$, in order to be safe w.r.t. a
potential car stopped right in front of it. Let us assume that $c_0$
can brake at $a=10m/s^2$, and that it is driving at $v=30m/s$, very
common in a highway. The braking distance is then
$d=\frac{v^2}{2a}=\frac{30^2}{20}=45$.  This means that keeping a
distance of less than $\approx 45m$ from $c_1$ (corresponding to $1.5$
seconds at $v$), is unsafe. Clearly, common highway driving does not
refrain from keeping even much less of a distance from the front car.

Careful examination of the given example reveals the following
observation: had $c_1$ kept a safe (longitudinal) distance from $c_2$,
there would be no problem; $c_1$ would always be able to brake before
hitting $c_2$. $c_0$, properly responding to $c_1$, is able to stop
before hitting $c_1$ - and therefore before hitting $c_2$. This
motivates the definition of our second type of unreasonable
scenarios. 
\begin{definition}[Unreasonable Situation due to Improper
  Behaviour] \label{def:dangerous_improper} A situation at time $t$
  between $c,c'$ is \emph{Unreasonable Situation due to Improper
  Behaviour} if it is impossible to reach it by always
  adhering to proper response rules, that is, there exists no sequence
  of states for the time range $(-\infty,t)$ which is:
\begin{itemize}
\item Physically possible (in the sense it abides to reasonable
velocity and acceleration bounds),
\item Not containing any improper response by either $c$ or $c'$.
\end{itemize}
\end{definition}

An example of an unreasonable situation due to improper behaviour is
the situation between the red and blue cars depicted on the left hand
side of \figref{fig:swerve}. In this situation, the red and blue cars
are at a dangerous situation, but it is impossible to reach this
situation if they would both adhere to proper response.  By allowing
the yellow car to assume that unreasonable situation due to improper
behavior will not happen, when it follows the red car and might fear
there is a car $c_2$ at a non-safe distance in front of the red car,
it can assume that its velocity is reasonably large - otherwise, it
must imply that the the red car did not perform a proper response
w.r.t. $c_2$. Of course, this is a reasonable assumption, but it may
not hold, and in this case, at the exposure time, the yellow car must
apply proper response (including evasive manoeuvres if possible). 

The following definition summarizes the discussion.
\begin{definition}[Proper Response in the presence of
  Occlusions] \label{def:occlusion_RSS_dangerous} A vehicle must
  perform proper response with respect to all road agents it
  observes. It also must perform proper response with respect to
  occluded ones, by assuming that at any occluded position there might
  be an occluded object with any possible speed, unless this yields an
  unreasonable situation as defined in \defref{def:occlusion_RSS} and
  \defref{def:dangerous_improper}. Finally, a vehicle must not cause
  an unreasonable situation (as defined in
  \defref{def:occlusion_RSS}). 
\end{definition}

\subsection{Responsibility}

We have defined the notion of proper response to dangerous
situations. Before an accident occurs, the situation must be
dangerous. We say that an agent is \emph{responsible} for the accident
if it did not comply with the proper response contraints. 

It is not hard to see that if there are no occlusions, if two agents
collide then it must be the case that at least one of them did not
comply with the proper response contraints, and this agent is
responsible for the accident. However, when there are occlusions,
there may be an accident with no clear responsibility. Consider for
example \figref{fig:swerve}. On one hand, the yellow car is allowed to
assume that if there is a vehicle $c_2$ in front of the red car, then
the red car is applying proper response on it. So, the yellow car is
not responsible for the accident. It is true that the red car did not
respond properly on the blue car, but the red car was not involved in
the accident at all so it is not clear if we can blame it for an
accident between the yellow car and the blue car.  In the next section
we prove that if all the agents adhere to proper response rules, then
no accidents would happen.

\subsection{Utopia is Possible}

In \lemref{lem:basic_utopia} we have shown that if all cars on the
road comply with the basic proper response rules, then there will be
no collisions.  However, in the previous subsection we have shown that
agents can make some reasonable assumptions on what happens at
occluded areas, and as a result, there may be an accident between two
agents where none of them is directly responsible. We now prove that
even when considering occlusions, if all agents adhere to the proper
response rules as given in \defref{def:occlusion_RSS_dangerous}, then
no accidents will happen.

The proof technique relies on the lemma below, which shows that if
\emph{all} agents adhere to the proper
response rules as given in \defref{def:occlusion_RSS_dangerous}, they
in fact also adhere to the basic proper response rules. 

\begin{lemma}
If all road agents comply with the proper response rules given in
\defref{def:occlusion_RSS_dangerous}, they also comply with the \emph{basic} proper
response rules. In other words, each agent will behave as if it
observes all other agents. 
\end{lemma}
\begin{proof}
  The proof is by induction on the distance~\footnote{This is an
    example of induction over the reals (see for example
    \cite{clark2012instructor}). We note that since a road agent has a
    minimal size (say, of $1cm$), it is easy to construct an analogue
    inductive argument over the natural numbers, while discretizing
    the distance.} between every two agents. The inductive claim is
  that ``If all road agents comply with the proper response rules
  given in \defref{def:occlusion_RSS_dangerous}, then every pair of
  road agents whose distance is at most $d$ comply with the
  \emph{basic} proper response rules''. The claim clearly holds for $d
  = 0.1$ (as an object at distance of 10cm must be observed). Consider
  some $d$ and assume that the claim holds for every $d' < d$. Assume
  by contradiction that some two agents $c_1,c_2$ of distance $d$ do
  not comply with the basic proper response rules. It must be the case
  that $c_1,c_2$ are occluded from each other by another agent
  $c_3$. But then, $c_3$ must perform basic proper response
  w.r.t. both $c_1$ and $c_2$. In this case, by
  \defref{def:dangerous_improper}, the situation between $c_1,c_3$ and
  $c_2,c_3$ is not unreasonable, from which we deduce that $c_1,c_2$
  must perform proper response with respect to each other, and we
  reached a contradiction. This completes the inductive argument. 
\end{proof}

Combining the above lemma with the argument in
\lemref{lem:basic_utopia} we conclude that:
\begin{corollary}[Utopia is possible]
If all road agents comply with the proper response rules given in
\defref{def:occlusion_RSS_dangerous} then no
collisions can occur.
\end{corollary}

\section{Driving Policy} \label{sec:driving-policy}

A driving policy is a mapping from a sensing state (a description of
the world around us) into a driving command (e.g., the command is
lateral and longitudinal accelerations for the coming second, which
determines where and at what speed should the car be in one second
from now). The driving command is passed to a controller, that aims at
actually moving the car to the desired position/speed.
 
In the previous sections we described a formal safety model and
proposed constraints on the commands issued by the driving policy that
guarantee safety. The constraints on safety are designed for extreme
cases. Typically, we do not want to even need these constraints, and
would like to construct a driving policy that leads to a
comfortable ride. The focus of this section is on how to build an
efficient driving policy, in particular, one that requires
computational resources that can scale to millions of cars.  For now,
we ignore the issue of how to obtain the sensing state and assume an
utopic sensing state, that faithfully represents the world around us
without any limitations. In later sections we will discuss the effect
of inaccuracies in the sensing state on the driving policy.  

We can cast the problem of defining a driving policy in the language
of Reinforcement Learning (RL). At each iteration of RL, an agent
observes a state describing the world, denoted $s_t$, and should pick
an action, denoted $a_t$, based on a policy function, $\pi$, that maps
states into actions. As a result of its action and other factors out
of its control (such as the actions of other agents), the state of the
world is changed to $s_{t+1}$.  We denote a (state,action) sequence by
$\bar{s} =
((s_1,a_1),(s_2,a_2),\ldots,(s_{\textrm{len}(\bar{s})},a_{\textrm{len}(\bar{s})}))$.
Every policy induces a probability function over (state,action)
sequences. This probability function is affected by the actions taken
by the agent, but also depends on the environment (and in particular,
on how other agents behave). We denote by $P_{\pi}$ the probability
over (state,action) sequences induced by $\pi$. The quality of a
policy is defined to be $\E_{\bar{s} \sim P_{\pi}} [ \rho(\bar{s})]$,
where $\rho(\bar{s})$ is a reward function that measures how good the
sequence $\bar{s}$ is. In most case, $\rho(\bar{s})$ takes the form
$\rho(\bar{s}) = \sum_{t=1}^{\textrm{len}(\bar{s})} \rho(s_t,a_t)$,
where $\rho(s,a)$ is an instantaneous reward function, that measures
the immediate quality of being at state $s$ and performing action
$a$. For simplicity, we stick to this simpler case.

To cast the driving policy problem in the above RL language, let $s_t$
be some representation of the road, and the positions, velocities, and
accelerations, of the ego vehicle as well as other road users. Let
$a_t$ be a lateral and longitudinal acceleration command. The next
state, $s_{t+1}$, depends on $a_t$ as well as on how the other agents
will behave. The instantaneous reward, $\rho(s_t,a_t)$, may depend on
the relative position/velocities/acceleration to other cars, the
difference between our speed and the desired speed, whether we follow
the desired route, whether our acceleration is comfortable etc.

The main difficulty of deciding what action should the policy take at
time $t$ stems from the fact that one needs to estimate the long term
effect of this action on the reward. For example, in the context of
driving policy, an action that is taken at time $t$ may seem a good
action for the present (that is, the reward value $\rho(s_t,a_t)$ is
good), but might lead to an accident after $5$ seconds (that is, the
reward value in $5$ seconds would be catastrophic). We therefore need
to estimate the long term quality of performing an action $a$ when the
agent is at state $s$. This is often called the $Q$ function, namely,
$Q(s,a)$ should reflect the long term quality of performing action $a$
at state $s$. Given such a $Q$ function, the natural choice of an
action is to pick the one with highest quality,
$\pi(s) = \argmax_{a} Q(s,a)$.

The immediate questions are how to define $Q$ and how to evaluate $Q$
efficiently. Let us first make the (completely non-realistic)
simplifying assumption that $s_{t+1}$ is some deterministic function
of $(s_t,a_t)$, namely, $s_{t+1} = f(s_t,a_t)$. The reader familiar
with Markov Decision Processes (MDPs), will quickly notice that this
assumption is even stronger than the Markovian assumption of MDPs
(i.e., that $s_{t+1}$ is conditionally independent of the past given
$(s_t,a_t)$). As noted in \cite{shalev2016safe}, even the Markovian
assumption is not adequate for multi-agent scenarios, such as driving,
and we will therefore later relax the assumption.

Under this simplifying assumption, given $s_t$, for every sequence of
decisions for $T$ steps, $(a_t,\ldots,a_{t+T})$, we can calculate
exactly the future states $(s_{t+1},\ldots,s_{t+T+1})$ as well as the
reward values for times $t,\ldots,T$. Summarizing all these reward
values into a single number, e.g. by taking their sum
$\sum_{\tau=t}^T \rho(s_\tau,a_\tau)$, we can define $Q(s,a)$ as
follows:
\[
Q(s,a) = \max_{(a_t,\ldots,a_{t+T})} \sum_{\tau=t}^T
\rho(s_\tau,a_\tau)  ~~~\textrm{s.t.}~~~ 
s_t = s,~a_t = a,~ \forall \tau,~s_{\tau+1} = f(s_\tau,a_\tau)
\]
That is, $Q(s,a)$ is the best future we can hope for, if we are
currently at state $s$ and immediately perform action $a$. 

Let us discuss how to calculate $Q$. The first idea is to discretize
the set of possible actions, $A$, into a finite set $\hat{A}$, and
simply traverse all action sequences in the discretized set. Then, the
runtime is dominated by the number of discrete action sequences,
$|\hat{A}|^T$. If $\hat{A}$ represents $10$ lateral accelerations and
$10$ longitudinal accelerations, we obtain $100^T$ possibilities,
which becomes infeasible even for small values of $T$. While there are
heuristics for speeding up the search (e.g. coarse-to-fine search),
this brute-force approach requires tremendous computational power.

The parameter $T$ is often called the ``time horizon of planning'',
and it controls a natural tradeoff between computation time and
quality of evaluation --- the larger $T$ is, the better our evaluation
of the current action (since we explicitly examine its effect deeper into the
future), but on the other hand, a larger $T$ increases the computation
time exponentially. To understand why we may need a large value of
$T$, consider a scenario in which we are 200 meters before a highway
exit and we should take it. When the time horizon is long enough, the
cumulative reward will indicate if at some time $\tau$ between $t$ and
$t+T$ we have arrived to the exit lane. On the other
hand, for a short time horizon, even if we perform the right immediate
action we will not know if it will lead us eventually to the exit lane.

A different approach attempts to perform offline calculations in order
to construct an approximation of $Q$, denoted $\hat{Q}$, and then
during the online run of the policy, use $\hat{Q}$ as an
approximation to $Q$, without explicitly rolling out the future.  One
way to construct such an approximation is to discretize both the
action domain and the state domain. Denote by $\hat{A},\hat{S}$ these
discretized sets. We can perform an offline calculation for evaluating
the value of $Q(s,a)$ for every $(s,a) \in \hat{S} \times
\hat{A}$. Then, for every $a \in \hat{A}$ we define $\hat{Q}(s_t,a)$
to be $Q(s,a)$ for $s = \argmin_{s \in \hat{S}}
\|s-s_t\|$. Furthermore, based on the pioneering work of Bellman
\cite{bellman1956dynamic,bellman1971introduction}, we can calculate
$Q(s,a)$ for every $(s,a) \in \hat{S} \times \hat{A}$, based on
dynamic programming procedures (such as the Value Iteration
algorithm), and under our assumptions, the total runtime is order of
$T\,|\hat{A}|\,|\hat{S}|$.  The main problem with this approach is
that in any reasonable approximation, $\hat{S}$ is extremely large
(due to the curse of dimensionality). Indeed, the sensing state should
represent $6$ parameters for every other relevant vehicle in the sense
--- the longitudinal and lateral position, velocity, and
acceleration. Even if we discretize each dimension to only $10$ values
(a very crude discretization), since we have $6$ dimensions, to
describe a single car we need $10^6$ states, and to describe $k$ cars
we need $10^{6k}$ states. This leads to unrealistic memory
requirements for storing the values of $Q$ for every $(s,a)$ in
$\hat{S} \times \hat{A}$.

A popular approach to deal with this curse of dimensionality is to
restrict $Q$ to come from a restricted class of functions (often
called a hypothesis class), such as linear functions over manually
determined features or deep neural networks. For example,
\cite{mnih2015human} learned a deep neural network that approximates 
$Q$ in the context of playing Atari games. This leads to a
resource-efficient solution, provided that the class of functions that
approximate $Q$ can be evaluated efficiently. However, there are
several disadvantages of this approach. First, it is not known if the
chosen class of functions contain a good approximation to the desired
$Q$ function. Second, even if such function exists, it is not known if
existing algorithms will manage to learn it efficiently. So far, there
are not many success stories for learning a $Q$ function for
complicated multi-agent problems, such as the ones we are facing in
driving. There are several theoretical reasons why this task is
difficult. We have already mentioned that the Markovian assumption,
underlying existing methods, is problematic. But, a more severe
problem is that we are facing a very small signal-to-noise ratio
due to the time resolution of decision making, as we
explain below. 

Consider a simple scenario in which we need to change lane in order to
take a highway exit in $200$ meters and the road is currently
empty. The best decision is to start making the lane change. We are
making decisions every $0.1$ second, so at the current time $t$, the
best value of $Q(s_t,a)$ should be for the action $a$ corresponding to
a small lateral acceleration to the right. Consider the action $a'$
that corresponds to zero lateral acceleration. Since there is a very
little difference between starting the change lane now, or in $0.1$
seconds, the values of $Q(s_t,a)$ and $Q(s_t,a')$ are almost the
same. In other words, there is very little \emph{advantage} for
picking $a$ over $a'$. On the other hand, since we are using a
function approximation for $Q$, and since there is noise in measuring
the state $s_t$, it is likely that our approximation to the $Q$ value
is noisy. This yields a very small signal-to-noise ratio, which leads
to an extremely slow learning, especially for stochastic learning
algorithms which are heavily used for the neural networks
approximation class. However, as noted in
\cite{baird1994reinforcement}, this problem is not a property of any
particular function approximation class, but rather, it is inherent in
the definition of the $Q$ function.

In summary, existing approaches can be roughly divided into two
camps. The first one is the brute-force approach which includes
searching over many sequences of actions or discretizing the sensing
state domain and maintaining a huge table in memory. This approach can
lead to a very accurate approximation of $Q$ but requires unleashed
resources, either in terms of computation time or in terms of memory.
The second one is a resource efficient approach in which we either
search for short sequences of actions or we apply a function
approximation to $Q$. In both cases, we pay by having a less accurate
approximation of $Q$, that might lead to poor decisions.

Our approach to constructing a $Q$ function that is both
resource-efficient and accurate is to depart from geometrical actions
and to adapt a semantic action space, as described in the next
subsection.

\subsection{Semantics to the rescue}

To motivate our semantic approach, consider a teenager that just got
his driving license. His father seats next to him and gives him
``driving policy'' instructions. These instructions are not geometric
--- they do not take the form ``drive 13.7 meters at the current speed
and then accelerate at a rate of 0.8 $m/s^2$''. Instead, the
instructions are of semantic nature --- ``follow the car in front of
you'' or ``quickly overtake that car on your left''. We formalize a
semantic language for such instructions, and use them as a semantic
action space. We then define the $Q$ function over the semantic action
space. We show that a semantic action can have a very long time
horizon, which allows us to estimate $Q(s,a)$ without planning for
many future semantic actions. Yet, the total number of semantic
actions is still small. This allows us to obtain an accurate
estimation of the $Q$ function while still being resource
efficient. Furthermore, as we show later, we combine learning
techniques for further improving the quality function, while not
suffering from a small signal-to-noise ratio due to a significant
difference between different semantic actions.

We now define our semantic action space. The main idea is to define
lateral and longitudinal goals, as well as the aggressiveness level of
achieving them. Lateral goals are desired positions in lane coordinate
system (e.g., ``my goal is to be in the center of lane number
2''). Longitudinal goals are of three types. The first is relative
position and speed w.r.t. other vehicles (e.g., ``my goal is to be
behind car number 3, at its same speed, and at a distance of $2$
seconds from it''). The second is a speed target (e.g., ``drive at the
allowed speed for this road times 110\%''). The third
is a speed constraint at a certain position (e.g., when approaching a
junction, ``speed of 0 at the stop line'', or when passing a sharp
curve, ``speed of at most 60kmh at a certain position on the
curve''). For the third option we can instead apply a ``speed
profile'' (few discrete points on the route and the desired speed at
each of them).  A reasonable number of lateral goals is bounded by
$16 = 4 \times 4$ ($4$ positions in at most $4$ relevant lanes). A
reasonable number of longitudinal goals of the first type is bounded
by $8 \times 2 \times 3 = 48$ ($8$ relevant cars, whether to be in
front or behind them, and $3$ relevant distances). A reasonable number
of absolute speed targets are $10$, and a reasonable upper bound on
the number of speed constraints is $2$. To implement a given lateral
or longitudinal goal, we need to apply acceleration and then
deceleration (or the other way around). The aggressiveness of
achieving the goal is a maximal (in absolute value)
acceleration/deceleration to achieve the goal. With the goal and
aggressivness defined, we have a closed form formula to implement the
goal, using kinematic calculations. The only remaining part is to determine the
combination between the lateral and
longitudinal goals (e.g., ``start with the lateral goal, and exactly
at the middle of it, start to apply also the longitudinal goal''). A
set of $5$ mixing times and $3$ aggressiveness levels seems more than
enough. All in all, we have obtained a semantic action space whose
size is $\approx 10^4$.

It is worth mentioning that the variable time required for fulfilling
these semantic actions is not the same as the frequency of the
decision making process. To be reactive to the dynamic world, we
should make decisions at a high frequency --- in our implementation,
every 100ms.  In contrast, each such decision is based on constructing
a trajectory that fulfills some semantic action, which will have a
much longer time horizon (say, 10 seconds). We use the longer time
horizon since it helps us to better evaluate the short term prefix of
the trajectory. In the next subsection we discuss the evaluation of
semantic actions, but before that, we argue that semantic actions
induce a sufficient search space.

\paragraph{Is this sufficient:} We have seen that a semantic action
space induces a subset of all possible geometrical curves, whose size
is exponentially smaller (in $T$) than enumerating all possible
geometrical curves. The first immediate question is whether the set of
short term prefixes of this smaller search space contains all
geometric commands that we will ever want to use. We argue that this
is indeed sufficient in the following sense. If the road is free
of other agents, then there is no reason to make changes except
setting a lateral goal and/or absolute acceleration commands and/or
speed constraints on certain positions. If the road contains other
agents, we may want to negotiate the right of way with the other
agents. In this case, it suffices to set longitudinal goals relatively
to the other agents. The exact implementation of these goals in the
long run may vary, but the short term prefixes will not change by
much. Hence, we obtain a very good cover of the relevant short term
geometrical commands.

\subsection{Constructing an evaluation function for semantic actions}

We have defined a semantic set of actions, denoted by $A^s$. Given
that we are currently in state $s$, we need a way to choose the best
$a^s \in A^s$. To tackle this problem, we follow a similar approach to
the \emph{options mechanism} of \cite{sutton1999between}. The basic
idea is to think of $a^s$ as a meta-action (or an option). For each
choice of a meta-action, we construct a geometrical trajectory
$(s_1,a_1),\ldots,(s_T,a_T)$ that represents an implementation of the
meta-action, $a^s$. To do so we of course need to know how other
agents will react to our actions, but for now we are still relying on
(the non-realistic) assumption that $s_{t+1} = f(s_t,a_t)$ for some
known deterministic function $f$. We can now use
$\frac{1}{T} \sum_{t=1}^T \rho(s_t,a_t)$ as a good approximation of
the quality of performing the semantic action $a^s$ when we are at
state $s_1$.

Most of the time, this simple approach yields a powerful driving
policy. However, in some situations a more sophisticated quality function is
required. For example, suppose that we are following a slow truck 
before an exit lane, where we need to take the exit lane. One
semantic option is to keep driving slowly behind the truck. Another
one is to overtake the truck, hoping that later we can get back to the
exit lane and make the exit on time. The quality measure described
previously does not consider what will happen after we will overtake the
truck, and hence we will not choose the second semantic action even if
there is enough time to make the overtake and return to the exit
lane. Machine learning can help us to construct a better evaluation of
semantic actions, that will take into account more than the immediate
semantic actions. Previously, we have argued that learning a $Q$
function over immediate geometric actions is problematic due to the
low signal-to-noise ratio (the lack of advantage). This is not
problematic when considering semantic actions, both because there is a
large difference between performing the different semantic actions and
because the semantic time horizon (how many semantic actions we take
into account) is very small (probably less than three in most cases). 

Another advantage of applying machine learning is for the sake of
\emph{generalization}: we can probably set an adequate evaluation function for
\emph{every} road, by a manual inspection of the properties of the road, and
maybe some trial and error. But, can we automatically generalize to
\emph{any} road? Here, a machine learning approach can be trained on a
large variety of road types so as to generalize to unseen roads as
well. 

To summarize, our semantic action space allows to enjoy the benefits
of both worlds: semantic actions contain information on a
long time horizon, hence we can obtain a very accurate evaluation of
their quality while being resource efficient.

\subsection{The dynamics of the other agents}

So far, we have relied on the assumption that $s_{t+1}$ is a
deterministic function of $s_t$ and $a_t$. As we have emphasized
previously, this assumption is completely not realistic as our actions
affect the behavior of other road users.  While we do take into
account some reactions of other agents to our actions (for example, we
assume that if we will perform a safe cut-in, then the car behind us
will adjust its speed so as not to hit us from behind), it is not
realistic to assume that we model all of the dynamics of other agents.

The solution to this problem is to re-apply our decision making at a
high frequency, and by doing this, we constantly adapt our policy to the parts of
the environment that are beyond our modeling. In a sense, one can
think of this as a Markovization of the world at every step. This is a
common technique that tends to work very good in practice as long as
the balance between modeling error and frequency of planning is adequate.

\section{Sensing} \label{sec:sensing}

In this section we describe the sensing state, which is a description
of the relevant information of the scene, and forms the input to the
driving policy module. By and large, the sensing state contains static
and dynamic objects. The static objects are lanes, physical road
delimiters, constraints on speed, constraints on the right of way, and
information on occluders (e.g. a fence that occludes relevant part of
a merging road). Dynamic objects are vehicles (bounding box, speed,
acceleration), pedestrians (bounding box, speed, acceleration),
traffic lights, dynamic road delimiters (e.g. cones at a construction
area), temporary traffic signs and police activity, and other
obstacles on the road (e.g. an animal, a mattress that fell from a
truck, etc.).

In any reasonable sensor setting, we cannot expect to obtain the exact
sensing state, $s$. Instead, we view raw sensor and mapping data,
which we denote by $x \in X$, and there is a sensing system that takes
$x$ and produces an approximate sensing state. Formally,
\begin{definition}[Sensing system]
Let $S$ denote the domain of sensing state and let $X$ be the domain
  of raw sensor and mapping data. A sensing system is a function $\hat{s} : X \to S$.
\end{definition} 

It is important to understand when we should accept $\hat{s}(x)$ as a reasonable approximation to $s$. The ultimate way to answer this question is by examining the implications of this approximation on the performance of our driving policy in general, and on the safety in particular.
Following our safety-comfort distinction, here again we distinguish between sensing mistakes that lead to non-safe behaviour and sensing mistakes that affect the comfort aspects of the ride. 

Before we dive into the details, let us first describe the type of errors a sensing system might make: 
\begin{itemize}
\item False negative: the sensing system misses an object
\item False positive: the sensing system indicates a ``ghost'' object
\item Inaccurate measurements: the sensing system correctly detects an
  object but incorrectly estimates its position or speed
\item Inaccurate semantic: the sensing system correctly detects an
  object but misinterpret its semantic meaning, for example, the color
  of a traffic light
\end{itemize}

\subsection{Comfort} \label{sec:sensing-comfort}

Recall that for a semantic action $a$, we have used $Q(s,a)$ to denote our evaluation of $a$ given that the current sensing state is $s$. Our policy picks the action $\pi(s) = \argmax_a Q(s,a)$. If we inject $\hat{s}(x)$ instead of $s$ then the selected semantic action would be $\pi(\hat{s}(x)) = \argmax_a Q(\hat{s}(x),a)$. 
Clearly, if $\pi(\hat{s}(x)) = \pi(s)$ then $\hat{s}(x)$ should be accepted as a good approximation to $s$. But, it is also not bad at all to pick $\pi(\hat{s}(x))$ as long as  the quality of  $\pi(\hat{s}(x))$ w.r.t. the true state, $s$, is almost optimal, namely, 
$Q(s, \pi(\hat{s}(x))) \ge Q(s, \pi(s)) - \epsilon$, for some parameter $\epsilon$. We say that $\hat{s}$ is $\epsilon$-accurate w.r.t. $Q$ in such case. Naturally, we cannot expect the sensing system to be $\epsilon$-accurate all the time. We therefore also allow the sensing system to fail with some small probability $\delta$. In such a case we say that $\hat{s}$ is Probably (w.p. of at least $1-\delta$), Approximately (up to $\epsilon$), Correct, or PAC for short (borrowing Valiant's PAC learning terminology~\cite{Valiant84}).  

We may use several $(\epsilon,\delta)$ pairs for evaluating different aspects of the system. For example, we can choose three thresholds, $\epsilon_1 < \epsilon_2 < \epsilon_3$ to represent mild, medium, and gross mistakes, and for each one of them set a different value of $\delta$. This leads to the following definition.

\begin{definition}[PAC sensing system]
Let $((\epsilon_1,\delta_1),\ldots,(\epsilon_k,\delta_k))$ be a set of (accuracy,confidence) pairs, let $S$ be the sensing state domain, let $X$ be the 
  raw sensor and mapping data domain, and let $D$ be a distribution over
  $X \times S$.  Let $A$ be an action space, $Q : S \times A \to \reals$ be a quality function, and $\pi : S \to A$ be such that $\pi(s) \in \argmax_a Q(s,a)$.  A
  sensing system, $\hat{s} : X \to S$, is
  Probably-Approximately-Correct (PAC) with respect to the above parameters if for every $i \in \{1,\ldots,k\}$ we have that
  $\prob_{(x,s) \sim D}[Q(s, \pi(\hat{s}(x))) \ge Q(s, \pi(s)) - \epsilon_i] \ge
  1-\delta_i$.
\end{definition}

Few remarks are in order:
\begin{itemize}
\item The definition depends on a distribution $D$ over $X \times S$. It is important to emphasize that we construct this distribution by recording data of many human drivers but not by following the particular policy of our autonomous vehicle. While the latter seems more adequate, it necessitates online validation, which makes the development of the sensing system impractical. Since the effect of any reasonable policy on $D$ is minor, by applying simple data augmentation techniques we can construct an adequate distribution and then perform offline validation after every major update of the sensing system.
\item The definition provides a sufficient, but not necessary, condition for comfort ride using $\hat{s}$. It is not necessary because it ignores the important fact that short term wrong decisions have little effect on the comfort of the ride. For example, suppose that there is a vehicle $100$ meters in front of us, and it is slower than the host vehicle. The best decision would be to start accelerating slightly now. If the sensing system misses this vehicle, but will detect it in the next time (after 100 mili-seconds), then the difference between the two rides will not be noticeable. To simplify the presentation, we have neglected this issue and required a stronger condition. The adaptation to a multi-frame PAC definition is conceptually straightforward, but involves more technicality and therefore we omit it. 
\end{itemize}

We next derive design principles that follow from the above PAC definition. 
Recall that we have described several types of sensing mistakes. For mistakes of types false negative, false positive, and inaccurate semantic, either the mistakes will be on non-relevant objects (e.g., a traffic light for left turn when we are proceeding straight), or they will be captured by the $\delta$ part of the definition. We therefore focus on the ``inaccurate measurements'' type of errors, which happens frequently. 

Somewhat surprisingly, we will show that the popular approach of measuring the accuracy of a sensing system via ego-accuracy (that is, by measuring the accuracy of position of every object with respect to the host vehicle) is not sufficient for ensuring PAC sensing system. We will then propose a different approach that ensures PAC sensing system, and will show how to obtain it efficiently. We start with some additional definitions. 

For every object $o$ in the scene, let $p(o), \hat{p}(o)$ be the positions of $o$ in the coordinate system of the host vehicle according to $s,\hat{s}(x)$, respectively. Note that the distance between $o$ and the host vehicle is $\|p\|$. The \emph{additive} error of $\hat{p}$ is $\|p(o)-\hat{p}(o)\|$.The \emph{relative} error of $\hat{p}(o)$, w.r.t. the distance between $o$ and the host vehicle, is the additive error divided by $\|p(o)\|$, namely $\frac{\|p(o)-\hat{p}(o)\|}{\|p(o)\|}$.  

We first argue that it is not realistic to require that the additive
error is small for far away objects. Indeed, consider $o$ to be a
vehicle at a distance of $150$ meters from the host vehicle, and let
$\epsilon$ be of moderate size, say $\epsilon = 0.1$. For additive
accuracy, it means that we should know the position of the vehicle up
to $10$cm of accuracy. This is not realistic for reasonably priced
sensors. On the other hand, for relative accuracy we need to estimate
the position up to 10\%, which amounts to $15$m of accuracy. This is
feasible to achieve (as we will describe later).

We say that a sensing system, $\hat{s}$, positions a set of objects, $O$, in an $\epsilon$-ego-accurate way, if for every $o \in O$, the (relative) error between $p(o)$ and $\hat{p}(o)$ is at most $\epsilon$. 
The following example shows that an $\epsilon$-ego-accurate sensing state does not guarantee PAC sensing system with respect to every reasonable $Q$. Indeed, consider a scenario in which the host vehicle drives at a speed of $30m/s$, and there is a stopped vehicle $150$ meters in front of it. If this vehicle is in the ego lane, and there is no option to change lanes in time, we must start decelerating now at a rate of at least $3m/s^2$ (otherwise, we will either not stop in time or we will need to decelerate strongly later). On the other hand, if the vehicle is on the side of the road, we don't need to apply a strong deceleration. Suppose that $p(o)$ is one of these cases while $\hat{p}(o)$ is the other case, and there is a $5$ meters difference between these two positions. Then, the relative error of $\hat{p}(o)$ is 
\[
\frac{\|\hat{p}(o) - p(o)\|}{\|p(o)\|} ~=~  \frac{5}{150} = \frac{1}{30} \le 0.034 ~.
\]
That is, our sensing system may be $\epsilon$-ego-accurate for a rather small value of $\epsilon$ (less than 3.5\% error), and yet, for any reasonable $Q$ function, the values of $Q$ are completely different since we are confusing between a situation in which we need to brake strongly and a situation in which we do not need to brake strongly.

The above example shows that $\epsilon$-ego-accuracy does not guarantee that our sensing system is PAC. Is there another property that is sufficient for PAC sensing system?
Naturally, the answer to this question depends on $Q$. We will describe a family of $Q$ functions for which there is a simple property of the positioning that guarantees PAC sensing system. Intuitively, the problem of $\epsilon$-ego-accuracy is that it might lead to semantic mistakes --- in the aforementioned example, even though $\hat{s}$ was $\epsilon$-ego-accurate with $\epsilon < 3.5\%$, it mis-assigned the vehicle to the correct lane. To solve this problem, we rely on \emph{semantic units} for lateral position.  
\begin{definition}[semantic units]
A lane center is a simple natural curve, namely, it is a differentiable, injective, mapping $\ell : [a,b] \to \reals^3$, where for every $a \le t_1 < t_2 \le b$ we have that the length $\mathrm{Length}(t_1,t_2) := \int_{\tau=t_1}^{t_2} |\ell'(\tau)| d\tau$ equals to $t_2-t_1$. 
The width of the lane is a function $w: [a,b] \to \reals_+$. The projection of a point $x \in \reals^3$ onto the curve is the point on the curve closest to $x$, namely, the point $\ell(t_x)$ for $t_x = \argmin_{t \in [a,b]} \|\ell(t)-x\|$. The semantic longitudinal position of $x$ w.r.t. the lane is $t_x$ and the semantic lateral position of $x$ w.r.t. the lane is $\|\ell(t_x)-x\|/w(t_x)$. Semantic speed and acceleration are defined as first and second derivatives of the above. 
\end{definition}

Similarly to geometrical units, for semantic longitudinal distance we use relative error: if $\hat{s}$ induces a semantic longitudinal distance of $\hat{p}(o)$ for some object, while the true distance is $p(o)$, then the relative error is $\frac{|\hat{p}(o)-p(o)|}{\max\{p(o),1\}}$ (where the maximum in the denominator deals with cases in which the object has almost the same longitudinal distance (e.g., a car next to us on another lane). Since semantic lateral distances are small we can use additive error for them. This leads to the following defintion:
\begin{definition}[error in semantic units]
Let $\ell$ be a lane and suppose that the semantic longitudinal distance of the host vehicle w.r.t. the lane is $0$. Let $x \in \reals^3$ be a point and let $p_{\mathrm{lat}}(x), p_{\mathrm{lon}}(x)$ be the semantic lateral and longitudinal distances to the point w.r.t. the lane. Let  $\hat{p}_{\mathrm{lat}}(x), \hat{p}_{\mathrm{lon}}(x)$ be approximated measurements. The distance between $\hat{p}$ and $p$ w.r.t. $x$ is defined as
\[
d(\hat{p},p;x) ~=~ 
\max\left\{  |\hat{p}_{\mathrm{lat}}(x) - p_{\mathrm{lat}}(x)|  ~,~ \frac{| \hat{p}_{\mathrm{lon}}(x) - p_{\mathrm{lon}}(x)|}{\max\{p_{\mathrm{lon}}(x),1\}} \right\}
\]
\end{definition}

The error of the lateral and longitudinal semantic velocities is defined analogously. 

Equipped with the above definition, we are ready to define the property of $Q$ and the corresponding sufficient condition for PAC sensing system.  \begin{definition}[Semantically-Lipschitz $Q$] 
A $Q$ function is $L$-semantically-Lipschitz if for every $a$, $s,\hat{s}$,  $|Q(s,a)-Q(\hat{s}(x),a)| \le L\, \max_o d(\hat{p},p;o)$, where $\hat{p},p$ are the measurements induced by $s,\hat{s}$ on an object $o$.  
\end{definition}

As an immediate corollary we obtain:
\begin{lemma}
If $Q$ is $L$-semantically-Lipschitz and 
a sensing system $\hat{s}$ produces semantic measurements such that with probability of at least $1-\delta$ we have $d(\hat{p},p;o) \le \epsilon/L$, then $\hat{s}$ is a PAC sensing system with parameters $\epsilon,\delta$.
\end{lemma}

\subsection{Safety} \label{sec:sensing-safety}

We now discuss sensing mistakes that lead to non-safe behavior.  As
mentioned before, our policy is provably safe, in the sense that it
won't lead to accidents of the autonomous vehicle's blame. Such accidents might still
occur due to hardware failure (e.g., a break down of all the sensors
or exploding tire on the highway), software failure (a significant bug
in some of the modules), or a sensing mistake. Our ultimate goal is
that the probability of such events will be extremely small --- a
probability of $10^{-9}$ for such an accident per hour. To appreciate
this number, the average number of hours an american driver spends on
the road is (as of 2016) less than $300$. So, in expectation, one
needs to live 3.3 million years to be in an accident.

Roughly speaking, there are two types of safety-critic sensing
mistake. The first type is a \emph{safety-critic miss}, meaning that a
dangerous situation is considered non-dangerous according to our
sensing system. The second type is a \emph{safety-critic ghost},
meaning that a non-dangerous situation is considered dangerous
according to our sensing system. Safety-critic misses are obviously
dangerous as we will not know that we should respond properly to the
danger. Safety-critic ghosts might be dangerous when our speed is
high, we brake hard for no reason, and there is a car behind us.

Usually, a safety-critic miss is caused by a false negative while a
safety-critic ghost is caused by a false positive. Such mistakes can
also be caused from significantly incorrect measurements, but in most
cases, our comfort objective ensures we are far away from the boundaries
of non-safe distances, and therefore reasonable measurement
errors are unlikely to lead to safety-critic mistakes.

How can we ensure that the probability of safety-critic mistakes will
be very small, say, smaller than $10^{-9}$ per hour? As followed from
\lemref{lem:statistical_validation_infeasible}, without making further
assumptions we need to check our system on more than $10^9$ hours of
driving. This is unrealistic (or at least extremely challenging) ---
it amounts to recording the driving of $3.3$ million cars over a
year. Furthermore, building a system that achieves such a high
accuracy is a great challenge. Our solution for both the system design
and validation challenges is to rely on several sub-systems, each of
which is engineered independently and depends on a different
technology, and the systems are fused together in a way that ensures
boosting of their individual accuracy.

Suppose we build $3$ sub-systems, denoted, $s_1,s_2,s_3$ (the
extension to more than $3$ is straightforward). Each sub-system should
decide if the current situation is dangerous or not. Situations which
are non-dangerous according to the majority of the sub-systems ($2$ in
our case) are considered safe.

Let us now analyze the performance of this fusion scheme. We rely on the following definition: 
\begin{definition}[One side $c$-approximate independent]
Two Bernoulli random variables $r_1,r_2$ are called one side $c$-approximate independent if
\[
\prob[r_1 \land r_2] \le c~ \prob[r_1] \, \prob[r_2] ~.
\]
\end{definition}
For $i \in \{1,2,3\}$, denote by $e^{m}_i,e^{g}_i$ the Bernoulli random variables that indicate if sub-system $i$ has a safety-critic miss/ghost respectively. Similarly, $e^m,e^g$ indicate a  safety-critic miss/ghost of the fusion system. 
We rely on the assumption that for any pair $i \neq j$, the random variables $e^{m}_i, e^{m}_j$ are one sided $c$-approximate independent, and the same holds for $e^{g}_i, e^{g}_j$. Before explaining why this assumption is reasonable, let us first analyze its implication. We can bound the probability of $e^m$ by:
\begin{align*}
\prob[e^m] &= \prob[e^m_1 \land e^m_2 \land e^m_3] + \sum_{j=1}^3 \prob[\lnot e^m_j \land \land_{i \neq j} e^m_i] \\
&\le 3\prob[e^m_1 \land e^m_2 \land e^m_3] + \sum_{j=1}^3 \prob[\lnot e^m_j \land \land_{i \neq j} e^m_i] \\
&= \sum_{j=1}^3 \prob[ \land_{i \neq j} e^m_i] \\
&\le c~ \sum_{j=1}^3 \prod_{i \neq j} \prob[ e^m_i] .
\end{align*} 
Therefore, if all sub-systems have $\prob[ e^m_i] \le p$ then $\prob[e^m] \le 3\,c\,p^2$. The exact same derivation holds for the safety-critic ghost mistakes. By applying a union bound we therefore conclude:
\begin{corollary}
Assume that for any pair $i \neq j$, the random variables $e^{m}_i, e^{m}_j$ are one sided $c$-approximate independent, and the same holds for $e^{g}_i, e^{g}_j$. Assume also that for every $i$, $\prob[e^m_i] \le p$ and $\prob[e^g_i] \le p$. Then, 
\[
\prob[e^m \lor e^g] \le 6\,c\,p^2 ~. 
\]
\end{corollary}
This corollary allows us to use significantly smaller data sets in order to validate the sensing system. For example, if we would like to achieve a safety-critic mistake probability of $10^{-9}$, instead of taking order of $10^9$ examples, it suffices to take order of $10^5$ examples and test each system separately. 

It is left to reason about the rational behind the one sided
independence assumption.  There are pairs of sensors that yield
completely non-correlated errors. For example, radar works well in bad
weather conditions but might fail due to non-relevant metallic
objects, as opposed to camera that is affected by bad weather but is
not likely to be affected by metallic objects. Seemingly, camera and
lidar have common sources of mistakes --- both are affected by foggy
weather, heavy rain or snow. However, the type of mistake for camera
and lidar would be different --- camera might miss objects due to bad
weather while lidar might detect a ghost due to reflections from
particles in the air. Since we have distinguished between the two
types of mistakes, the approximate independency is still likely to
hold.

\begin{remark} Our definition of safety-critic ghost requires that the
  situation is dangerous by at least two sensors. We argue that even
  in difficult conditions (e.g. heavy fog), this is unlikely to
  happen. The reason is that in such situations, systems that are
  affected by the difficult conditions (e.g. the lidar), will dictate
  a very defensive driving to the policy, as they can declare that
  high velocity and lateral maneuvers would lead to a dangerous
  situation. As a result, we will drive slowly, and then even if we
  require an emergency stop, it is not dangerous due to the low speed
  of driving. Therefore, an adaptation of the driving style to the
  conditions of the road will follow from the definitions.
\end{remark}

\subsection{Building a scalable sensing system}

We have described the requirements from a sensing system, both in
terms of comfort and safety. We now briefly suggest our approach for
building a sensing system that meets these requirements while being
scalable.

There are three main components of our sensing system. The first is
long range, $360$ degrees coverage, of the scene based on cameras. The
three main advantages of cameras are: (1) high resolution, (2)
texture, (3) price. The low price enables a scalable system. The
texture enables to understand the semantics of the scene, including
lane marks, traffic light, intentions of pedestrians, and more. The
high resolution enables a long range of detection. Furthermore,
detecting lane marks and objects in the same domain enables excellent
semantic lateral accuracy.  The two main disadvantages of cameras are:
(1) the information is 2D and estimating longitudinal distance is
difficult, (2) sensitivity to lighting conditions (low sun, bad
weather). We overcome these difficulties using the next two components
of our system.

The second component of our system is a semantic high-definition
mapping technology, called Road Experience Management (REM).  A common
geometrical approach to map creation is to record a cloud of 3D points
(obtained by a lidar) in the map creation process, and then,
localization on the map is obtained by matching the existing lidar
points to the ones in the map. There are several disadvantages of this
approach. First, it requires a large memory per kilometer of mapping
data, as we need to save many points. This necessitates an expensive
communication infrastructure. Second, only few cars are equipped with
lidar sensors, and therefore, the map is updated very
infrequently. This is problematic as changes in the road can occur
(construction zones, hazards), and the ``time-to-reflect-reality'' of
lidar-based mapping solutions is large. In contrast, REM follows a
semantic-based approach. The idea is to leverage the large number of
vehicles that are equipped with cameras and with software that detects
semantically meaningful objects in the scene (lane marks, curbs,
poles, traffic lights, etc.). Nowadays, many new cars are equipped
with ADAS systems which can be leveraged for crowd source based
creation of the map. Since the processing is done on the vehicle side,
only a small amount of semantic data should be communicated to the
cloud. This allows a very frequent update of the map in a scalable
way. In addition, the autonomous vehicles can receive the small sized
mapping data over existing communication platforms (the cellular
network). Finally, highly accurate localization on the map can be
obtained based on cameras, without the need for expensive lidars.

REM is used for three purposes. First, it gives us a foresight on the
static structure of the road (we can plan for a highway exit way in
advance). Second, it gives us another source of accurate information
of all of the static information, which together with the camera
detections yields a robust view of the static part of the
world. Third, it solves the problem of lifting the 2D information from
the image plane into the 3D world as follows. The map describes all of
the lanes as curves in the 3D world. Localization of the ego vehicle
on the map enables to trivially lift every object on the road from the
image plane to its 3D position. This yields a positioning system that
adheres to the accuracy in semantic units described in
\secref{sec:sensing-comfort}.

The third component of our system is a complementary radar and lidar
system. This system serves two purposes. First, they enable to yield
an extremely high accuracy for the sake of safety (as described in
\secref{sec:sensing-safety}). Second, they give direct measurements on
speed and distances, which further improves the comfort of the ride.

\bibliographystyle{plain}
\bibliography{bib}

\appendix

\section{Technical Lemmas}
\subsection{Technical Lemma 1}\label{lem:1-x}
\begin{lemma}
For all $x\in[0,0.1]$, it holds that $1-x\geq e^{-2x}$.
\end{lemma}
\begin{proof}
Let $f(x)=1-x-e^{-2x}$. Our goal is to show $f(x)\geq 0$ for
$x\in[0,0.1]$. Note that $f(0)=0$, and it is therefore sufficient to
have that $f'(x)\geq 0$ in the aforementioned range. Explicitly,
$f'(x)=-1+2e^{-2x}$. Clearly, $f'(0)=1$, and it is monotonically
decreasing, hence it is sufficient to verify that $f'(0.1)>0$, which
is easy to do numerically, $f'(0.1)\approx 0.637$.
\end{proof}

\section*{Acknowledgments}
We thank Marc Lavabre from Renault for fruitful discussions. We thank Jack Weast for assisting with the write-up of rev.4 of this manuscript.

\end{document}